\documentclass[10pt]{article}

\usepackage[utf8]{inputenc}
\usepackage[T1]{fontenc}

\usepackage{epsf}
\usepackage{amsmath}

\newcommand\dbover[1]{\overline{\overline{#1}}}

\allowdisplaybreaks

\usepackage[showframe=false]{geometry}
\usepackage{changepage}

\usepackage{epsfig}
\usepackage{amssymb}

\usepackage{amsthm}
\usepackage{setspace}
\usepackage{cite}
\usepackage{mcite}

\usepackage{algorithmic}  % This package provides an algorithmic environment fo describing algorithms.
\usepackage{algorithm}

\usepackage{shadow}
\usepackage{fancybox}
\usepackage{fancyhdr}

\usepackage{color}
\usepackage[usenames,dvipsnames,svgnames,table]{xcolor}
\newcommand{\bl}[1]{\textcolor{blue}{#1}}

\definecolor{mypurple}{rgb}{.4,.0,.5}

\usepackage[hyphens]{url}

\usepackage[colorlinks=true,
            linkcolor=black,
            urlcolor=blue,
            citecolor=purple]{hyperref}

\usepackage{breakurl}
\def\s{{\bf s}}

\def\y{{\bf y}}
\def\v{{\bf v}}
\def\x{{\bf x}}

\def\x{{\mathbf x}}

\def\s{{\bf s}}

\def\v{{\bf v}}
\def\x{{\bf x}}
\def\y{{\bf y}}
\def\z{{\bf z}}

\def\c{{\bf c}}

\def\h{{\bf h}}

\def\tr{\mbox{Tr}}

\def\tr{{\rm tr}\,}

\def\be{\begin{equation}}
\def\ee{\end{equation}}
\def\ba{\left[\begin{array}}
\def\ea{\end{array}\right]}

\def\t{{\bf t}}

\def\s{{\bf s}}

\def\v{{\bf v}}
\def\x{{\bf x}}
\def\y{{\bf y}}
\def\z{{\bf z}}

\def\c{{\bf c}}

\def\e{{\bf e}}

\def\1{{\bf 1}}

\def\g{{\bf g}}
\def\0{{\bf 0}}

\def\X{{\bf X}}

\def\mR{{\mathbb R}}

\def\mE{{\mathbb E}}

\def\mP{{\mathbb P}}

\def\lp{\left (}
\def\rp{\right )}

\def\s{{\bf s}}

\def\y{{\bf y}}
\def\v{{\bf v}}
\def\x{{\bf x}}

\def\x{{\mathbf x}}

\def\s{{\bf s}}

\def\v{{\bf v}}
\def\x{{\bf x}}
\def\y{{\bf y}}
\def\z{{\bf z}}

\def\c{{\bf c}}

\def\h{{\bf h}}

\def\tr{\mbox{Tr}}

\def\tr{{\rm tr}\,}

\def\be{\begin{equation}}
\def\ee{\end{equation}}
\def\ba{\left[\begin{array}}
\def\ea{\end{array}\right]}

\def\t{{\bf t}}

\def\s{{\bf s}}

\def\v{{\bf v}}
\def\x{{\bf x}}
\def\y{{\bf y}}
\def\z{{\bf z}}

\def\c{{\bf c}}

\def\e{{\bf e}}

\def\({\left (}
\def\){\right )}

\def\1{{\bf 1}}

\def\g{{\bf g}}
\def\0{{\bf 0}}

\def\cL{{\mathcal L}}
\def\cA{{\mathcal A}}

\usepackage{xcolor}
\usepackage{color}

\definecolor{darkgreen}{rgb}{0, 0.4,0}
\newcommand{\dgr}[1]{\textcolor{darkgreen}{#1}}

\definecolor{purplebrown}{rgb}{0.5,0.1,0.6}

\definecolor{ultclupcol}{rgb}{0.1,0.5,0.5}

\definecolor{mytrycolor}{rgb}{0.5,0.7,0.2}

\definecolor{ultclupcola}{rgb}{.5,0,.5}

\definecolor{shadebrown}{rgb}{0.1,0.1,0.9}
\definecolor{lightblue}{rgb}{0.2,0,1}

\usepackage{fancybox}
\usepackage{graphicx}
\usepackage{epstopdf}
\usepackage{epsfig}
\usepackage{wrapfig}
\usepackage{subfigure}

\usepackage{xcolor}
\usepackage{tcolorbox}
\tcbuselibrary{skins}

\newtcbox{\xmybox}{on line,
arc=7pt,
before upper={\rule[-3pt]{0pt}{10pt}},boxrule=0pt,
boxsep=0pt,left=6pt,right=6pt,top=0pt,bottom=0pt,enhanced, coltext=blue, colback=white!10!yellow}

\newtcbox{\xmyboxa}{on line,
arc=7pt,
before upper={\rule[-3pt]{0pt}{10pt}},boxrule=0pt,
boxsep=0pt,left=6pt,right=6pt,top=0pt,bottom=0pt,enhanced, colback=white!10!yellow}

\newtcbox{\xmyboxb}{on line,
arc=7pt,
before upper={\rule[-3pt]{0pt}{10pt}},boxrule=1pt,colframe=darkgreen!100!blue,
boxsep=0pt,left=6pt,right=6pt,top=0pt,bottom=0pt,enhanced, colback=white!10!yellow}

\newtcbox{\xmyboxc}{on line,
arc=7pt,
before upper={\rule[-3pt]{0pt}{10pt}},boxrule=.7pt,colframe=blue!100!blue,
boxsep=0pt,left=6pt,right=6pt,top=0pt,bottom=0pt,enhanced, coltext=blue, colback=white!10!yellow}

\newtcbox{\xmytboxa}{on line,
arc=7pt,
before upper={\rule[-3pt]{0pt}{10pt}},boxrule=.0pt,colframe=pink!50!yellow,
boxsep=0pt,left=6pt,right=6pt,top=0pt,bottom=0pt,enhanced, coltext=white, colback=blue!40!red}

\newtcbox{\xmytboxb}{on line,
arc=7pt,
before upper={\rule[-3pt]{0pt}{10pt}},boxrule=.0pt,colframe=pink!50!yellow,
boxsep=0pt,left=6pt,right=6pt,top=0pt,bottom=0pt,enhanced, coltext=white, colback=white!40!green}

\makeatletter
\newcommand\subsubsubsection{\@startsection{paragraph}{4}{\z@}{-2.5ex\@plus -1ex \@minus -.25ex}{1.25ex \@plus .25ex}{\normalfont\normalsize\bfseries}}
\newcommand\subsubsubsubsection{\@startsection{subparagraph}{5}{\z@}{-2.5ex\@plus -1ex \@minus -.25ex}{1.25ex \@plus .25ex}{\normalfont\normalsize\bfseries}}
\makeatother

\newtheorem{theorem}{Theorem}

\newtheorem{corollary}{Corollary}

\newtheorem{lemma}{Lemma}

\begin{document}

\begin{singlespace}

\title {Ridge interpolators in correlated \emph{factor} regression models  -- exact risk analysis  %A tight variant of Gordon's escape through a mesh theorem
%\footnote{ This work was supported in
%part.}
}
\author{
\textsc{Mihailo Stojnic
\footnote{e-mail: {\tt flatoyer@gmail.com}} }}
\date{}
\maketitle

%%%%%%%%%%%%%%%%%%%%%%%%%%%%%%%%%%%%%%%%%%%%%%%%%%%%%%%%%%%%%%%%%%%%%%%%%%%%%%%%
%%%%%%%%%%%%%%%%%%%%%%%%%%%%%%%%%%%%%%%%%%%%%%%%%%%%%%%%%%%%%%%%%%%%%%%%%%%%%%%%
\centerline{{\bf Abstract}} \vspace*{0.1in}
%%%%%%%%%%%%%%%%%%%%%%%%%%%%%%%%%%%%%%%%%%%%%%%%%%%%%%%%%%%%%%%%%%%%%%%%%%%%%%%%
%%%%%%%%%%%%%%%%%%%%%%%%%%%%%%%%%%%%%%%%%%%%%%%%%%%%%%%%%%%%%%%%%%%%%%%%%%%%%%%%

We consider correlated \emph{factor} regression models (FRM) and analyze the performance of classical ridge interpolators. Utilizing powerful \emph{Random Duality Theory} (RDT) mathematical engine, we obtain \emph{precise} closed form characterizations of the underlying optimization problems and  all associated optimizing quantities. In particular, we provide \emph{excess prediction risk} characterizations that clearly show the dependence on all key model parameters, covariance matrices, loadings, and dimensions. As a function of the over-parametrization ratio, the generalized least squares (GLS) risk also exhibits the well known \emph{double-descent} (non-monotonic) behavior. Similarly to the classical linear regression models (LRM), we demonstrate that such FRM phenomenon can be smoothened out by the optimally tuned ridge regularization. The theoretical results are supplemented by numerical simulations and an excellent agrement between the two is observed. Moreover, we note that ``ridge smootenhing'' is often of limited effect already for over-parametrization ratios above $5$ and of virtually no effect for those above $10$. This solidifies the notion that one of the recently most popular neural networks paradigms -- \emph{zero-training (interpolating) generalizes well} -- enjoys wider applicability, including the one within the FRM estimation/prediction context.

\vspace*{0.25in} \noindent {\bf Index Terms: Factor regression model; Interpolators. Ridge estimators; Correlations}.

\end{singlespace}

%%%%%%%%%%%%%%%%%%%%%%%%%%%%%%%%%%%%%%%%%%%%%%%%%%%%%%%%%%%%%%%%%
%%%%%%%%%%%%%%%%%%%%%%%%%%%%%%%%%%%%%%%%%%%%%%%%%%%%%%%%%%%%%%%%%
%%%%%%%%%%%%%%%%%%%%%%%%%%%%%%%%%%%%%%%%%%%%%%%%%%%%%%%%%%%%%%%%%
%%%%%%%%%%%%%%%%%%%%%%%%%%%%%%%%%%%%%%%%%%%%%%%%%%%%%%%%%%%%%%%%%
%%%%%%%%%%%%%%%%%%%%%%%%%%%%%%%%%%%%%%%%%%%%%%%%%%%%%%%%%%%%%%%%%
%%%%%%%%%%%%%%%%%%%%%%%%%%%%%%%%%%%%%%%%%%%%%%%%%%%%%%%%%%%%%%%%%
\section{Introduction}
\label{sec:back}
%%%%%%%%%%%%%%%%%%%%%%%%%%%%%%%%%%%%%%%%%%%%%%%%%%%%%%%%%%%%%%%%%
%%%%%%%%%%%%%%%%%%%%%%%%%%%%%%%%%%%%%%%%%%%%%%%%%%%%%%%%%%%%%%%%%
%%%%%%%%%%%%%%%%%%%%%%%%%%%%%%%%%%%%%%%%%%%%%%%%%%%%%%%%%%%%%%%%%
%%%%%%%%%%%%%%%%%%%%%%%%%%%%%%%%%%%%%%%%%%%%%%%%%%%%%%%%%%%%%%%%%
%%%%%%%%%%%%%%%%%%%%%%%%%%%%%%%%%%%%%%%%%%%%%%%%%%%%%%%%%%%%%%%%%

Studying random structures parametric characterizations with  \emph{non-monotonic} behavior became very popular over the last decade. A substantial success in understanding and utilization of these phenomena is particularly associated with the training and testing of neural networks (NN). Observation of a somewhat surprising NN property that zero training error might not necessarily prevent excellent generalization abilities (see, e.g., \cite{BelkinHM18,BelkinMM18,ZhangBHRV21,ZBHRV17})  has been among the key motivating examples for a tone of recent studies in various scientific and engineering areas (see, e.g., \cite{SGDSBW19,BHMM19,HMRT22,Dicker16,BHX20,DobWag18}). For example, positioning such a seemingly contradictory phenomenon \cite{HastieFT01,HastieTF09} within a bias-variance tradeoff discussion, \cite{BHMM19} emphasized existence of a \emph{double-descent} NN generalization error behavior. An (already non-monotonic) U-shape error's dependence on the network size exhibited before the interpolating limit can in fact be followed by a \emph{second} descent after the limit (see, also \cite{ASS20,MuthukumarVSS20,NeyshaburTS14,ZhangBHRV21,ZBHRV17} and particularly  \cite{VCR89} for early double-decent displays). Such a positioning directly established a statistical interconnection that was substantially furthered in recent years.

Early connections were made with the classical  regression statistical models where similar double-descent behavior was empirically observed \cite{SGDSBW19,BHMM19} and additionally theoretically substantiated in e.g., \cite{MuthukumarVSS20,BHX20} (an introductory statistical learning context approach can  also be seen in, e.g., \cite{BRT19}). The appearance of \cite{HMRT22} is without a doubt one of the big breakthroughs that in a way skyrocketed the popularity of studying these phenomena. Choosing a particularly attractive linear regression model (LRM), \cite{HMRT22} focused on \emph{precise} studying of the behavior of several classical LRM associated linear estimators. Through such a studying it then highlighted that the presence of the non-monotonic prediction risk behavior is not only a feature of the complex models (such as generic or deep NNs) but also of very simple classical ones as well. Some of the \cite{HMRT22}'s results were already known (see, e.g., \cite{Dicker16,DobWag18}). However, reconnecting back to the NN and considering various other related models including the nonlinear ones, it basically emphasized a potential universality and usefulness of the type of behavior already present in simple models and exhibited by relatively simple algorithmic techniques. It also strengthened the belief that similar behavior might be happening in other well known models as well.

Our interest here is precisely in one of such well known models -- the factor regression model (FRM). After classical factor modeling considerations from the middle and second half of the last century (see, e.g., \cite{Jolliffe82,AndRub56,Joreskog67,Lawley40,Lawley41,Lawley43}), the last two decades have seen various forms of these models finding applications in estimating/predicting type of problems in a host of fields ranging from statistics, machine learning (ML), to many  branches of engineering and biotechnology (see, e.g., \cite{FanYM11,FanYM13,BhatDun11,StockW02,HCM113,BHPT06}). While the algorithmic techniques centered around principal component (PC) approaches  have been shown to have excellent properties \cite{FanYM11,FanYM13,BhatDun11,StockW02,HCM113,BHPT06,BaiNg02,BaiNg06} (especially in regimes with limited number of factors), the classical least squares based ones regained popularity and virtually out of nowhere came into the spotlight in the last several years in large part due to the above mentioned revamped interest essentially originating from studying interpolating zero-training neural networks properties.

After initial consideration of closely related latent space models in \cite{HMRT22}, \cite{BSMW22} studied the pure FRM (where both response variables and covariates are directly connected to the low-dimensional space) and for the GLS (generalized least squares) estimators (which in the over-parameterized regime correspond to the minimum norm ridge interpolators), observed a similar (non-monotonic) double-descent phenomenon and theoretically justified the existence of the scenarios where its associated excess risk can be smaller than the null one. It also showed that some of the considered scenarios allow that the GLS exhibits a full \emph{consistency} (vanishing excess risk). Similar results were obtained a bit earlier in \cite{BLT20} (see, also \cite{TB20}) for related models. These studies provided a strong advancement in understanding the estimation and prediction  within the FRM/LRM contexts, but are of the \emph{qualitative} type and as such are more focused on the confirmation of the existence of certain intriguing and popular phenomena (the double-descent and (large dimensional limiting) consistency being among the key ones). At the same time they are very different from the studies that focus on the so-called \emph{precise analyses}. For example, \cite{Dicker16,DobWag18} in the LRM context and \cite{HMRT22} in a bit wider context, utilized spectral random matrix (free probability) theory to obtain precise/exact characterizations of the excess risk as functions of the over- or under-parametrization ratio, thereby confirming and \emph{quantitatively describing} the underlying non-monotonicity  (appearance of singularities, ascents, descents, double-descents and so on). Within the scenarios where the utilized spectral methods techniques are applicable, they indeed produce excellent results. However, outside of such scenarios, they remain a bit limited, and potential alternatives are rather welcome.

In this paper, we present such a powerful alternative and do so while studying the pure factor regression models and their classical GLS and ridge estimators. The methodology is based on a completely different mathematical engine, called Random Duality Theory (RDT), that we developed in a long line of work \cite{StojnicRegRndDlt10,StojnicCSetam09,StojnicICASSP10var,StojnicISIT2010binary,StojnicGenLasso10,StojnicGardGen13,StojnicDiscPercp13,StojnicGorEx10,StojnicUpper10}.  The emphasis is on \emph{correlated} FRM and the RDT's ability to produce very \emph{precise} and generic analyses.

%%%%%%%%%%%%%%%%%%%%%%%%%%%%%%%%%%%%%%%%%%%%%%%%%%%%%%%%%%%%%%%%%
%%%%%%%%%%%%%%%%%%%%%%%%%%%%%%%%%%%%%%%%%%%%%%%%%%%%%%%%%%%%%%%%%
%%%%%%%%%%%%%%%%%%%%%%%%%%%%%%%%%%%%%%%%%%%%%%%%%%%%%%%%%%%%%%%%%
%%%%%%%%%%%%%%%%%%%%%%%%%%%%%%%%%%%%%%%%%%%%%%%%%%%%%%%%%%%%%%%%%
%%%%%%%%%%%%%%%%%%%%%%%%%%%%%%%%%%%%%%%%%%%%%%%%%%%%%%%%%%%%%%%%%
%%%%%%%%%%%%%%%%%%%%%%%%%%%%%%%%%%%%%%%%%%%%%%%%%%%%%%%%%%%%%%%%%
\section{Correlated factor regression models --- mathematical setup}
\label{sec:randlincons}
%%%%%%%%%%%%%%%%%%%%%%%%%%%%%%%%%%%%%%%%%%%%%%%%%%%%%%%%%%%%%%%%%
%%%%%%%%%%%%%%%%%%%%%%%%%%%%%%%%%%%%%%%%%%%%%%%%%%%%%%%%%%%%%%%%%
%%%%%%%%%%%%%%%%%%%%%%%%%%%%%%%%%%%%%%%%%%%%%%%%%%%%%%%%%%%%%%%%%
%%%%%%%%%%%%%%%%%%%%%%%%%%%%%%%%%%%%%%%%%%%%%%%%%%%%%%%%%%%%%%%%%
%%%%%%%%%%%%%%%%%%%%%%%%%%%%%%%%%%%%%%%%%%%%%%%%%%%%%%%%%%%%%%%%%
%%%%%%%%%%%%%%%%%%%%%%%%%%%%%%%%%%%%%%%%%%%%%%%%%%%%%%%%%%%%%%%%%

We consider the following classical (linear) pure factor regression model (FRM)
\begin{eqnarray}
\y=\bar{Z}\bar{\beta}+\e,\quad X=\bar{Z}L+\bar{E} \label{eq:model01}
\end{eqnarray}
where $X\in \mR^{m\times n}$ is a matrix of covariates or feature vectors, $\y\in\mR^m$ is a vector of response variables, and vector $\e\in\mR^m$ and matrix $\bar{E}\in\mR^{m\times n}$ characterize deviations from the idealized noiseless pure factor models (throughout the paper, unless clear from the context or explicitly stated otherwise, all vectors are assumed to be \emph{column vectors}). $\bar{Z}\in\mR^{m\times k}$ ($k<m$) represents a low-dimensional space of relevant factors and $L\in\mR^{k\times n}$ is the matrix of associated covariates/factors loadings. The model is well known and has been extensively used in a variety of fields, statistics, machine learning, econometrics, finance and so on. As such it needs no extensive introduction besides a few basic pointers. \textbf{\emph{(i)}} First, while it clearly resembles classical linear regression model (LRM), its a particularly  distinguishing feature is that both response variables $\y$ and covariates $X$ are directly connected to a low-dimensional space (here represented through the matrix of factors, $\bar{Z}$). In fact, if both $\e$ and $\bar{E}$ are zero, we have the \emph{idealized} FRM where both response variables and covariates reside exclusively in a common low-dimensional space $Z$. Assuming that $\e$ and $\bar{E}$ are significantly smaller (in magnitude) than $L$, model in (\ref{eq:model01}) effectively represents a realistic version of the idealized FRM. \textbf{\emph{(ii)}} Second, differently from the classical LRM, the factors $\bar{Z}$ are \emph{not} observable (within the LRM context, on the other hand, the factors would be viewed as the covariates $X$ and as such would be observable). In other words, the factors $\bar{Z}$ are hidden and the model is therefore often called \emph{hidden factor model} as well. \emph{(iii)} While it is clear that the chosen model is a classical FRM, it should also be noted that it is a bit different from similar ones studied throughout the literature. For example, the ones studied in \cite{BLT20,TB20} (and, in a way, revisited in \cite{BSMW22}) also in their essence rely on a connection to a lower dimensional space. However, a bit extra algebraic (and statistical) work is needed to make the precise connection. A very similar observation applies to the discussion of the latent space model in \cite{HMRT22} mentioned earlier.

%%%%%%%%%%%%%%%%%%%%%%%%%%%%%%%%%%%%%%%%%%%%%%%%%%%%%%%%%%%%%%%%%
%%%%%%%%%%%%%%%%%%%%%%%%%%%%%%%%%%%%%%%%%%%%%%%%%%%%%%%%%%%%%%%%%
%%%%%%%%%%%%%%%%%%%%%%%%%%%%%%%%%%%%%%%%%%%%%%%%%%%%%%%%%%%%%%%%%
\subsection{FRM versus LRM estimation}
\label{sec:frmlrm}
%%%%%%%%%%%%%%%%%%%%%%%%%%%%%%%%%%%%%%%%%%%%%%%%%%%%%%%%%%%%%%%%%
%%%%%%%%%%%%%%%%%%%%%%%%%%%%%%%%%%%%%%%%%%%%%%%%%%%%%%%%%%%%%%%%%
%%%%%%%%%%%%%%%%%%%%%%%%%%%%%%%%%%%%%%%%%%%%%%%%%%%%%%%%%%%%%%%%%

Besides the above key technical differences, there is actually a very relevant conceptual similarity between the FRM and the LRM. Namely, as the LRM, the FRM can be tightly connected to the training/testing data internal relations paradigm. In such a context, one takes $X$ and $\y$ as fully observable and views the rows of $X$ ($X_{i,:}, i=1,\dots,m$) as vectors of features that correspond to the response variables which are the elements of $\y$ ($\y_i,i=1,\dots,m$). Given the access to $m$ data pairs $(X_{i,:},\y_i)$. one would like to determine the common relation that presumably exists within each of them. Given a new data pair $(\x^{(t)},y^{(t)})$ and restricting to linearity, the estimation goal can then be summarized as
 \begin{eqnarray}
\bl{\textbf{Ultimate \underline{\emph{estimation}} goal:}} \quad \quad \quad \quad \mbox{Find}  & & \hat{\beta}\in\mR^n \nonumber \\
\mbox{such that} & &   \lp\x^{(t)}\rp^T\hat{\beta}\quad \mbox{is ``close'' to}\quad  y^{(t)}. \label{eq:model02}
\end{eqnarray}
One should note right here that, differently from the LRM,  $\hat{\beta}$ is not an estimate of $\bar{\beta}$. Moreover, their dimensions are even different, $\bar{\beta}$ is a $k$-dimensional whereas $\hat{\beta}$ is an $n$-dimensional vector. This precisely pinpoints the key difference between the LRM and the FRM estimation. Namely, while in  both, LRM and FRM, the estimators aim to provide linear relation between the response variables and covariates, in the FRM, the covariates are not factors themselves (as they are in the LRM) but rather admit a factor type linear decomposition given in (\ref{eq:model01}). As is clear from (\ref{eq:model01}) and (\ref{eq:model02}), $\bar{\beta}$ is the linear representation of $\y$ in factors $Z$ space, whereas $\hat{\beta}$ should be such a representation in covariates $X$ space.

One should also note that the estimation goal in (\ref{eq:model02}) is somewhat flexibly imprecise. It states that the aimed linear representation should be ``close'' to the response variable. Various measures of accuracy could be of interest depending on the area of application (for a similar consideration within a sparse regression approach see, e.g., \cite{StojnicGenLasso10,StojnicGenSocp10,StojnicPrDepSocp10}). We here consider the classical \emph{prediction risk} or, in the machine learning/neural networks terminology, the \emph{testing/generalization error}. Assuming a statistical context with random $X$ and $\e$, the prediction risk of an FRM type of estimator, $\hat{\beta}$, is formally defined as the following
\begin{eqnarray}
R_e(\bar{\beta},\hat{\beta})\triangleq \mE_{\x^{(t)}}\lp\lp\lp \x^{(t)}\rp^T\hat{\beta}- y^{(t)} \rp^2 |X,\y\rp, \label{eq:model03}
\end{eqnarray}
where, the \emph{testing} data pair $(\x^{(t)},y^{(t)})$ (testing vector of features $\x^{(t)}\in\mR^n$ and testing response variable $y^{(t)}$), are independent of the \emph{training} data pairs $(X_{i,:},\y_i)$. One should note that the expectation is over $\x^{(t)}$ (throughout the paper, the subscripts next to $\mE$ and $\mP$ denote the randomness to which they relate (when clear from the context, the subscript(s) are not specified)). This basically means that $R_e(\bar{\beta},\hat{\beta})$, as a function of $X$ and $\y$, is still a random quantity. Throughout the paper, we will find it convenient to work with the closely related, so-called, \emph{excess (prediction) risk} which we will introduce a bit later on. As studying the risk will be the central topic of this paper, two points particularly related to a line of work \cite{StojnicPrDepSocp10,StojnicGenLasso10,StojnicGenSocp10} should be emphasized: \textbf{\emph{1)}} In  \cite{StojnicPrDepSocp10,StojnicGenLasso10,StojnicGenSocp10} a (sparse) LRM context is studied where the main goal was the (\emph{in-sample}) estimation of the regression coefficients. As a measure of accuracy of the considered LASSO and SOCP estimators, the residual RMSE (root mean square error) was considered. Here a general (non-sparse) FRM context is studied where the predictive (\emph{out-of-sample}) abilities of the associated estimators are of interest. This effectively dictates a natural replacement of the RMSE by the risk from (\ref{eq:model03}). \emph{\textbf{2)}} The definition of the risk given in  (\ref{eq:model03}) clearly indicates that it depends on both $\bar{\beta}$ (the ``true'' $\y$ linear representation in the factors $Z$ space) and $\hat{\beta}$ (the estimated linear representations in the covariates $X$ space). Besides the natural difference between the meaning of $\bar{\beta}$ and $\hat{\beta}$, this also means that studying $\hat{\beta}$ falls into the category of, what \cite{StojnicPrDepSocp10} terms as, the \emph{problem dependent analyses}. It should be noted that, differently from  \cite{StojnicGenLasso10,StojnicGenSocp10} (which considered the usual problem data insensitive worst case context), \cite{StojnicPrDepSocp10} is the first problem dependent consideration within the \emph{precise} analyses of regression models. Moreover, within such a context, \cite{StojnicPrDepSocp10} uncovered the existence of a non-monotonicity (the U-shape RMSE functional dependence of the SNR) and effectively pioneered studying the appearance of such phenomena in precise analyses. As mentioned earlier, these phenomena would later on become very popular in the context of prediction risk dimensional dependence and associated double-descent (see, e.g., \cite{SGDSBW19,BHMM19,ZBHRV17,HMRT22,Dicker16,DobWag18,BHX20}).

%%%%%%%%%%%%%%%%%%%%%%%%%%%%%%%%%%%%%%%%%%%%%%%%%%%%%%%%%%%%%%%%%
%%%%%%%%%%%%%%%%%%%%%%%%%%%%%%%%%%%%%%%%%%%%%%%%%%%%%%%%%%%%%%%%%
%%%%%%%%%%%%%%%%%%%%%%%%%%%%%%%%%%%%%%%%%%%%%%%%%%%%%%%%%%%%%%%%%
\subsection{Key problem infrastructure features}
\label{sec:dse}
%%%%%%%%%%%%%%%%%%%%%%%%%%%%%%%%%%%%%%%%%%%%%%%%%%%%%%%%%%%%%%%%%
%%%%%%%%%%%%%%%%%%%%%%%%%%%%%%%%%%%%%%%%%%%%%%%%%%%%%%%%%%%%%%%%%
%%%%%%%%%%%%%%%%%%%%%%%%%%%%%%%%%%%%%%%%%%%%%%%%%%%%%%%%%%%%%%%%%

Before starting the technical analysis, we below first introduce the scenarios that will be considered and the associated analytical infrastructure. Problem dimensions, underlying statistics, and relevant estimators are the three architectural components of particular interest. We discuss all of them separately in the following subsections.

%%%%%%%%%%%%%%%%%%%%%%%%%%%%%%%%%%%%%%%%%%%%%%%%%%%%%%%%%%%%%%%%%
%%%%%%%%%%%%%%%%%%%%%%%%%%%%%%%%%%%%%%%%%%%%%%%%%%%%%%%%%%%%%%%%%
%%%%%%%%%%%%%%%%%%%%%%%%%%%%%%%%%%%%%%%%%%%%%%%%%%%%%%%%%%%%%%%%%
\subsubsection{\emph{Linearly} related dimensions}
\label{sec:dims}
%%%%%%%%%%%%%%%%%%%%%%%%%%%%%%%%%%%%%%%%%%%%%%%%%%%%%%%%%%%%%%%%%
%%%%%%%%%%%%%%%%%%%%%%%%%%%%%%%%%%%%%%%%%%%%%%%%%%%%%%%%%%%%%%%%%
%%%%%%%%%%%%%%%%%%%%%%%%%%%%%%%%%%%%%%%%%%%%%%%%%%%%%%%%%%%%%%%%%

We will consider large dimensional scenario with \emph{all} key underlying dimensions going to infinity. Within such a context, of particular interest is mathematically the most challenging, large $n$ \emph{linear} (or proportional) regime, where
\begin{eqnarray}
\alpha=\lim_{n\rightarrow \infty} \frac{m}{n}, \quad \kappa=\lim_{n\rightarrow \infty} \frac{k}{m}, \label{eq:model03a0}
\end{eqnarray}
and both $\alpha$ and $\kappa$ clearly remain constant as $n$ grows. In a way, $\alpha$ can be connected to what is in the classical LRM called the under-parametrization ratio (or the inverse of the over-parametrization ratio). Also, unless otherwise stated, all underlying matrices are assumed to be of full rank and we consider $\alpha\kappa<1$ context (the number of factors $k$ smaller than the number of features $n$). Even outside such regimes, the results that we will present hold with very minimal (or sometimes even without any) adjustments. However, constantly separating special cases would fill the presentation with a tone of unnecessary details that, besides slightly enhanced generality, would bring no novel conceptual insights. We instead opt to keep the presentation neater and, along the same lines, also often simplify writing by avoiding repeatedly using the $\lim_{n\rightarrow \infty}$ notation throughout the derivations. It will not be that difficult to see from the context that all expressions involving $n$ are assumed to be taken in the $\lim_{n\rightarrow \infty}$ sense.

%%%%%%%%%%%%%%%%%%%%%%%%%%%%%%%%%%%%%%%%%%%%%%%%%%%%%%%%%%%%%%%%%
%%%%%%%%%%%%%%%%%%%%%%%%%%%%%%%%%%%%%%%%%%%%%%%%%%%%%%%%%%%%%%%%%
%%%%%%%%%%%%%%%%%%%%%%%%%%%%%%%%%%%%%%%%%%%%%%%%%%%%%%%%%%%%%%%%%
\subsubsection{Underlying statistics}
\label{sec:stats}
%%%%%%%%%%%%%%%%%%%%%%%%%%%%%%%%%%%%%%%%%%%%%%%%%%%%%%%%%%%%%%%%%
%%%%%%%%%%%%%%%%%%%%%%%%%%%%%%%%%%%%%%%%%%%%%%%%%%%%%%%%%%%%%%%%%
%%%%%%%%%%%%%%%%%%%%%%%%%%%%%%%%%%%%%%%%%%%%%%%%%%%%%%%%%%%%%%%%%

To further strengthen the overall elegance of the exposition, a general Gaussian statistical scenario will be considered (it should though be kept in mind that the obtained results are universal and hold way more generally). Also, $\bar{\beta}$ in (\ref{eq:model01}) will be viewed as deterministic and fixed which will render $X$ and $\e$ as the only source of randomness. To make everything a bit more concrete, all three $\bar{Z}$, $\e$, and $\bar{E}$ are assumed to have Gaussian entries and, for three full rank (non-necessarily symmetric) matrices $A\in\mR^{n\times n}$, $\overline{A}\in\mR^{m\times m}$, and $\dbover{A}\in\mR^{n\times n}$ given in the following way
\begin{eqnarray}
\bar{Z}=ZA, \quad \quad \e=\sigma \overline{A}\v, \quad \quad \bar{E}=E\dbover{A}, \label{eq:model04}
\end{eqnarray}
where $Z\in\mR^{m\times n}$, $\v\in\mR^m$, and $E\in\mR^{n\times n}$ are independent of each other and have iid standard normal entries and $\sigma$ is a deterministic parameter independent of $n$ or any other problem dimension. Utilizing (\ref{eq:model04}), (\ref{eq:model01}) can be rewritten as
\begin{eqnarray}
\y=\bar{Z}\bar{\beta}+\e=ZA\bar{\beta}+\sigma\overline{A}\v, \quad  X=\bar{Z}L+\bar{E}=ZAL+E\dbover{A}. \label{eq:model05}
\end{eqnarray}
Clearly,  $\bar{Z}_{i:,}$ (the rows of factor matrix $\bar{Z}$) are independent of each other but the elements within each row are correlated among themselves. Moreover,the correlation matrix associated with each row is given as
\begin{eqnarray}
 \mE \bar{Z}_{i,:}^T\bar{Z}_{i,:}=A^TA. \label{eq:model06}
\end{eqnarray}
One then analogously has for the rows of matrix $\bar{E}$, $\bar{E}_{i,:}$
\begin{eqnarray}
 \mE \bar{E}_{i,:}^T\bar{E}_{i,:}=\dbover{A}^T\dbover{A}, \label{eq:model06a0}
\end{eqnarray}
and for the noise correlations
\begin{eqnarray}
\mE \e\e^T=\overline{A}\overline{A}^T. \label{eq:model07}
\end{eqnarray}
It will also be useful to consider the following SVDs of all three $A$, $\overline{A}$, and $\dbover{A}$
\begin{eqnarray}
A=U\Sigma V^T, \quad\overline{A}=\overline{U}\overline{\Sigma} \overline{V}^T
, \quad\dbover{A}=\dbover{U}\dbover{\Sigma} \dbover{V}^T. \label{eq:model08}
\end{eqnarray}
A combination of (\ref{eq:model05}) and (\ref{eq:model08}) easily gives
\begin{eqnarray}
\y=ZA\bar{\beta}+\sigma\overline{A}\v=ZU\Sigma V^T\bar{\beta}+\sigma \overline{U}\overline{\Sigma} \overline{V}^T\v, \label{eq:model09}
\end{eqnarray}
and
\begin{eqnarray}
 X= ZAL+E\dbover{A} =ZU\Sigma V^TL + E\dbover{U}\dbover{\Sigma} \dbover{V}^T. \label{eq:model09a0}
\end{eqnarray}
Utilizing the rotational symmetry of the standard normal entries of $Z$, $\v$, and $E$, (\ref{eq:model09}) and (\ref{eq:model09a0}) are statistically equivalent to the following
\begin{eqnarray}
\y=Z \Sigma V^T\bar{\beta}+\sigma \overline{U}\overline{\Sigma} \v,
\quad  X= Z\Sigma V^TL + E\dbover{\Sigma} \dbover{V}^T. \label{eq:model010}
\end{eqnarray}
In other words, one can basically view $\bar{Z}$, $\e$, and $\bar{E}$ as
\begin{eqnarray}
\bar{Z}=Z \Sigma V^T, \quad \e=\sigma\overline{U}\overline{\Sigma} \v, \quad \bar{E}= E\dbover{\Sigma} \dbover{V}^T. \label{eq:model010a0}
\end{eqnarray}
This also means that the model is statistically fully characterized by three iid standard normal objects $Z$, $\v$, and $E$, three fixed (deterministic) unitary matrices $V$, $\overline{U}$, and $\dbover{V}$, and three fixed (deterministic) diagonal matrices $\Sigma$, $\overline{\Sigma}$, and $\dbover{\Sigma}$. The above assumed statistics  allows for the existence of the cross-sectional (or inter-factors and inter-features) correlations. It also allows for intra-sample (or time-series) noise correlations. Taking $\overline{U}$ and $\overline{\Sigma}$ to be identity matrices easily removes these correlations. We keep them to ensure an easier parallelism with the already existing results.

As is typically the case when considering the predictive abilities of the estimators, unless otherwise stated, we assume that $\x^{(t)}$ (the testing features vector) has the same statistics as each of the rows of the training features matrix $X$ and that $y^{(t)}$ (the testing response variable) has the same statistics as any of the elements of the vector of the training response variables, $\y$. In more concrete terms, this means that we take
\begin{eqnarray}
\lp\x^{(t)}\rp^T=\lp\z^{(t)}\rp^T \Sigma V^TL+\lp \e^{(t)}\rp^T\dbover{\Sigma} \dbover{V}^T, \quad y^{(t)}=\lp\z^{(t)}\rp^T \Sigma V^T\bar{\beta}+\sigma \sqrt{\sum_{j=1}^{m}\lp\overline{U}_{ij}\overline{\Sigma}_{jj}\rp^2} v, \label{eq:model010a1}
\end{eqnarray}
where $\z^{(t)}\in\mR^k$, $\e^{(t)}\in\mR^{n}$, and $v\in\mR$ are comprised of iid standard normals and independent among themselves and of $X$ and $\y$ (basically, $X$ and $\e$, or $Z$, $\v$, and $E$). One should note that, due to intra-sample noise correlations,  $y^{(t)}$ is a function of $i$. This will however play no role in the risk analysis as we now for any chosen $i$ formally introduce the following \emph{excess risk} that will be the main subject of our study
\begin{eqnarray}
R(\bar{\beta},\hat{\beta}) & \triangleq &
R_e(\bar{\beta},\hat{\beta}) -\sigma ^2\lp \sum_{j=1}^{m}\lp\overline{U}_{ij}\overline{\Sigma}_{jj}\rp^2 \rp \nonumber \\
&= &
\mE_{\x^{(t)}}\lp\lp\lp \x^{(t)}\rp^T\hat{\beta}- y^{(t)} \rp^2 |X,\y\rp
-\sigma ^2\lp \sum_{j=1}^{m}\lp\overline{U}_{ij}\overline{\Sigma}_{jj}\rp^2 \rp \nonumber \\
&= &
\mE_{\z^{(t)},\e^{(t)},v}\lp\lp
\lp\z^{(t)}\rp^T \Sigma V^T \lp L\hat{\beta} - \bar{\beta} \rp  +   \lp \e^{(t)}\rp^T\dbover{\Sigma} \dbover{V}^T
 \hat{\beta}- \sigma \sqrt{\sum_{j=1}^{m}\lp\overline{U}_{ij}\overline{\Sigma}_{jj}\rp^2} v \rp^2 |X,\y\rp \nonumber \\
 & &
-\sigma ^2\lp \sum_{j=1}^{m}\lp\overline{U}_{ij}\overline{\Sigma}_{jj}\rp^2 \rp \nonumber \\
&= & \left  \|\Sigma V^T\lp \bar{\beta} -L\hat{\beta} \rp\right \|_2^2
+ \left \|\dbover{\Sigma} \dbover{V}^T \hat{\beta}\right \|_2^2, \label{eq:model010a1a0}
\end{eqnarray}
where we utilized
\begin{eqnarray}
 \mE V\Sigma \z^{(t)} \lp\z^{(t)}\rp^T \Sigma V^T = V\Sigma\Sigma V^T, \quad
  \mE \dbover{V}\dbover{\Sigma} \e^{(t)} \lp\e^{(t)}\rp^T \dbover{\Sigma} \dbover{V}^T = \dbover{V}\dbover{\Sigma}\dbover{\Sigma}\dbover{V}^T, \label{eq:model010a2}
\end{eqnarray}
and independence of $\z^{(t)},\e^{(t)}$ , and $v$. Also, as having both $\sigma$ and magnitude of $\bar{\beta}$ simultaneously vary brings no conceptual advantage over having one of them fixed, we scale one of them to one. For example, unless otherwise stated, we take $\sigma$ as potentially varying and $\|\bar{\beta}\|_2=1$.

%%%%%%%%%%%%%%%%%%%%%%%%%%%%%%%%%%%%%%%%%%%%%%%%%%%%%%%%%%%%%%%%%
%%%%%%%%%%%%%%%%%%%%%%%%%%%%%%%%%%%%%%%%%%%%%%%%%%%%%%%%%%%%%%%%%
%%%%%%%%%%%%%%%%%%%%%%%%%%%%%%%%%%%%%%%%%%%%%%%%%%%%%%%%%%%%%%%%%
\subsubsection{Classical linear estimators}
\label{sec:estimators}
%%%%%%%%%%%%%%%%%%%%%%%%%%%%%%%%%%%%%%%%%%%%%%%%%%%%%%%%%%%%%%%%%
%%%%%%%%%%%%%%%%%%%%%%%%%%%%%%%%%%%%%%%%%%%%%%%%%%%%%%%%%%%%%%%%%
%%%%%%%%%%%%%%%%%%%%%%%%%%%%%%%%%%%%%%%%%%%%%%%%%%%%%%%%%%%%%%%%%

Many estimators $\hat{\beta}$ could be of interest. Depending on the main objectives and available resources to achieve them some would be preferable over others. The performance versus computational complexity is usually the key tradeoff that ultimately determines the practically usability of the estimator. We here consider there classical ``low complexity'' estimators: \textbf{\emph{1)}} Minimum $\ell_2$ norm interpolator (we will often refer to it as GLS (or the generalized least squares); \textbf{\emph{2)}} Ridge regression estimator (we will often refer to it as ridge estimator); and \textbf{\emph{3)}} Basic least squares (or just LS). These estimators are particularly tailored/useful for certain dimensional regimes. For example, the GLS can be used only in the so-called over-parameterized regime, i.e., when $n>m$ (or when $\alpha<1$). Exactly opposite of the GLS, the plain LS makes sense only in the under-parameterized regimes where $\alpha>1$ (in a way, one can basically think of the LS as being the under-parameterized adaptation of the GLS). On the other hand, the ridge estimator is a ridge regularized version of the least squares and differently from both non-regularized variants, LS and GLS, it can be used in any dimensional regime. More concrete/technical relations between these estimators are also well known and we will revisit some of them throughout the presentation. To formally introduce the estimators, we parallel their corresponding LRM introduction (see also, e.g., \cite{Stojnicridge24}).

\underline{\textbf{\emph{1) Minimum $\ell_2$ norm interpolator (over-parameterized regime, $\alpha<1$):}}}
\begin{eqnarray}
\beta_{gls}\triangleq \mbox{arg}\min_{\beta} & &  \|\beta\|_2^2 \nonumber \\
\mbox{subject to} & & X\beta=\y. \label{eq:model011}
\end{eqnarray}
One also easily observes the following closed form solution
\begin{eqnarray}
\beta_{gls} = X^T(XX^T)^{-1}\y, \label{eq:model012}
\end{eqnarray}
or, alternatively,
\begin{eqnarray}
\beta_{gls} = (X^TX)^{-1}X^T\y, \label{eq:model013}
\end{eqnarray}
where $(X^TX)^{-1}$ is the pseudo-inverse of $X^TX$ (only the positive eigenvalues are inverted).

\underline{\textbf{\emph{2) Ridge estimator (any $\alpha$):}}} After formally defining the ridge regression estimator (for a given fixed $\lambda> 0$) as
\begin{eqnarray}
\beta_{rr}(\lambda)\triangleq \mbox{arg}\min_{\beta}  \lambda\|\beta\|_2^2 +\frac{1}{m}\|\y-X\beta\|_2^2. \label{eq:model014}
\end{eqnarray}
one easily obtains its following closed form
 \begin{eqnarray}
\beta_{rr}(\lambda) = \frac{1}{m}\lp \lambda I + \frac{1}{m} X^TX \rp^{-1}X^T \y= \lp \lambda m I +  X^TX \rp^{-1}X^T \y. \label{eq:model015}
\end{eqnarray}
it is well known (and not that difficult to see by simply comparing (\ref{eq:model013}) and (\ref{eq:model015})) that the GLS estimator is a special case of the ridge regression one, obtained when $\lambda\rightarrow 0$. In other words, we have
 \begin{eqnarray}
\beta_{gls}=\lim_{\lambda\rightarrow 0} \beta_{rr}(\lambda). \label{eq:model016}
\end{eqnarray}

\underline{\textbf{\emph{3) Plain least squares (under-parameterized regime, $\alpha>1$):}}} Quite likely the most well known (linear or nonlinear) estimator is formally defined as the solution of the least square opptmization
\begin{eqnarray}
\beta_{ls}\triangleq \mbox{arg}\min_{\beta} \frac{1}{m}\|\y-X\beta\|_2^2. \label{eq:model017}
\end{eqnarray}
which trivially admits the following closed form solution
\begin{eqnarray}
\beta_{ls} = (X^TX)^{-1}X^T\y. \label{eq:model018}
\end{eqnarray}
It is interesting to note that in the under-parameterized regime the above ridge estimator definition can be extended to $\lambda\geq 0$ which then gives
\begin{eqnarray}
\beta_{ls}= \beta_{rr}(0). \label{eq:model019}
\end{eqnarray}

 Each of the estimators, (\ref{eq:model013}), (\ref{eq:model015}), and (\ref{eq:model018}) depends on the training data $(X,\y)$. Under the above assumed training/testing  statistics of $(X,\y)$ and $(\x^{(t)},y^{(t)})$, we, in the rest of the paper, provide a precise statistical characterization of the (random) excess prediction risk
\begin{eqnarray}
R(\bar{\beta},\hat{\beta})
&= & \left  \|\Sigma V^T\lp \bar{\beta} -L\hat{\beta} \rp\right \|_2^2
+ \left \|\dbover{\Sigma} \dbover{V}^T \hat{\beta}\right \|_2^2, \label{eq:model020}
\end{eqnarray}
where, depending on the dimensional regime of interest, $\hat{\beta}$ will be one of $\beta_{gls}$, $\beta_{rr}$, and $\beta_{ls}$. The over-parameterized regime and the GLS estimator will be of our primary interest.

%%%%%%%%%%%%%%%%%%%%%%%%%%%%%%%%%%%%%%%%%%%%%%%%%%%%%%%%%%%%%%%%%
%%%%%%%%%%%%%%%%%%%%%%%%%%%%%%%%%%%%%%%%%%%%%%%%%%%%%%%%%%%%%%%%%
%%%%%%%%%%%%%%%%%%%%%%%%%%%%%%%%%%%%%%%%%%%%%%%%%%%%%%%%%%%%%%%%%
%%%%%%%%%%%%%%%%%%%%%%%%%%%%%%%%%%%%%%%%%%%%%%%%%%%%%%%%%%%%%%%%%
%%%%%%%%%%%%%%%%%%%%%%%%%%%%%%%%%%%%%%%%%%%%%%%%%%%%%%%%%%%%%%%%%
%%%%%%%%%%%%%%%%%%%%%%%%%%%%%%%%%%%%%%%%%%%%%%%%%%%%%%%%%%%%%%%%%
\subsection{Related literature and our contributions}
\label{sec:priorwork}
%%%%%%%%%%%%%%%%%%%%%%%%%%%%%%%%%%%%%%%%%%%%%%%%%%%%%%%%%%%%%%%%%
%%%%%%%%%%%%%%%%%%%%%%%%%%%%%%%%%%%%%%%%%%%%%%%%%%%%%%%%%%%%%%%%%
%%%%%%%%%%%%%%%%%%%%%%%%%%%%%%%%%%%%%%%%%%%%%%%%%%%%%%%%%%%%%%%%%
%%%%%%%%%%%%%%%%%%%%%%%%%%%%%%%%%%%%%%%%%%%%%%%%%%%%%%%%%%%%%%%%%
%%%%%%%%%%%%%%%%%%%%%%%%%%%%%%%%%%%%%%%%%%%%%%%%%%%%%%%%%%%%%%%%%

%\cite{Dicker16,DobWag18,BSMW22,BBSMW21,BHX20,BHMM19,BRT19,BMR21,BLT20,MeiMon22,RMR20,WuXu20,XMRH21,LRZ20,LiR21,ASS20}

As mentioned in the introduction, the last decade has seen a rapidly growing interest in studying the non-monotonic random structures behavior (see, e.g., \cite{SGDSBW19,BHMM19,ZBHRV17,HMRT22,Dicker16,DobWag18,BHX20}). The interpolating/zero-training neural networks generalizing properties particularly contributed to it. New research directions have been further branched out of these foundational considerations with many of them interconnecting machine learning, statistical, and various scientific and engineering communities. We here briefly revisit a few interesting ones (some of them we have already discussed in the introduction as well).

Building further on the initial interpolation observations, an interesting and strong connection to machine learning kernel based methods has been developed as well \cite{SGDSBW19,BHMM19,NeyshaburTS14,ZhangBHRV21,ZBHRV17}. Connection between implicit regularization and the problem architecture (studied earlier  within deep learning and matrix factorization contexts in \cite{NeyshaburTS14,GWBNS17}) was considered in \cite{LiR21}. A kernel ridgeless regression context was discussed and the generalization error was uncovered to exhibit (as a function of spectral decay) both monotonic and  U-shape (non-monotonic) behavior. The appearance of multiple descents and an even more non-monotonic behavior was observed in \cite{LRZ20}. Of particularly strong interest over the last several years has been a rather surprising interpolation -- training NN algorithms connection \cite{AliKT19,ChizatB18,LiL18a,OymSol19,ZCZG18,ADHLW19,DuZPS19,JGH18,LXSBNSP20,DuLL0Z19}. Namely, a strong effort has been put forth and excellent results have been achieved in understanding how/when (in the over-parameterized regimes) optimizing algorithms (mostly based on the gradient descent) achieve the perfect fitting and how/why it relates to the optimal ridge regression becoming interpolation. It is interesting to note that this also relates to and, in a way, complements some classical machine learning/regression cross-considerations from a couple of decades ago \cite{CuckerS02,YRC07}.

Besides the above mentioned developments that are more ML/NN oriented, a large body of related work leans more on the statistical side (and as such is more closely related to our work). They, to a large degree, relate to the classical regression models, i.e., LRM. For example, \cite{BHX20}  study a LRM setup as a way of providing a simple model where the double-descent phenomenon appears. Considering an isotropic (uncorrelated) standard normal context, it extends to the over-parameterized regimes classical results of \cite{BF83} and indeed observes the appearance of the double-descent phenomenon. Moreover, as a further upgrade, it maintains a bit of a stronger connection with ML/NN as well, and also studies the random features model of \cite{RahimiR07}. Relying on the discrete Fourier transform (DFT) matrices empirical distributions  (see, e.g., \cite{Farrell11}), it again demonstrates the existence of the double-descent phenomenon. In a somewhat related line of work,  \cite{MeiMon22} \emph{precisely} analyzes nonlinear random features models of \cite{RahimiR07} and considers statistical universality (which was also discussed in \cite{MRSY19,HuL23}). Following the initial spectral methods considerations \cite{HMRT22,Dicker16,DobWag18,BHX20}, further studies appeared \cite{RMR20,WuXu20} (see also, e.g., \cite{XMRH21}). Their contributions are more of  a technical nature. For example, as \cite{HMRT22}, \cite{RMR20,WuXu20} also rely on the utilization of the spectral random matrix theory. They however pursue a slightly different approach (more akin to \cite{LP11}) and consequently formulate results a bit differently (\cite{HMRT22}'s initial results assumed a random $\bar{\beta}$ which, in a way, resembles a corresponding statistical technical assumption on $\bar{\beta}$  needed in \cite{RMR20}). Despite important technical differences, the key practical conceptual points/results stated in \cite{HMRT22} are maintained in both \cite{RMR20,WuXu20}.

Studying FRM has been popular in recent years as well. However, as mentioned in the introduction, these models have quite a long history that dates back at least to the  middle of the last century and Lawley's initial factor formulations \cite{Lawley40,Lawley41,Lawley43}. These early considerations were then considered in a more systematic way in classical \cite{AndRub56} (for further maximum-likelihood, principal components, and covariance estimations connections that naturally followed in the second half of the last century, see, e.g., \cite{Jolliffe82,Joreskog67,Joreskog70}). Many of these basic concepts have also been revisited over the last a couple of decades with the use of modern mathematical tools. For example, covariance estimations and principal components have been thoroughly studied and excellent results have been obtained \cite{FanYM11,FanYM13,BhatDun11,StockW02,HCM113,BHPT06,BaiNg02,BaiNg06}. The main takeaways are usually that if the number of factors is significantly smaller than the underlying dimensions, the PC based techniques should work well and might also be competitive with more advanced (and often more complicated and computationally expensive) techniques. As one of the key factors in determining the usefulness of an estimator or predictor is the performance-computational complexity tradeoff, understanding the ultimate power of the computationally simpler strategies is of great interest. That in a way explains the strength and popularity of the methods similar to those from, say, \cite{FanYM11,FanYM13,BhatDun11,StockW02,HCM113,BHPT06,BaiNg02,BaiNg06}. Keeping this in mind and pairing it with the earlier mentioned growing interpolation popularity within the NN context naturally provided an ambient for a revival of the classical least squares estimating techniques and particularly their interpolating variants. As mentioned earlier and as is natural, this has been first extensively studied within the LRM context.
After \cite{HMRT22} extended it to a latent space model context (which is closely related to the FRM), \cite{BSMW22} studied the pure FRM. In the pure FRM, both response variables and covariates are predicated to be directly connected to a low-dimensional space.
\cite{BSMW22} considers the GLS estimators and for the interpolating (over-parameterizing) regime under certain structural dimensions/data covariance related assumptions bounds the excess risk and even proves consistency. It also experimentally shows that the GLS estimator (the interpolator in the over-parameterized and the plain LS in the under-parameterized regime) exhibits a double-descent phenomenon. The results obtained in \cite{BSMW22} are, in a way, generalized in \cite{BBSMW21}. A more general variant of the estimator is considered which, as special cases, includes both the GLS and the PCR (the principal component regression based one). Results similar in flavor to the ones from \cite{BSMW22} are obtained
and similar risk bounding and consistency properties are shown again to hold for the general estimator (and then automatically for all its special cases as well). Models similar to FRM were considered in \cite{TB20,BLT20}. In particular, motivated by the experimental deep learning excellent over-fitting generalization abilities \cite{BelkinMM18,BelkinHM18}, \cite{BLT20}  (as \cite{BSMW22,BBSMW21}) shows that a mild over-parametrization can result in a excess risk being comparable to the corresponding one of the best theoretical benchmarks.  Differently from \cite{BSMW22,BBSMW21}, it relies on a specific structure of the inter-features covariances, defines the so-called \emph{benign} ones together with the associated \emph{effective rank} concepts, and then requires that such ranks are properly related to the the sample size $m$ and the number of features $n$.  As mentioned earlier, these studies put a focus on the  \emph{qualitative} type of analysis where the emphasis is on the confirmation of the existence of certain phenomena (the double-descent, consistency and so on).

 As explained in the mathematical setup (see Section \ref{sec:randlincons}) and the discussion in the introduction, we study the pure correlated FRM and the three associated estimators introduced in Section \ref{sec:estimators}. Mathematically, our model is identical to (or slightly more general than) the one considered in  \cite{BBSMW21,BSMW22} (although related to and resembling of, it is, however, a bit different from the latent space one of \cite{HMRT22} and a bit more from the one used in \cite{BLT20} or the plain LRM ones of, say, \cite{HMRT22,Dicker16,DobWag18,BHX20})). Differently from
 \cite{BBSMW21,BSMW22} (and \cite{BLT20,TB20} as well), we allow for all underlying dimensions to be linearly related to each other and impose no restrictions on the types of correlations.  We here utilize a powerful mathematical engine, called Random Duality Theory (RDT) (see, e.g., \cite{Stojnicridge24,StojnicRegRndDlt10,StojnicCSetam09,StojnicICASSP10var,StojnicISIT2010binary,StojnicGenLasso10,StojnicGardGen13,StojnicDiscPercp13,StojnicGorEx10,StojnicUpper10}) and demonstrate what role each of the main RDT principles plays in the analysis of the optimization programs
(\ref{eq:model011}), (\ref{eq:model014}), and (\ref{eq:model017}) that ultimately produce the three estimators. In particular, the RDT machinery enables us to determine very \emph{precisely} \emph{all} optimizing quantities associated with these programs. As a function of the over-parametrization ratio $\frac{1}{\alpha}$, the GLS excess prediction risk is shown to
 exhibit the double-descent phenomenon (with a potential for exhibiting the U-shape behavior in the interpolating regime ($\frac{1}{\alpha}>1$)).
We demonstrate that optimally tuned ridge regularization does smoothen such a behavior. The obtained results, on the other hand, also indicate that the ``ridge smoothening''  effect is rather marginal already for the over-parametrization ratios $\frac{1}{\alpha}\geq 5$ and practically nonexistent for the ratios $\frac{1}{\alpha}\geq 10$. This basically confirms that the simple interpolating method (which requires no tuning) can in fact be at no disadvantage compared to more complex/complicated strategies. We also, show how the obtained results simplify for the uncorrelated models.

%%%%%%%%%%%%%%%%%%%%%%%%%%%%%%%%%%%%%%%%%%%%%%%%%%%%%%%%%%%%%%%%%
%%%%%%%%%%%%%%%%%%%%%%%%%%%%%%%%%%%%%%%%%%%%%%%%%%%%%%%%%%%%%%%%%
%%%%%%%%%%%%%%%%%%%%%%%%%%%%%%%%%%%%%%%%%%%%%%%%%%%%%%%%%%%%%%%%%
%%%%%%%%%%%%%%%%%%%%%%%%%%%%%%%%%%%%%%%%%%%%%%%%%%%%%%%%%%%%%%%%%
%%%%%%%%%%%%%%%%%%%%%%%%%%%%%%%%%%%%%%%%%%%%%%%%%%%%%%%%%%%%%%%%%
%%%%%%%%%%%%%%%%%%%%%%%%%%%%%%%%%%%%%%%%%%%%%%%%%%%%%%%%%%%%%%%%%
\section{Precise excess risk analysis}
\label{sec:analuncorr}
%%%%%%%%%%%%%%%%%%%%%%%%%%%%%%%%%%%%%%%%%%%%%%%%%%%%%%%%%%%%%%%%%
%%%%%%%%%%%%%%%%%%%%%%%%%%%%%%%%%%%%%%%%%%%%%%%%%%%%%%%%%%%%%%%%%
%%%%%%%%%%%%%%%%%%%%%%%%%%%%%%%%%%%%%%%%%%%%%%%%%%%%%%%%%%%%%%%%%
%%%%%%%%%%%%%%%%%%%%%%%%%%%%%%%%%%%%%%%%%%%%%%%%%%%%%%%%%%%%%%%%%
%%%%%%%%%%%%%%%%%%%%%%%%%%%%%%%%%%%%%%%%%%%%%%%%%%%%%%%%%%%%%%%%%
%%%%%%%%%%%%%%%%%%%%%%%%%%%%%%%%%%%%%%%%%%%%%%%%%%%%%%%%%%%%%%%%%

In this section we statistically analyze the excess risk from (\ref{eq:model020}) for the above introduced three estimators, $\beta_{gls}$, $\beta_{rr}$, and $\beta_{ls}$. The discussion of each of the estimators is conducted in a separate subsection. As our primary interest is in the over-parameterized regime (which has gained a lot popularity in machine learning in recent years), we start the technical analysis by presenting the relevant GLS considerations. To smoothen the presentation we proceed by systematically discussing all the key RDT steps. At the same time, we often skip unnecessary repetitions of sufficiently similar concepts already discussed in deep detail in our prior/concurrent RDT considerations. Instead the focus is on key differences. Along the same lines, to prevent overloading  the presentation with a tone of unnecessary tiny details that bring no conceptual insights, we assume that all the quantities in the derivations below are  (deterministically  or randomly) bounded.

%%%%%%%%%%%%%%%%%%%%%%%%%%%%%%%%%%%%%%%%%%%%%%%%%%%%%%%%%%%%%%%%%
%%%%%%%%%%%%%%%%%%%%%%%%%%%%%%%%%%%%%%%%%%%%%%%%%%%%%%%%%%%%%%%%%
%%%%%%%%%%%%%%%%%%%%%%%%%%%%%%%%%%%%%%%%%%%%%%%%%%%%%%%%%%%%%%%%%
%%%%%%%%%%%%%%%%%%%%%%%%%%%%%%%%%%%%%%%%%%%%%%%%%%%%%%%%%%%%%%%%%
%%%%%%%%%%%%%%%%%%%%%%%%%%%%%%%%%%%%%%%%%%%%%%%%%%%%%%%%%%%%%%%%%
%%%%%%%%%%%%%%%%%%%%%%%%%%%%%%%%%%%%%%%%%%%%%%%%%%%%%%%%%%%%%%%%%
\subsection{Excess risk of the GLS interpolator}
\label{sec:analgls}
%%%%%%%%%%%%%%%%%%%%%%%%%%%%%%%%%%%%%%%%%%%%%%%%%%%%%%%%%%%%%%%%%
%%%%%%%%%%%%%%%%%%%%%%%%%%%%%%%%%%%%%%%%%%%%%%%%%%%%%%%%%%%%%%%%%
%%%%%%%%%%%%%%%%%%%%%%%%%%%%%%%%%%%%%%%%%%%%%%%%%%%%%%%%%%%%%%%%%
%%%%%%%%%%%%%%%%%%%%%%%%%%%%%%%%%%%%%%%%%%%%%%%%%%%%%%%%%%%%%%%%%
%%%%%%%%%%%%%%%%%%%%%%%%%%%%%%%%%%%%%%%%%%%%%%%%%%%%%%%%%%%%%%%%%
%%%%%%%%%%%%%%%%%%%%%%%%%%%%%%%%%%%%%%%%%%%%%%%%%%%%%%%%%%%%%%%%%

As stated earlier, the analytical goal is a precise characterization of the optimization program (\ref{eq:model011}) and all of its associated quantities. The excess risk from (\ref{eq:model020}) is among them as well. Let the objective of the program (\ref{eq:model011}) be formally defined as
\begin{eqnarray}
\xi_{gls} \triangleq \min_{\beta} & &  \|\beta\|_2^2 \nonumber \\
\mbox{subject to} & & X\beta=\y. \label{eq:randlincons1}
\end{eqnarray}
Utilizing (\ref{eq:model010}) and (\ref{eq:model010a0}) one can write the following statistical equivalent to (\ref{eq:randlincons1})
\begin{eqnarray}
\xi_{gls} \triangleq \min_{\beta} & &  \|\beta\|_2^2 \nonumber \\
\mbox{subject to} & & Z\Sigma V^T L\beta +E\dbover{\Sigma}\dbover{V}^T\beta=Z\Sigma V^T\bar{\beta}+\sigma \overline{U}\overline{\Sigma}\v. \label{eq:randlincons1a0}
\end{eqnarray}
Writing the Lagrangian further gives
\begin{eqnarray}
\xi_{gls} = \min_{\beta} \max_{\nu} & & \lp \|\beta\|_2^2 +\nu^T \lp Z\Sigma V^TL\beta-Z\Sigma V^T\bar{\beta}
+E\dbover{\Sigma}\dbover{V}^T\beta
-\sigma \overline{U}\overline{\Sigma}\v \rp\rp. \label{eq:randlincons2}
\end{eqnarray}
The above holds generically, i.e., for any $V$, $\overline{U}$, $\dbover{V}$, $\Sigma$, $\overline{\Sigma}$, $\dbover{\Sigma}$, $\sigma$, and $\bar{\beta}$ and any (deterministic or random) $Z$, $\v$, $E$. Our particular interest here is in random instances and the following conveniently summarizes the \emph{analytical goal} related to the probabilistic characterization of $\xi_{gls}$ and effectively complements the corresponding estimation one from (\ref{eq:model02}).
 \begin{eqnarray}
\mbox{\bl{\textbf{Ultimate \underline{\emph{analytical}} goal:}}}
\quad \mbox{Given:}  && \alpha  =    \lim_{n\rightarrow \infty} \frac{m}{n}\in(1,\infty) \quad \mbox{and}\quad  V, \overline{U},\dbover{V},\Sigma,\overline{\Sigma},\dbover{\Sigma},\sigma,\bar{\beta}  \nonumber \\
\quad \mbox{find}  & & \xi_{gls}^{(opt)}\nonumber \\
\mbox{such that}  && \forall \epsilon>0, \quad \lim_{n\rightarrow\infty}\mP_{Z,\v,E}\lp (1-\epsilon)\xi_{gls}^{(opt)}  \leq \xi_{gls} \leq (1+\epsilon)\xi_{gls}^{(opt)} \rp\longrightarrow 1.\nonumber \\
  \label{eq:ex4}
\end{eqnarray}
To achieve the above goal we utilize the path traced by the the powerful Random Duality Theory (RDT) mathematical engine.

%%%%%%%%%%%%%%%%%%%%%%%%%%%%%%%%%%%%%%%%%%%%%%%%%%%%%%%%%%%%%%%%%
%%%%%%%%%%%%%%%%%%%%%%%%%%%%%%%%%%%%%%%%%%%%%%%%%%%%%%%%%%%%%%%%%
%%%%%%%%%%%%%%%%%%%%%%%%%%%%%%%%%%%%%%%%%%%%%%%%%%%%%%%%%%%%%%%%%
%%%%%%%%%%%%%%%%%%%%%%%%%%%%%%%%%%%%%%%%%%%%%%%%%%%%%%%%%%%%%%%%%
%%%%%%%%%%%%%%%%%%%%%%%%%%%%%%%%%%%%%%%%%%%%%%%%%%%%%%%%%%%%%%%%%
%%%%%%%%%%%%%%%%%%%%%%%%%%%%%%%%%%%%%%%%%%%%%%%%%%%%%%%%%%%%%%%%%
\subsubsection{GLS excess risk via RDT}
\label{sec:randlinconsrdt}
%%%%%%%%%%%%%%%%%%%%%%%%%%%%%%%%%%%%%%%%%%%%%%%%%%%%%%%%%%%%%%%%%
%%%%%%%%%%%%%%%%%%%%%%%%%%%%%%%%%%%%%%%%%%%%%%%%%%%%%%%%%%%%%%%%%
%%%%%%%%%%%%%%%%%%%%%%%%%%%%%%%%%%%%%%%%%%%%%%%%%%%%%%%%%%%%%%%%%
%%%%%%%%%%%%%%%%%%%%%%%%%%%%%%%%%%%%%%%%%%%%%%%%%%%%%%%%%%%%%%%%%
%%%%%%%%%%%%%%%%%%%%%%%%%%%%%%%%%%%%%%%%%%%%%%%%%%%%%%%%%%%%%%%%%
%%%%%%%%%%%%%%%%%%%%%%%%%%%%%%%%%%%%%%%%%%%%%%%%%%%%%%%%%%%%%%%%%

It is useful to briefly recall on the main RDT principles (for a more complete overview, see,  \cite{StojnicRegRndDlt10,StojnicCSetam09,StojnicICASSP10var,StojnicISIT2010binary,StojnicGenLasso10,StojnicGardGen13,StojnicDiscPercp13,StojnicGorEx10,StojnicUpper10}).

\vspace{-.0in}\begin{center}
%\tcbset{colback=orange!40!white!40!yellow,colframe=blue!75!black,fonttitle=\bfseries,title style={left color=black, right color=cyan},interior %style={left color=yellow!10!white,right color=yellow!80!white}}
 	\tcbset{beamer,lower separated=false, fonttitle=\bfseries, coltext=black ,
		interior style={top color=yellow!20!white, bottom color=yellow!60!white},title style={left color=black!80!purple!60!cyan, right color=yellow!80!white},
		width=(\linewidth-4pt)/4,before=,after=\hfill,fonttitle=\bfseries}
% 	\tcbset{beamer,lower separated=false, fonttitle=\bfseries, coltext=black ,
%		interior style={top color=yellow!20!white, bottom color=yellow!60!white},title style={left color=black, right color=red!50!blue!60!white},
%		width=(\linewidth-4pt)/4,before=,after=\hfill,fonttitle=\bfseries}
 \begin{tcolorbox}[beamer,title={\small Summary of the RDT's main principles} \cite{StojnicCSetam09,StojnicRegRndDlt10}, width=1\linewidth]
\vspace{-.15in}
{\small \begin{eqnarray*}
 \begin{array}{ll}
\hspace{-.19in} \mbox{1) \emph{Finding underlying optimization algebraic representation}}
 & \hspace{-.0in} \mbox{2) \emph{Determining the random dual}} \\
\hspace{-.19in} \mbox{3) \emph{Handling the random dual}} &
 \hspace{-.0in} \mbox{4) \emph{Double-checking strong random duality.}}
 \end{array}
  \end{eqnarray*}}
\vspace{-.2in}
 \end{tcolorbox}
\end{center}\vspace{-.0in}
We below proceed by separately discussing each of the four principles in full detail and uncover how they can be formulated to particularly relate to the problems of interest here. All key results (simple or more complicated) are formalized as lemmas and theorems.

\vspace{.1in}

\noindent \underline{1) \textbf{\emph{Algebraic characterization:}}}  The above $\xi_{gls}$ related algebraic discussions are conveniently  summarized in the following lemma (for analogous LRM ones, see \cite{Stojnicridge24}).

\begin{lemma}(Algebraic optimization representation) Let  $V\in\mR^{n\times n}$, $\overline{U}\in\mR^{m\times m}$, and $\dbover{V}\in\mR^{n\times n}$  be three given unitary (orthogonal) matrices and let $\Sigma\in\mR^{n\times n}$, $\overline{\Sigma}\in\mR^{m\times m}$, and $\dbover{\Sigma}\in\mR^{n\times n}$  be three given diagonal positive definite matrices. Also, let matrix $L\in\mR^{k\times n}$, vector  $\bar{\beta}\in\mR^n$, and scalar  $\sigma\geq 0$ be given as well. Assume that the components of any of these given objects are fixed real numbers that do not change as $n\rightarrow\infty$ and, for (possibly random) matrix $Z\in\mR^{m\times n}$, vector $\v\in\mR^m$, and matrix $E\in\mR^{m\times n}$, let $\xi_{gls}$ be as in (\ref{eq:randlincons1}) or (\ref{eq:randlincons1a0}). Set
\begin{eqnarray}\label{eq:ta11}
f_{rp}(Z,\v,E) & \triangleq & \min_{\beta} \max_{\nu} \lp \|\beta\|_2^2 +\nu^T \lp Z\Sigma V^TL\beta-Z\Sigma V^T\bar{\beta}
+E\dbover{\Sigma}\dbover{V}^T\beta  -\sigma \overline{U}\overline{\Sigma}\v \rp\rp
 \hspace{.2in} (\bl{\textbf{random primal}})
\nonumber \\
\xi_{rp} & \triangleq & \lim_{n\rightarrow\infty } \mE_{Z,\v,E} f_{rp}(Z,\v,E).   \end{eqnarray}
Then
\begin{equation}\label{eq:ta11a0}
\xi_{gls}=f_{rp}(Z,\v,E) \quad \mbox{and} \quad \lim_{n\rightarrow\infty} \mE_{Z,\v,E}\xi_{gls} =\xi_{rp}.
\end{equation}
\label{lemma:lemma1}
\end{lemma}
\begin{proof}
Immediate consequence of the Lagrangian from (\ref{eq:randlincons2}).
\end{proof}

As mentioned earlier, the above lemma holds for any $V$, $\overline{U}$, $\dbover{V}$, $\Sigma$, $\overline{\Sigma}$, $\dbover{\Sigma}$, $\sigma$, and $\beta$ and any (deterministic or random) $Z$, $\v$, and $E$. As usual, the RDT proceeds by imposing a statistics  on $Z$, $\v$, and $E$.

%%%%%%%%%%%%%%%%%%%%%%%%%%%%%%%%%%%%%%%%%%%%%%%%%%%%%%%%%%%%%%%%%
%%%%%%%%%%%%%%%%%%%%%%%%%%%%%%%%%%%%%%%%%%%%%%%%%%%%%%%%%%%%%%%%%
%\subsubsection{Determining the random dual}
%\label{sec:thranl2}
%%%%%%%%%%%%%%%%%%%%%%%%%%%%%%%%%%%%%%%%%%%%%%%%%%%%%%%%%%%%%%%%%
%%%%%%%%%%%%%%%%%%%%%%%%%%%%%%%%%%%%%%%%%%%%%%%%%%%%%%%%%%%%%%%%%

\vspace{.1in}
\noindent \underline{2) \textbf{\emph{Determining the random dual:}}} Following further the common practice within the RDT, the concentration of measure will be utilized as well. Here that means that for any fixed $\epsilon >0$,  we can write (see, e.g. \cite{StojnicCSetam09,StojnicRegRndDlt10,StojnicICASSP10var})
\begin{equation}
\lim_{n\rightarrow\infty}\mP_{Z,\v,E}\left (\frac{|f_{rp}(Z,\v,E)-\mE_{Z,\v,E}(f_{rp}(Z,\v,E))|}{\mE_{Z,\v,E}(f_{rp}(Z,\v,E))}>\epsilon\right )\longrightarrow 0.\label{eq:ta15}
\end{equation}
The so-called random dual theorem stated below is another key RDT ingredient. It characterizes in a specific way the objectives of (\ref{eq:model011}) and (\ref{eq:randlincons1a0}) ( the optimization programs (\ref{eq:model011}) and (\ref{eq:randlincons1a0}) are used to obtain the GLS estimator, $\beta_{gls}$).

\begin{theorem}(Objective characterization via random dual) Assume the setup of Lemma \ref{lemma:lemma1} and let $Z$, $\v$, and $E$ be independent of each other and let their components be iid standard normals. Moreover, let $\g\in\mR^m$, $\h^{(1)}\in\mR^n$, and $\h^{(2)}\in\mR^n$ be vectors that are also comprised of iid standard normals and independent of each other. Set
\vspace{-.0in}
\begin{eqnarray}
c_2 & \triangleq & \sqrt{\left \|\Sigma V^T\lp L\beta -\bar{\beta}\rp \right \|_2^2 + \left \|\dbover{\Sigma}\dbover{V}^T\beta \right\|_2^2  } \nonumber \\
  f_{rd}(\g,\h^{(1)},\h^{(2)},\v) & \triangleq &
 \min_{\beta}\max_{\nu}
 \Bigg.\Bigg( \|\beta\|_2^2+c_2\nu^T\g
 +\|\nu\|_2\lp \lp\h^{(1)}\rp^T\Sigma V^T\lp L\beta -\bar{\beta}\rp  +\lp\h^{(2)}\rp^T \dbover{\Sigma}\dbover{V}^T \beta    \rp \nonumber \\
& &  -\sigma \nu^T\overline{U}\overline{\Sigma} \v \Bigg. \Bigg )    \hspace{2.0in} (\bl{\textbf{random dual}})
  \nonumber \\
 \xi_{rd} & \triangleq & \lim_{n\rightarrow\infty} \mE_{\g,\h^{(1)},\h^{(2)},\v} f_{rd}(\g,\h^{(1)},\h^{(2)},\v)  .\label{eq:ta16}
\vspace{-.0in}\end{eqnarray}
One then has \vspace{-.0in}
\begin{eqnarray}
  \xi_{rd} & \triangleq & \lim_{n\rightarrow\infty} \mE_{\g,\h^{(1)},\h^{(2)},\v} f_{rd}(\g,\h^{(1)},\h^{(2)},\v)
  \leq
  \lim_{n\rightarrow\infty} \mE_{Z,\v,E} f_{rp}(Z,\v,E)  \triangleq  \xi_{rp}. \label{eq:ta16a0}
\vspace{-.0in}\end{eqnarray}
and
\begin{eqnarray}
 \lim_{n\rightarrow\infty}\mP_{\g,\h^{(1)},\h^{(2)},\v} \lp f_{rd}(\g,\h^{(1)},\h^{(2)},\v)\geq (1-\epsilon)\xi_{rd}\rp
 \leq  \lim_{n\rightarrow\infty}\mP_{Z,\v,E} \lp f_{rp}(Z,\v,E)\geq (1-\epsilon)\xi_{rd}\rp.\label{eq:ta17}
\end{eqnarray}
\label{thm:thm1}
\end{theorem}\vspace{-.17in}
\begin{proof}
  Follows automatically as a direct application of the Gordon's probabilistic comparison theorem (see, e.g., Theorem B in \cite{Gordon88} and also Theorem 1, Corollary 1, and Section 2.7.2 in \cite{Stojnicgscomp16}).
\end{proof}

 \vspace{.1in}
\noindent \underline{3) \textbf{\emph{Handling the random dual:}}} Solving the inner maximization over $\nu$ is the starting point. To keep the notation lighter, we keep the statistical dependence implicit and write
\begin{eqnarray}
  f_{rd}
  & \triangleq &
 \min_{\beta}\max_{\nu}
\lp \|\beta\|_2^2+c_2\nu^T\g
 +\|\nu\|_2\lp \lp\h^{(1)}\rp^T\Sigma V^T\lp L\beta -\bar{\beta}\rp  +\lp\h^{(2)}\rp^T \dbover{\Sigma}\dbover{V}^T \beta    \rp    -\sigma \nu^T\overline{U}\overline{\Sigma} \v \rp
\nonumber \\
  & = &
\min_{\beta}\max_{\|\nu\|_2} \lp \|\beta\|_2^2
 + \|\nu\|_2\lp \lp\h^{(1)}\rp^T\Sigma V^T\lp L\beta -\bar{\beta}\rp  +\lp\h^{(2)}\rp^T \dbover{\Sigma}\dbover{V}^T \beta    \rp
 + \|\nu\|_2\|c_2\g
 - \sigma \overline{U}\overline{\Sigma} \v \|_2\rp \nonumber \\
  & = &
\min_{\beta}\max_{\|\nu\|_2}
\bigg.\Bigg( \left \| \dbover{V}^T\beta \right \|_2^2
+ \|\nu\|_2\lp \lp\h^{(1)}\rp^T\Sigma \lp V^TL \dbover{V} \dbover{V}^T\beta -V^T\bar{\beta}\rp  +\lp\h^{(2)}\rp^T \dbover{\Sigma}\dbover{V}^T \beta    \rp \nonumber \\
& &
+\|\nu\|_2\|c_2\g-\sigma \overline{U}\overline{\Sigma} \v \|_2\Bigg. \Bigg) \nonumber \\
  & = &
\min_{\x,c_2\geq 0}\max_{\nu_s\geq 0}
\lp \|\x\|_2^2
+\nu_s \lp \lp\h^{(1)}\rp^T\Sigma \lp \bar{L}\x -\c\rp  + \lp\h^{(2)}\rp^T\dbover{\Sigma}\x\rp
+\nu_s\|c_2\g-\sigma \overline{U}\overline{\Sigma} \v \|_2\rp, \label{eq:ta18a0}
\end{eqnarray}
where the third equality is implied by $\dbover{V}$ being unitary and the fourth by the following change of variables
\begin{eqnarray}
\bar{L}=V^TL\dbover{V},\quad  \x=\dbover{V}^T\beta, \quad \c\triangleq V^T\bar{\beta}, \quad \nu_s=\|\nu\|_2. \label{eq:ta18a0b0}
\end{eqnarray}
It is useful to note that such a change of variables also implies
\begin{eqnarray}
  c_2& \triangleq & \sqrt{\left \|\Sigma V^T\lp L\beta -\bar{\beta}\rp \right \|_2^2 + \left \|\dbover{\Sigma}\dbover{V}^T\beta \right\|_2^2  }
  \nonumber \\
  & = & \sqrt{\left \|\Sigma \lp V^TL \dbover{V}\dbover{V}^T\beta -V^T\bar{\beta}\rp \right \|_2^2 + \left \|\dbover{\Sigma}\dbover{V}^T\beta \right\|_2^2  }
  \nonumber \\
 & = & \sqrt{\|\Sigma\lp \bar{L}\x-\c\rp\|_2^2 +  \left \|\dbover{\Sigma}\x\right \|_2^2}. \label{eq:ta18a0b0c0}
\end{eqnarray}
We then set
\begin{eqnarray}
  \cL\lp\x,c_2,\nu_s \rp \triangleq
  \lp \|\x\|_2^2
+\nu_s \lp \lp\h^{(1)}\rp^T\Sigma \lp \bar{L}\x -\c\rp  + \lp\h^{(2)}\rp^T\dbover{\Sigma}\x\rp
+\nu_s\|c_2\g-\sigma \overline{U}\overline{\Sigma} \v \|_2\rp, \label{eq:ta18a0b1}
\end{eqnarray}
and observe  that (\ref{eq:ta18a0}) can be rewritten as
\begin{eqnarray}
  f_{rd}
  =
\min_{\x,c_2\geq 0}\max_{\nu_s\geq 0}  \cL\lp\x,c_2,\nu_s \rp. \label{eq:ta18a0b2}
\end{eqnarray}
We follow the standard RDT methodology (see, also \cite{Stojnicridge24}), fix $c_2$ and $\nu_s$, and consider the following minimization  over $\x$
\begin{eqnarray}
  \cL_1\lp c_2,\nu_s\rp \triangleq  \min_{\x} & &   \lp \|\x\|_2^2
+\nu_s \lp \lp\h^{(1)}\rp^T\Sigma \lp \bar{L}\x -\c\rp  + \lp\h^{(2)}\rp^T\dbover{\Sigma}\x\rp
+\nu_s\|c_2\g-\sigma \overline{U}\overline{\Sigma} \v \|_2\rp \nonumber \\
  \mbox{subject to} & &  \sqrt{\|\Sigma\lp \bar{L}\x-\c\rp\|_2^2 +  \left \|\dbover{\Sigma}\x\right \|_2^2}=c_2. \label{eq:ta18a0b2c0}
\end{eqnarray}
Writing the Lagrangian and relying on the strong Lagrange duality, we obtain
\begin{eqnarray}
  \cL_1\lp c_2,\nu_s\rp
  & = &  \min_{\x}\max_{\gamma}
  \Bigg.\Bigg (
   \|\x\|_2^2
  + \nu_s \lp \lp\h^{(1)}\rp^T\Sigma \lp \bar{L}\x -\c\rp  + \lp\h^{(2)}\rp^T\dbover{\Sigma}\x\rp
  + \nu_s\|c_2\g-\sigma \overline{U}\overline{\Sigma} \v \|_2 \nonumber \\
  & &
  +\gamma \|\Sigma\lp \bar{L}\x-\c\rp\|_2^2 +  \gamma\left \|\dbover{\Sigma}\x\right \|_2^2
  -\gamma c_2^2     \Bigg.\Bigg)
  \nonumber \\
  & = &  \max_{\gamma}   \min_{\x}
    \Bigg.\Bigg (
   \|\x\|_2^2
  + \nu_s \lp \lp\h^{(1)}\rp^T\Sigma \lp \bar{L}\x -\c\rp  + \lp\h^{(2)}\rp^T\dbover{\Sigma}\x\rp
  + \nu_s\|c_2\g-\sigma \overline{U}\overline{\Sigma} \v \|_2 \nonumber \\
  & &
  +\gamma \|\Sigma\lp \bar{L}\x-\c\rp\|_2^2 +  \gamma\left \|\dbover{\Sigma}\x\right \|_2^2
  -\gamma c_2^2     \Bigg.\Bigg)
  \nonumber \\
  & = &  \max_{\gamma}   \min_{\x}   \cL_2\lp \x,\gamma;c_2,\nu_s\rp, \label{eq:ta18a0b2c1}
\end{eqnarray}
with
\begin{eqnarray}
  \cL_2\lp \x,\gamma;c_2,\nu_s\rp
&  \triangleq &
  \Bigg.\Bigg (
   \|\x\|_2^2
  + \nu_s \lp \lp\h^{(1)}\rp^T\Sigma \lp \bar{L}\x -\c\rp  + \lp\h^{(2)}\rp^T\dbover{\Sigma}\x\rp
  + \nu_s\|c_2\g-\sigma \overline{U}\overline{\Sigma} \v \|_2 \nonumber \\
  & &
  +\gamma \|\Sigma\lp \bar{L}\x-\c\rp\|_2^2 +  \gamma\left \|\dbover{\Sigma}\x\right \|_2^2
  -\gamma c_2^2     \Bigg.\Bigg). \label{eq:ta18a0b2c2}
\end{eqnarray}
Computing the derivative gives
 \begin{eqnarray}
  \frac{\cL_2\lp\x,\gamma;c_2,\nu_s\rp}{d\x} =
  2\x
  +\nu_s\lp \bar{L}^T\Sigma \h^{(1)} +\dbover{\Sigma}\h^{(2)}   \rp
  +2\gamma \bar{L}^T\Sigma\Sigma \bar{L}\x
  -2\gamma\bar{L}^T\Sigma\Sigma\c
  +2\gamma\dbover{\Sigma}\dbover{\Sigma}\x.
   \label{eq:ta18a0b3}
\end{eqnarray}
After equalling the above derivative to zero we find
\begin{eqnarray}
 \hat{\x}=
 \lp I +\gamma \lp \bar{L}^T\Sigma\Sigma\bar{L} + \dbover{\Sigma} \dbover{\Sigma} \rp  \rp^{-1}
 \lp -\frac{1}{2}\nu_s \lp \bar{L}^T\Sigma \h^{(1)} +\dbover{\Sigma}\h^{(2)}   \rp +\gamma\bar{L}^T\Sigma\Sigma\c\rp. \label{eq:ta18a0b4}
\end{eqnarray}
Plugging $\hat{\x}$ back in (\ref{eq:ta18a0b2c2})   further gives
\begin{eqnarray}
 \min_{\x}   \cL_2\lp \x,\gamma;c_2,\nu_s\rp & = &
  \cL_2\lp \hat{\x},\gamma;c_2,\nu_s\rp \nonumber \\
 & = &
-\lp -\frac{1}{2}\nu_s \lp \bar{L}^T\Sigma \h^{(1)} +\dbover{\Sigma}\h^{(2)}   \rp +\gamma\bar{L}^T\Sigma\Sigma\c\rp^T
 \lp I +\gamma \lp \bar{L}^T\Sigma\Sigma\bar{L} + \dbover{\Sigma} \dbover{\Sigma} \rp  \rp^{-1} \nonumber \\
 & &
\times
\lp -\frac{1}{2}\nu_s \lp \bar{L}^T\Sigma \h^{(1)} +\dbover{\Sigma}\h^{(2)}   \rp +\gamma\bar{L}^T\Sigma\Sigma\c\rp
\nonumber \\
   & & -\nu_s \lp \h^{(1)} \rp^T\Sigma\c+ \nu_s\|c_2\g-\sigma \overline{U}\overline{\Sigma} \v \|_2 +\gamma \c^T\Sigma\Sigma\c -\gamma c_2^2. \label{eq:ta18a0b4c0}
\end{eqnarray}
Introducing scaling $\nu_s\rightarrow \frac{\nu_1}{\sqrt{n}}$ (with $\nu_1$ that does not change with $n$), (\ref{eq:ta18a0b4c0}) becomes
\begin{eqnarray}
  \cL_2\lp \hat{\x},\gamma;c_2,\nu_1\rp
 & = &
-\lp -\frac{1}{2\sqrt{n}}\nu_1 \lp \bar{L}^T\Sigma \h^{(1)} +\dbover{\Sigma}\h^{(2)}   \rp +\gamma\bar{L}^T\Sigma\Sigma\c\rp^T
 \lp I +\gamma \lp \bar{L}^T\Sigma\Sigma\bar{L} + \dbover{\Sigma} \dbover{\Sigma} \rp  \rp^{-1} \nonumber \\
 & &
\times
\lp -\frac{1}{2\sqrt{n}}\nu_1 \lp \bar{L}^T\Sigma \h^{(1)} +\dbover{\Sigma}\h^{(2)}   \rp +\gamma\bar{L}^T\Sigma\Sigma\c\rp
\nonumber \\
   & & -\frac{1}{\sqrt{n}}\nu_1 \lp \h^{(1)} \rp^T\Sigma\c
   + \frac{1}{\sqrt{n}}\nu_1\|c_2\g-\sigma \overline{U}\overline{\Sigma} \v \|_2
   +\gamma \c^T\Sigma\Sigma\c -\gamma c_2^2. \label{eq:ta18a0b4c1}
\end{eqnarray}
To make writing easier and more concrete, we consider the following eigen-decomposition
\begin{eqnarray}
\bar{L}^T\Sigma\Sigma\bar{L} + \dbover{\Sigma} \dbover{\Sigma}    = U^{(d)} D \lp U^{(d)}\rp^T,
% \s^{(1)} \triangleq \mbox{diag}\lp\Sigma\rp, \quad  \s^{(2)} \triangleq \mbox{diag}\lp\dbover{\Sigma}\rp.
\label{eq:ta18a0b4c1d0}
\end{eqnarray}
where $U^{(d)}$ is unitary and $D$ is diagonal. We also set
\begin{eqnarray}
\bar{\c}\triangleq  \lp U^{(d)}\rp^T\bar{L}^T\Sigma\Sigma\c,\quad  \s\triangleq \sqrt{\mbox{diag}\lp D\rp}, \quad \s^{(1)} \triangleq \mbox{diag}\lp\Sigma\rp.
\label{eq:ta18a0b4c2}
\end{eqnarray}
One can then rewrite (\ref{eq:ta18a0b4c1}) in the following way
\begin{eqnarray}
  \cL_2\lp \hat{\x},\gamma;c_2,\nu_1\rp
 & = &
-\lp -\frac{1}{2\sqrt{n}}\nu_1 \lp \bar{L}^T\Sigma \h^{(1)} +\dbover{\Sigma}\h^{(2)}   \rp +\gamma\bar{L}^T\Sigma\Sigma\c\rp^T
 \lp U^{(d)}  \lp U^{(d)}\rp^T +\gamma U^{(d)} D \lp U^{(d)}\rp^T  \rp^{-1} \nonumber \\
 & &
\times
\lp -\frac{1}{2\sqrt{n}}\nu_1 \lp \bar{L}^T\Sigma \h^{(1)} +\dbover{\Sigma}\h^{(2)}   \rp +\gamma\bar{L}^T\Sigma\Sigma\c\rp
\nonumber \\
   & & -\frac{1}{\sqrt{n}}\nu_1 \lp \h^{(1)} \rp^T\Sigma\c
   + \frac{1}{\sqrt{n}}\nu_1\|c_2\g-\sigma \overline{U}\overline{\Sigma} \v \|_2
   +\gamma \c^T\Sigma\Sigma\c -\gamma c_2^2 \nonumber \\
& = &
-\lp -\frac{1}{2\sqrt{n}}\nu_1 \lp \bar{L}^T\Sigma \h^{(1)} +\dbover{\Sigma}\h^{(2)}   \rp +\gamma\bar{L}^T\Sigma\Sigma\c\rp^T
 U^{(d)} \lp I  +\gamma D  \rp^{-1} \nonumber \\
 & &
\times
\lp U^{(d)}\rp^T\lp -\frac{1}{2\sqrt{n}}\nu_1  \lp \bar{L}^T\Sigma \h^{(1)} +\dbover{\Sigma}\h^{(2)}   \rp +\gamma  \bar{L}^T\Sigma\Sigma\c\rp
\nonumber \\
   & & -\frac{1}{\sqrt{n}}\nu_1 \lp \h^{(1)} \rp^T\Sigma\c
   + \frac{1}{\sqrt{n}}\nu_1\|c_2\g-\sigma \overline{U}\overline{\Sigma} \v \|_2
   +\gamma \c^T\Sigma\Sigma\c -\gamma c_2^2 \nonumber \\
   & = &
-\t^T \lp I  +\gamma D  \rp^{-1} \t
  -\frac{1}{\sqrt{n}}\nu_1 \lp \h^{(1)} \rp^T\Sigma\c
   + \frac{1}{\sqrt{n}}\nu_1\|c_2\g-\sigma \overline{U}\overline{\Sigma} \v \|_2
   +\gamma \c^T\Sigma\Sigma\c -\gamma c_2^2, \nonumber \\
   & = &
-\tr\lp \lp I  +\gamma D  \rp^{-1} \t \t^T\rp
   -\frac{1}{\sqrt{n}}\nu_1 \lp \h^{(1)} \rp^T\Sigma\c
   + \frac{1}{\sqrt{n}}\nu_1\|c_2\g-\sigma \overline{U}\overline{\Sigma} \v \|_2
   +\gamma \c^T\Sigma\Sigma\c -\gamma c_2^2, \nonumber \\
   \label{eq:ta18a0b4c1d1}
\end{eqnarray}
where
\begin{eqnarray}
 \t
  & \triangleq &
\lp U^{(d)}\rp^T \lp -\frac{1}{2\sqrt{n}}\nu_1   \lp \bar{L}^T\Sigma \h^{(1)} +\dbover{\Sigma}\h^{(2)}   \rp +\gamma  \bar{L}^T\Sigma\Sigma\c\rp
\nonumber \\
  & = &
  -\frac{1}{2\sqrt{n}}\nu_1  \lp U^{(d)}\rp^T \lp \bar{L}^T\Sigma \h^{(1)} +\dbover{\Sigma}\h^{(2)}   \rp +\gamma  \lp U^{(d)}\rp^T\bar{L}^T\Sigma\Sigma\c
\nonumber \\
  & = &
  -\frac{1}{2\sqrt{n}}\nu_1  \lp U^{(d)}\rp^T \lp \bar{L}^T\Sigma \h^{(1)} +\dbover{\Sigma}\h^{(2)}   \rp +\gamma  \bar{\c}, \label{eq:ta18a0b4c1d2}
\end{eqnarray}
and $\bar{\c}$ is as introduced in (\ref{eq:ta18a0b4c2}). Relying on concentrations while keeping $c_2$, $\nu_1$, and $\gamma$ fixed, we further find
\begin{eqnarray}
\lim_{n\rightarrow\infty}  \mE_{\g,\h^{(1)},\h^{(2)},\v}\cL_2\lp \hat{\x},\gamma;c_2,\nu_1\rp
 & = &
\lim_{n\rightarrow\infty}  \mE
\Bigg ( \Big.
-\tr\lp \lp I  +\gamma D  \rp^{-1} \t \t^T\rp
   -\frac{1}{\sqrt{n}}\nu_1 \lp \h^{(1)} \rp^T\Sigma\c \nonumber \\
   & &
   + \frac{1}{\sqrt{n}}\nu_1\|c_2\g-\sigma \overline{U}\overline{\Sigma} \v \|_2
   +\gamma \c^T\Sigma\Sigma\c -\gamma c_2^2    \Big.\Bigg ) \nonumber \\
 & = &
\lim_{n\rightarrow\infty}  \mE
\lp
-\tr\lp \lp I  +\gamma D  \rp^{-1} \t \t^T\rp \rp \nonumber \\
& &
+\nu_1\sqrt{\alpha}\sqrt{c_2^2+\bar{\sigma}^2}   + \gamma \sum_{i=1}^n \lp \s_i^{(1)}\rp^2\c_i^2 -\gamma c_2^2,
%\nonumber \\
%  & = & \lim_{n\rightarrow\infty} \Bigg ( \Big. -\frac{\nu_1^2}{4n}\sum_{i=1}^n \frac{\s_i^2}{1+\gamma\s_i^2}
%  - \sum_{i=1}^n \frac{\gamma^2\s_i^4\c_i^2}{1+\gamma\s_i^2}
%+\nu_1\sqrt{\alpha}\sqrt{c_2^2+\bar{\sigma}^2}   + \gamma \sum_{i=1}^n \s_i^2\c_i^2 -\gamma c_2^2  \Big.\Bigg ),
\nonumber \\ \label{eq:ta18a0b4c3d0}
\end{eqnarray}
where
\begin{eqnarray}
\bar{\sigma}\triangleq \sigma \sqrt{\lim_{n\rightarrow\infty} \frac{\tr\lp\overline{\Sigma}\overline{\Sigma} \rp}{m}}. \label{eq:ta18a0b4c4}
\end{eqnarray}
We then first observe
\begin{eqnarray}
\lim_{n\rightarrow\infty}  \mE
\lp
-\tr\lp \lp I  +\gamma D  \rp^{-1} \t \t^T\rp \rp
& = &
\lim_{n\rightarrow\infty}
\lp
-\tr\lp \lp I  +\gamma D  \rp^{-1} \mE \t \t^T\rp \rp, \label{eq:ta18a0b4c3d1}
\end{eqnarray}
and after utilizing (\ref{eq:ta18a0b4c1d0}) and (\ref{eq:ta18a0b4c1d2}) also find
\begin{eqnarray}
  \mE
\lp  \t \t^T\rp
& = &
 \frac{1}{4n}\nu_1^2 \lp U^{(d)}\rp^T \lp \bar{L}^T\Sigma\Sigma \bar{L} +\dbover{\Sigma}\dbover{\Sigma}   \rp U^{(d)}
  +\gamma^2 \bar{\c} \bar{\c}^T \nonumber \\
& = &
 \frac{1}{4n}\nu_1^2 \lp U^{(d)}\rp^T  U^{(d)}D\lp U^{(d)}\rp^T  U^{(d)}
  +\gamma^2 \bar{\c} \bar{\c}^T  \nonumber \\
& = &
 \frac{1}{4n}\nu_1^2  D
  +\gamma^2 \bar{\c} \bar{\c}^T. \label{eq:ta18a0b4c3d2}
\end{eqnarray}
A combination of (\ref{eq:ta18a0b4c3d0}), (\ref{eq:ta18a0b4c3d1}), and (\ref{eq:ta18a0b4c3d2}) gives
\begin{eqnarray}
\lim_{n\rightarrow\infty}  \mE_{\g,\h^{(1)},\h^{(2)},\v}\cL_2\lp \hat{\x},\gamma;c_2,\nu_1\rp
  & = &
\lim_{n\rightarrow\infty} \Bigg ( \Big. \mE
\lp
-\tr\lp \lp I  +\gamma D  \rp^{-1} \t \t^T\rp \rp
+\nu_1\sqrt{\alpha}\sqrt{c_2^2+\bar{\sigma}^2}   \nonumber \\
& &
+\gamma \sum_{i=1}^n \lp \s_i^{(1)}\rp^2\c_i^2  -\gamma c_2^2 \Big.\Bigg ) \nonumber \\
  & = &
\lim_{n\rightarrow\infty}
\Bigg ( \Big.
\lp
-\tr\lp \lp I  +\gamma D  \rp^{-1} \mE\t \t^T\rp \rp
+\nu_1\sqrt{\alpha}\sqrt{c_2^2+\bar{\sigma}^2}
\nonumber \\
& &
  + \gamma \sum_{i=1}^n \lp \s_i^{(1)}\rp^2\c_i^2  -\gamma c_2^2 \Big.\Bigg )\nonumber \\
  & = &
\lim_{n\rightarrow\infty}
\Bigg ( \Big.
\lp
-\tr\lp \lp I  +\gamma D  \rp^{-1} \lp  \frac{1}{4n}\nu_1^2  D
  +\gamma^2 \bar{\c} \bar{\c}^T \rp\rp \rp \nonumber \\
  & &
+\nu_1\sqrt{\alpha}\sqrt{c_2^2+\bar{\sigma}^2}
 + \gamma \sum_{i=1}^n \lp \s_i^{(1)}\rp^2\c_i^2  -\gamma c_2^2
  \Big.\Bigg )
 \nonumber \\
   & = & \lim_{n\rightarrow\infty} \Bigg ( \Big. -\frac{\nu_1^2}{4n}\sum_{i=1}^n \frac{\s_i^2}{1+\gamma\s_i^2}
 - \sum_{i=1}^n \frac{\gamma^2\bar{\c}_i^2}{1+\gamma\s_i^2}
 \nonumber \\
 & &
+\nu_1\sqrt{\alpha}\sqrt{c_2^2+\bar{\sigma}^2}   + \gamma \sum_{i=1}^n \lp \s_i^{(1)}\rp^2\c_i^2  -\gamma c_2^2  \Big.\Bigg ). \label{eq:ta18a0b4c3}
\end{eqnarray}
As stated earlier, to make writing easier we avoid repeatedly using $\lim_{n\rightarrow\infty}$ (all expressions involving dimensions $n$, $m$, or $k$ are assumed to be written within the $\lim_{n\rightarrow\infty}$ context). Moreover, all such expressions are assumed to be bounded with well defined and existing limiting values. A simple combination of (\ref{eq:ta18a0b1})-(\ref{eq:ta18a0b2c1}), (\ref{eq:ta18a0b4c0}), and  (\ref{eq:ta18a0b4c3}) allows one to arrive at the following
\begin{eqnarray}
\lim_{n\rightarrow\infty}  \mE_{\g,\h^{(1)},\h^{(2)},\v}f_{rd}(\g,\h^{(1)},\h^{(2)},\v)
   & = & \lim_{n\rightarrow\infty} \min_{c_2\geq 0} \max_{\nu_1\geq 0,\gamma}  f_0(c_2,\nu_1,\gamma), \label{eq:ta18a0b4c5}
\end{eqnarray}
where
\begin{eqnarray}
f_0(c_2,\nu_1,\gamma) \triangleq  -\frac{\nu_1^2}{4n}\sum_{i=1}^n \frac{\s_i^2}{1+\gamma\s_i^2}
 - \sum_{i=1}^n \frac{\gamma^2\bar{\c}_i^2}{1+\gamma\s_i^2}
 +\nu_1\sqrt{\alpha}\sqrt{c_2^2+\bar{\sigma}^2}   + \gamma \sum_{i=1}^n \lp \s_i^{(1)}\rp^2\c_i^2  -\gamma c_2^2 .  \label{eq:ta18a0b4c6}
\end{eqnarray}
The optimization over $\nu_1$ is simple (in fact identical to the corresponding one in \cite{Stojnicridge24}), and we get easily
for the optimal $\nu_1$
\begin{eqnarray}
\hat{\nu}_1 =
 \frac{2\sqrt{\alpha}\sqrt{c_2^2+\bar{\sigma}^2}}{\frac{1}{n}\sum_{i=1}^n \frac{\s_i^2}{1+\gamma\s_i^2}}, \label{eq:ta18a0b4c8}
\end{eqnarray}
which together with (\ref{eq:ta18a0b4c6}) gives
\begin{eqnarray}
\max_{\nu_1\geq 0} f_0(c_2,\nu_1,\gamma) = f_0(c_2,\hat{\nu}_1,\gamma) =
  - \sum_{i=1}^n \frac{\gamma^2\bar{\c}_i^2}{1+\gamma\s_i^2}
+ \frac{\alpha\lp c_2^2+\bar{\sigma}^2\rp}{\frac{1}{n}\sum_{i=1}^n \frac{\s_i^2}{1+\gamma\s_i^2}} + \gamma \sum_{i=1}^n \lp \s_i^{(1)}\rp^2\c_i^2 -\gamma c_2^2.  \label{eq:ta18a0b4c9}
\end{eqnarray}
Combining further (\ref{eq:ta18a0b4c5}) and (\ref{eq:ta18a0b4c9}), we arrive at
\begin{eqnarray}
\lim_{n\rightarrow\infty}  \mE_{\g,\h^{(1)},\h^{(2)},\v}f_{rd}(\g,\h^{(1)},\h^{(2)},\v)
   & = & \lim_{n\rightarrow\infty}  \min_{c_2\geq 0} \max_{\gamma}  f_0(c_2,\hat{\nu}_1,\gamma), \label{eq:ta18a0b4c10}
\end{eqnarray}
where $f_0(c_2,\hat{\nu}_1,\gamma)$ as in (\ref{eq:ta18a0b4c9}). Computing the derivatives with respect to $c_2$ and $\gamma$, one finds
\begin{eqnarray}
\frac{d f_0(c_2,\hat{\nu}_1,\gamma)}{d c_2} =
  \frac{2\alpha c_2}{\frac{1}{n}\sum_{i=1}^n \frac{\s_i^2}{1+\gamma\s_i^2}}   -2\gamma c_2,  \label{eq:ta18a0b4c11}
\end{eqnarray}
and
\begin{eqnarray}
\frac{d f_0(c_2,\hat{\nu}_1,\gamma)}{d \gamma}
& = &
  - \sum_{i=1}^n \frac{2\gamma\bar{\c}_i^2}{1+\gamma\s_i^2}
  + \sum_{i=1}^n \frac{\gamma^2\s_i^2\bar{\c}_i^2}{\lp 1+\gamma\s_i^2\rp^2}
 +\frac{\alpha\lp c_2^2+\bar{\sigma}^2\rp}{\lp \frac{1}{n}\sum_{i=1}^n \frac{\s_i^2}{1+\gamma\s_i^2}\rp^2}
 \lp \frac{1}{n}\sum_{i=1}^n \frac{\s_i^4}{\lp 1+\gamma\s_i^2\rp^2}\rp \nonumber \\
 & &
 + \sum_{i=1}^n \lp \s_i^{(1)}\rp^2\c_i^2 - c_2^2. \label{eq:ta18a0b4c12}
\end{eqnarray}
Equalling to zero the derivative in  (\ref{eq:ta18a0b4c11}), we find that the optimal $\hat{\gamma}$ satisfies
\begin{eqnarray}
  \frac{1}{n}\sum_{i=1}^n \frac{\hat{\gamma}\s_i^2}{1+\hat{\gamma}\s_i^2} =\alpha.  \label{eq:ta18a0b4c13}
\end{eqnarray}
After setting
\begin{eqnarray}
a_1 & = &
  - \sum_{i=1}^n \frac{2\hat{\gamma}\bar{\c}_i^2}{1+\hat{\gamma}\s_i^2}
  + \sum_{i=1}^n \frac{\hat{\gamma}^2\s_i^2\bar{\c}_i^2}{\lp 1+\hat{\gamma}\s_i^2\rp^2}
  + \sum_{i=1}^n \lp \s_i^{(1)}\rp^2\c_i^2 ,
\nonumber \\
a_2 & = &
\frac{  \frac{1}{n}\sum_{i=1}^n \frac{\s_i^4}{\lp 1+\hat{\gamma}\s_i^2\rp^2}}{\lp  \frac{1}{n}\sum_{i=1}^n \frac{\s_i^2}{1+\hat{\gamma}\s_i^2}\rp^2},  \label{eq:ta18a0b4c15}
\end{eqnarray}
we from (\ref{eq:ta18a0b4c12}) find
\begin{eqnarray}
\hat{c}_2^2= \frac{a_1+\alpha \bar{\sigma}^2a_2 }{1-\alpha a_2}.  \label{eq:ta18a0b4c16}
\end{eqnarray}
From (\ref{eq:ta18a0b4c9}) and (\ref{eq:ta18a0b4c10}), we also have
\begin{eqnarray}
\lim_{n\rightarrow\infty}  \mE_{\g,\h^{(1)},\h^{(2)},\v}f_{rd}(\g,\h^{(1)},\h^{(2)},\v)
   & = & \lim_{n\rightarrow\infty} \min_{c_2\geq 0} \max_{\gamma}  f_0(c_2,\hat{\nu}_1,\gamma) \nonumber \\
   & = & \lim_{n\rightarrow\infty} f_0(\hat{c}_2,\hat{\nu}_1,\hat{\gamma}) \nonumber \\
   & = &  \lim_{n\rightarrow\infty} \lp  - \sum_{i=1}^n \frac{\hat{\gamma}^2\bar{\c}_i^2}{1+\hat{\gamma}\s_i^2}
+ \frac{\alpha\lp \hat{c}_2^2+\bar{\sigma}^2\rp}{\frac{1}{n}\sum_{i=1}^n \frac{\s_i^2}{1+\hat{\gamma}\s_i^2}} + \hat{\gamma} \sum_{i=1}^n \lp \s_i^{(1)}\rp^2\c_i^2 - \hat{\gamma} \hat{c}_2^2 \rp \nonumber \\
    & = &   \lim_{n\rightarrow\infty}  \lp  - \sum_{i=1}^n \frac{\hat{\gamma}^2\bar{\c}_i^2}{1+\hat{\gamma}\s_i^2}
+ \hat{\gamma}\lp \hat{c}_2^2+\bar{\sigma}^2\rp
+ \hat{\gamma} \sum_{i=1}^n \lp \s_i^{(1)}\rp^2\c_i^2
 - \hat{\gamma} \hat{c}_2^2 \rp \nonumber \\
   & = &   \lim_{n\rightarrow\infty}
   \lp  - \sum_{i=1}^n \frac{\hat{\gamma}^2\bar{\c}_i^2}{1+\hat{\gamma}\s_i^2}
+ \hat{\gamma} \sum_{i=1}^n \lp \s_i^{(1)}\rp^2\c_i^2
+ \hat{\gamma}\bar{\sigma}^2 \rp, \label{eq:ta18a10}
\end{eqnarray}
where $\hat{c}_2$ and $\hat{\gamma}$ are given through (\ref{eq:ta18a0b4c13}), (\ref{eq:ta18a0b4c15}), and (\ref{eq:ta18a0b4c16}). Taking the very same $\hat{c}_2$ and $\hat{\gamma}$ in (\ref{eq:ta18a0b4c8}) gives for the optimal $\nu_1$
\begin{eqnarray}
\hat{\nu}_1 =
 \frac{2\sqrt{\alpha}\sqrt{\hat{c}_2^2+\bar{\sigma}^2}}{\frac{1}{n}\sum_{i=1}^n \frac{\s_i^2}{1+\hat{\gamma}\s_i^2}}=
 \frac{2\hat{\gamma}\sqrt{\hat{c}_2^2+\bar{\sigma}^2}}{\sqrt{\alpha}}. \label{eq:ta18a0b4c17}
\end{eqnarray}

The above discussion is summarized in the following lemma.
\begin{lemma}(Characterization of random dual) Assume the setup of Theorem \ref{thm:thm1} with $\xi_{rd}$ as in (\ref{eq:ta16}) and consider a large $n$ linear regime with $\alpha=\lim_{n\rightarrow\infty}\frac{m}{n}$ that does not change as $n$ grows. Assume that all the considered limiting quantities are well defined and bounded. Let $\c=V^T\bar{\beta}$, $\bar{L}=V^TL\dbover{V}$,  and $\bar{\sigma}\triangleq \sigma \sqrt{\lim_{n\rightarrow\infty} \frac{\tr\lp\overline{\Sigma}\overline{\Sigma} \rp}{m}}$. Consider the eigen-decomposition
\begin{eqnarray}
\bar{L}^T\Sigma\Sigma\bar{L} + \dbover{\Sigma} \dbover{\Sigma}    = U^{(d)} D \lp U^{(d)}\rp^T,
% \s^{(1)} \triangleq \mbox{diag}\lp\Sigma\rp, \quad  \s^{(2)} \triangleq \mbox{diag}\lp\dbover{\Sigma}\rp.
\label{eq:thmaaeq1a0}
\end{eqnarray}
where $U^{(d)}$ is unitary and $D$ is diagonal and set
\begin{eqnarray}
\bar{\c}\triangleq  \lp U^{(d)}\rp^T\bar{L}^T\Sigma\Sigma\c,\quad  \s\triangleq \sqrt{\mbox{diag}\lp D\rp}, \quad \s^{(1)} \triangleq \mbox{diag}\lp\Sigma\rp.
\label{eq:thmaaeq1a1}
\end{eqnarray}
Moreover, let $\hat{\gamma}$ be the unique solution of
\begin{eqnarray}
\lim_{n\rightarrow\infty} \frac{1}{n}\sum_{i=1}^n \frac{\hat{\gamma}\s_i^2}{1+\hat{\gamma}\s_i^2} =\alpha.  \label{eq:thmaaeq1}
\end{eqnarray}
Set
\begin{eqnarray}
a_1 & = &
\lim_{n\rightarrow\infty}\lp   - \sum_{i=1}^n \frac{2\hat{\gamma}\bar{\c}_i^2}{1+\hat{\gamma}\s_i^2}
  + \sum_{i=1}^n \frac{\hat{\gamma}^2\s_i^2\bar{\c}_i^2}{\lp 1+\hat{\gamma}\s_i^2\rp^2}
  + \sum_{i=1}^n \lp \s_i^{(1)}\rp^2\c_i^2\rp,
\nonumber \\
a_2 & =&
\frac{\lim_{n\rightarrow\infty} \frac{1}{n}\sum_{i=1}^n \frac{\s_i^4}{\lp 1+\hat{\gamma}\s_i^2\rp^2}}{\lp \lim_{n\rightarrow\infty} \frac{1}{n}\sum_{i=1}^n \frac{\s_i^2}{1+\hat{\gamma}\s_i^2}\rp^2}
=\frac{\hat{\gamma}^2}{\alpha^2} \lim_{n\rightarrow\infty} \frac{1}{n}\sum_{i=1}^n \frac{\s_i^4}{\lp 1+\hat{\gamma}\s_i^2\rp^2}.  \label{eq:thmaaeq2}
\end{eqnarray}
One then has \vspace{-.0in}
\begin{eqnarray}
  \xi_{rd} & = & \lim_{n\rightarrow\infty} \mE_{\g,\h^{(1)},\h^{(2)},\v} f_{rd}(\g,\h^{(1)},\h^{(2)},\v)
 =
 \lim_{n\rightarrow\infty}
   \lp  - \sum_{i=1}^n \frac{\hat{\gamma}^2\bar{\c}_i^2}{1+\hat{\gamma}\s_i^2}
+ \hat{\gamma} \sum_{i=1}^n \lp \s_i^{(1)}\rp^2\c_i^2
+ \hat{\gamma}\bar{\sigma}^2 \rp \nonumber \\
&  &  \lim_{n\rightarrow\infty}  \hat{c}_2^2
 =
 \frac{a_1 +\alpha \bar{\sigma}^2a_2 }{1-\alpha a_2} \nonumber \\
 & & \lim_{n\rightarrow\infty}  \hat{\nu}_1
 = \frac{2\hat{\gamma}\sqrt{\hat{c}_2^2+\bar{\sigma}^2}}{\sqrt{\alpha}},
\label{eq:thmaaeq3}
\end{eqnarray}
and for any fixed $\epsilon>0$
\begin{eqnarray}
 \lim_{n\rightarrow\infty}\mP_{\g,\h^{(1)},\h^{(2)},\v} \lp  (1-\epsilon)\xi_{rd} \leq  f_{rd}(\g,\h^{(1)},\h^{(2)},\v)\leq (1+\epsilon)\xi_{rd}\rp
& \longrightarrow & 1 \nonumber \\
 \lim_{n\rightarrow\infty}\mP_{\g,\h^{(1)},\h^{(2)},\v} \lp  (1-\epsilon)\hat{c}_2^2 \leq  c_2^2\leq (1+\epsilon)\hat{c}_2^2\rp
& \longrightarrow & 1 \nonumber \\
 \lim_{n\rightarrow\infty}\mP_{\g,\h^{(1)},\h^{(2)},\v} \lp  (1-\epsilon)\hat{\nu}_1 \leq  \nu_1 \leq (1+\epsilon)\hat{\nu}_1\rp
& \longrightarrow & 1.\label{eq:thmaaeq4}
\end{eqnarray}
\label{lemma:lemma2}
\end{lemma}\vspace{-.17in}
\begin{proof}
Follows from the above presentation and trivial concentrations of $f_{rd}(\g,\h^{(1)},\h^{(2)},\v)$, $c_2$, and $\nu_1$.
\end{proof}

The preceding discussion and Lemma \ref{lemma:lemma2} allow us to also formulate the following theorem.

\begin{theorem}(Characterization of FRM GLS estimator) Assume the setup of Lemma \ref{lemma:lemma2} and let $\hat{\beta}=\beta_{gls}$ where $\beta_{gls}$ is as in
(\ref{eq:model011}). Also, let $\hat{\gamma}$, $a_1$, $a_2$, and $\hat{c}_2$ be as in (\ref{eq:thmaaeq1}), (\ref{eq:thmaaeq2}), and (\ref{eq:thmaaeq3}), respectively. Moreover,  let $\xi_{gls}$ defined in (\ref{eq:randlincons1}) or (\ref{eq:randlincons1a0}) be the objective value of the GLS estimating optimization  and let the GLS interpolator excess prediction risk, $R(\bar{\beta},\hat{\beta})=R(\bar{\beta},\beta_{gls})$, be as in (\ref{eq:model020}). Then
\begin{eqnarray}\label{eq:thm2a11a0}
 \lim_{n\rightarrow\infty} \mE_{Z,\v,E}\xi_{gls}
 &= &\xi_{rp}=\xi_{rd}
 =
 \lim_{n\rightarrow\infty}
   \lp  - \sum_{i=1}^n \frac{\hat{\gamma}^2\bar{\c}_i^2}{1+\hat{\gamma}\s_i^2}
+ \hat{\gamma} \sum_{i=1}^n \lp \s_i^{(1)}\rp^2\c_i^2
+ \hat{\gamma}\bar{\sigma}^2 \rp \nonumber \\
 \lim_{n\rightarrow\infty} \mE_{Z,\v,E} R(\bar{\beta},\beta_{gls})
& = &
\lim_{n\rightarrow\infty} \mE_{Z,\v,E} \lp \left \| \Sigma V^T \lp \bar{\beta}  - L\beta_{gls} \rp\right \|_2^2 +\left \|\dbover{\Sigma}\dbover{V}^T\hat{\beta}\right \|_2^2 \rp
 =
 \frac{a_1+\alpha \bar{\sigma}^2a_2 }{1-\alpha a_2}. \nonumber \\
\end{eqnarray}
Moreover,
\begin{eqnarray}
 \lim_{n\rightarrow\infty}\mP_{Z,\v,E} \lp  (1-\epsilon)\mE_{Z,\v,E}\xi_{gls} \leq  \xi_{gls} \leq (1+\epsilon)\mE_{Z,\v}\xi_{gls}\rp
& \longrightarrow & 1 \nonumber \\
 \lim_{n\rightarrow\infty}\mP_{Z,\v,E} \lp  (1-\epsilon)\mE_{Z,\v,E} R(\bar{\beta},\beta_{gls})  \leq  R(\bar{\beta},\beta_{gls}) \leq (1+\epsilon)\mE_{Z,\v,E}R(\bar{\beta},\beta_{gls})\rp
& \longrightarrow & 1.\label{eq:thm2ta17}
\end{eqnarray}
 \label{thm:thm2}
\end{theorem}
\begin{proof}
The proof follows from Lemma  \ref{lemma:lemma2} and the preceding discussion (in exactly the same way the proof of \cite{Stojnicridge24}'s Theorem 2 follows from its Lemma 2 and the corresponding discussion).
\end{proof}

\vspace{.1in}
\noindent \underline{ 4) \textbf{\emph{Double checking strong random duality:}}} Due to the underlying convexity,  the strong random duality and the reversal arguments of \cite{StojnicGorEx10,StojnicRegRndDlt10} are applicable and the above results establish the exact characterizations of the underlying optimization problems and the associated algorithms/estimators.

%%%%%%%%%%%%%%%%%%%%%%%%%%%%%%%%%%%%%%%%%%%%%%%%%%%%%%%%%%%%%%%%%
%%%%%%%%%%%%%%%%%%%%%%%%%%%%%%%%%%%%%%%%%%%%%%%%%%%%%%%%%%%%%%%%%
%%%%%%%%%%%%%%%%%%%%%%%%%%%%%%%%%%%%%%%%%%%%%%%%%%%%%%%%%%%%%%%%%
\subsubsection{Compact risk form - GLS}
\label{sec:compgls}
%%%%%%%%%%%%%%%%%%%%%%%%%%%%%%%%%%%%%%%%%%%%%%%%%%%%%%%%%%%%%%%%%
%%%%%%%%%%%%%%%%%%%%%%%%%%%%%%%%%%%%%%%%%%%%%%%%%%%%%%%%%%%%%%%%%
%%%%%%%%%%%%%%%%%%%%%%%%%%%%%%%%%%%%%%%%%%%%%%%%%%%%%%%%%%%%%%%%%

The risk characterization given in Theorem \ref{thm:thm2} can be rephrased in a more compact matrix form which relates directly to the matrices of the original FRM rather then to their decompositions. We start the  matrix recomposition by first observing that (\ref{eq:thmaaeq1}) can be rewritten as
\begin{eqnarray}
 1- \lim_{n\rightarrow\infty} \frac{1}{n}\sum_{i=1}^n \frac{\hat{\gamma}\s_i^2}{1+\hat{\gamma}\s_i^2}
 & = & 1-\alpha
 \nonumber \\
\Longleftrightarrow \qquad  \lim_{n\rightarrow\infty} \frac{1}{n}\sum_{i=1}^n \frac{1}{1+\hat{\gamma}\s_i^2}
 & = & 1-\alpha,
 \label{eq:compthmproofeq0a0}
\end{eqnarray}
or in a matrix form
\begin{eqnarray}
   \lim_{n\rightarrow\infty} \frac{1}{n}\tr\lp I+\hat{\gamma}D \rp^{-1}
 & = & 1-\alpha \nonumber \\
 \Longleftrightarrow \qquad  \lim_{n\rightarrow\infty} \frac{1}{n}\tr\lp I+\hat{\gamma} \bar{L}^T\Sigma\Sigma\bar{L} + \hat{\gamma} \dbover{\Sigma} \dbover{\Sigma}  \rp^{-1}
 & = & 1-\alpha
 \nonumber \\
 \Longleftrightarrow \qquad  \lim_{n\rightarrow\infty} \frac{1}{n}\tr\lp I+\hat{\gamma} \dbover{V} \bar{L}^T\Sigma\Sigma\bar{L} \dbover{V} ^T+ \hat{\gamma} \dbover{V} \dbover{\Sigma} \dbover{\Sigma} \dbover{V} ^T \rp^{-1}
 & = & 1-\alpha \nonumber \\
\Longleftrightarrow \qquad  \lim_{n\rightarrow\infty} \frac{1}{n}\tr\lp I+\hat{\gamma} L^T V\Sigma\Sigma V ^TL+ \hat{\gamma} \dbover{V} \dbover{\Sigma} \dbover{\Sigma} \dbover{V} ^T \rp^{-1}
 & = & 1-\alpha \nonumber \\
\Longleftrightarrow \qquad
\lim_{n\rightarrow\infty} \frac{1}{n}\tr\lp I+\hat{\gamma} L^T A^TAL + \hat{\gamma}\dbover{A}^T \dbover{A} \rp^{-1}
 & = & 1-\alpha,
 \label{eq:compthmproofeq0a1}
\end{eqnarray}
where the second line follows from (\ref{eq:ta18a0b4c1d0}) or ((\ref{eq:thmaaeq1a0})), the fourth one by recalling on $\bar{L}=V^TL\dbover{V}$, and the fifth by the SVDs from (\ref{eq:model08}). We can also rewrite $a_{1}$ from (\ref{eq:thmaaeq2}) in the following way
\begin{eqnarray}
a_{1}
& = &
-\lim_{n\rightarrow\infty}\lp     \sum_{i=1}^n \frac{2\hat{\gamma}\bar{\c}_i^2}{1+\hat{\gamma}\s_i^2} \rp
+\lim_{n\rightarrow\infty}\lp    \sum_{i=1}^n \frac{\hat{\gamma}^2\s_i^2\bar{\c}_i^2}{\lp 1+\hat{\gamma}\s_i^2\rp^2} \rp
+\lim_{n\rightarrow\infty}\lp    \sum_{i=1}^n \lp \s_i^{(1)}\rp^2\c_i^2\rp
\nonumber \\
& = &
-\lim_{n\rightarrow\infty}\lp     \sum_{i=1}^n \frac{\hat{\gamma}\bar{\c}_i^2}{1+\hat{\gamma}\s_i^2} \rp
- \lim_{n\rightarrow\infty}\lp    \sum_{i=1}^n \frac{\hat{\gamma}\bar{\c}_i^2}{\lp 1+\hat{\gamma}\s_i^2\rp^2} \rp
+ \lim_{n\rightarrow\infty}\lp    \sum_{i=1}^n \lp \s_i^{(1)}\rp^2\c_i^2\rp
\nonumber \\
& = &
 \lim_{n\rightarrow\infty}  \lp -\hat{\gamma}  \bar{\c}^T  \lp I +\hat{\gamma} D\rp ^{-1}\bar{\c}
- \hat{\gamma} \bar{\c}^T  \lp I +\hat{\gamma} D\rp ^{-2}\bar{\c}
 + \c^T\Sigma\Sigma\c
 \rp.
 \label{eq:compthmproofeq0}
\end{eqnarray}
Utilizing $\bar{\c}= \lp U^{(d)}\rp^T\bar{L}^T\Sigma\Sigma\c$ from (\ref{eq:ta18a0b4c2}), we then also have
\begin{eqnarray}
   \lim_{n\rightarrow\infty}  \lp \bar{\c}^T \lp I+\hat{\gamma} D\rp^{-1}\bar{\c}
 \rp
 & = &
 \lim_{n\rightarrow\infty}
  \lp \lp\lp U^{(d)}\rp^T\bar{L}^T\Sigma\Sigma\c\rp ^TD^{-1}\lp U^{(d)}\rp^T\bar{L}^T\Sigma\Sigma\c\rp
  \nonumber \\
 & = &
 \lim_{n\rightarrow\infty}  \lp
 \lp \c^T\Sigma\Sigma \bar{L} U^{(d)}\rp
 \lp I+\hat{\gamma} D\rp^{-1}
 \lp U^{(d)}\rp^T\bar{L}^T\Sigma\Sigma\c\rp
 \nonumber \\
 & = &
 \lim_{n\rightarrow\infty}  \lp
  \c^T\Sigma\Sigma \bar{L}
 \lp I + \hat{\gamma} \bar{L}^T\Sigma\Sigma\bar{L} + \hat{\gamma}\dbover{\Sigma} \dbover{\Sigma}  \rp^{-1}
  \bar{L}^T\Sigma\Sigma\c
  \rp.
 \label{eq:compthmproofeq1}
\end{eqnarray}
Recalling on $\bar{L}=V^TL\dbover{V}$  and $\c\triangleq V^T\bar{\beta}$ from (\ref{eq:ta18a0b0}) and the SVD decompositions from (\ref{eq:model08}), we first have
\begin{eqnarray}
   \lim_{n\rightarrow\infty}  \lp \bar{\c}^T \lp I+\hat{\gamma} D\rp^{-1}\bar{\c}
 \rp
  & = &
 \lim_{n\rightarrow\infty}  \lp
  \c^T\Sigma\Sigma \bar{L}
 \lp I + \hat{\gamma}\bar{L}^T\Sigma\Sigma\bar{L} + \hat{\gamma}\dbover{\Sigma} \dbover{\Sigma}  \rp^{-1}
  \bar{L}^T\Sigma\Sigma\c
  \rp
  \nonumber \\
   & = &
 \lim_{n\rightarrow\infty}  \lp
  \bar{\beta}^T V\Sigma\Sigma V^TL\dbover{V}
 \lp I + \hat{\gamma}\bar{L}^T\Sigma\Sigma\bar{L} + \hat{\gamma}\dbover{\Sigma} \dbover{\Sigma}  \rp^{-1}
  \lp V^TL\dbover{V} \rp^T\Sigma\Sigma V^T\bar{\beta}
  \rp
  \nonumber \\
   & = &
 \lim_{n\rightarrow\infty}   \lp
  \bar{\beta}^T V\Sigma\Sigma V^TL\dbover{V}
 \lp I +\hat{\gamma}\bar{L}^T\Sigma\Sigma\bar{L} + \hat{\gamma}\dbover{\Sigma} \dbover{\Sigma}  \rp^{-1}
   \dbover{V}^T L^T V \Sigma\Sigma V^T\bar{\beta}
  \rp
  \nonumber \\
   & = &
 \lim_{n\rightarrow\infty}   \lp
  \bar{\beta}^T A^TA L\dbover{V}
 \lp I + \hat{\gamma}\bar{L}^T\Sigma\Sigma\bar{L} + \hat{\gamma}\dbover{\Sigma} \dbover{\Sigma}  \rp^{-1}
   \dbover{V}^T L^T A^TA\bar{\beta}
  \rp
  \nonumber \\
   & = &
 \lim_{n\rightarrow\infty}   \lp
  \bar{\beta}^T A^TA L
\lp I + \hat{\gamma} \dbover{V}  \bar{L}^T\Sigma\Sigma\bar{L} \dbover{V}^T
 + \hat{\gamma}\dbover{V}\dbover{\Sigma} \dbover{\Sigma} \dbover{V}^T  \rp^{-1}
    L^T A^TA\bar{\beta}
  \rp
  \nonumber \\
   & = &
 \lim_{n\rightarrow\infty}   \lp
  \bar{\beta}^T A^TA L
\lp I + \hat{\gamma} L^T V\Sigma\Sigma V^T L + \hat{\gamma}\dbover{A}^T\dbover{A}    \rp^{-1}
    L^T A^TA\bar{\beta}
  \rp
  \nonumber \\
   & = &
 \lim_{n\rightarrow\infty}   \lp
  \bar{\beta}^T A^TA L
\lp I + \hat{\gamma} L^T A^TA L + \hat{\gamma}\dbover{A}^T\dbover{A}    \rp^{-1}
    L^T A^TA\bar{\beta}
  \rp.
 \label{eq:compthmproofeq2}
\end{eqnarray}
One then completely analogously obtains
\begin{eqnarray}
   \lim_{n\rightarrow\infty}  \lp \bar{\c}^T \lp I+\hat{\gamma} D\rp^{-2}\bar{\c}
 \rp
   & = &
 \lim_{n\rightarrow\infty}   \lp
  \bar{\beta}^T A^TA L
\lp I + \hat{\gamma} L^T A^TA L + \hat{\gamma}\dbover{A}^T\dbover{A}    \rp^{-2}
    L^T A^TA\bar{\beta}
  \rp.
 \label{eq:compthmproofeq3}
\end{eqnarray}
One also easily has
\begin{eqnarray}
   \lim_{n\rightarrow\infty}
   \lp
   \c^T\Sigma\Sigma \c
  \rp
  =
 \lim_{n\rightarrow\infty}
   \lp
   \bar{\beta}^T A^T A \bar{\beta}
  \rp.
 \label{eq:compthmproofeq4}
\end{eqnarray}
We rewrite $a_{2}$ from (\ref{eq:thmaaeq2}) as
\begin{eqnarray}
 a_2 & =&
 \frac{\hat{\gamma}^2}{\alpha^2} \lim_{n\rightarrow\infty} \frac{1}{n}\sum_{i=1}^n \frac{\s_i^4}{\lp 1+\hat{\gamma}\s_i^2\rp^2}
 =
 \frac{\hat{\gamma}^2}{\alpha^2} \lim_{n\rightarrow\infty} \frac{1}{n}\tr \lp  D^2\lp I+\hat{\gamma} D\rp^{-2}  \rp
 =
 \frac{\hat{\gamma}^2}{\alpha^2} \lim_{n\rightarrow\infty} \frac{1}{n}\tr \lp  D^{-1}+\hat{\gamma I}   \rp^{-2}.
 \label{eq:compthmproofeq5}
\end{eqnarray}
Following (\ref{eq:compthmproofeq0a1}), we further restructure (\ref{eq:compthmproofeq5})
\begin{eqnarray}
a_{2} & = &  \frac{\hat{\gamma}^2}{\alpha^2} \lim_{n\rightarrow\infty} \frac{1}{n}\tr \lp  D^{-1}+\hat{\gamma I}   \rp^{-2}. \nonumber \\
& = &
 \frac{\hat{\gamma}^2}{\alpha^2} \lim_{n\rightarrow\infty}  \frac{1}{n}\tr\lp \hat{\gamma} I+\lp \bar{L}^T\Sigma\Sigma\bar{L} +  \dbover{\Sigma} \dbover{\Sigma} \rp^{-1} \rp^{-2} \nonumber \\
  & = &
   \frac{\hat{\gamma}^2}{\alpha^2} \lim_{n\rightarrow\infty} \frac{1}{n}\tr\lp \hat{\gamma} I+\lp\dbover{V} \bar{L}^T\Sigma\Sigma\bar{L} \dbover{V} ^T+ \dbover{V} \dbover{\Sigma} \dbover{\Sigma} \dbover{V} ^T \rp^{-1}\rp^{-2}
  \nonumber \\
 & = &
   \frac{\hat{\gamma}^2}{\alpha^2} \lim_{n\rightarrow\infty} \frac{1}{n}\tr\lp \hat{\gamma} I+\lp L^T V\Sigma\Sigma V ^TL+ \dbover{V} \dbover{\Sigma} \dbover{\Sigma} \dbover{V} ^T \rp^{-1}\rp^{-2}
  \nonumber \\
 & = &
 \frac{\hat{\gamma}^2}{\alpha^2} \lim_{n\rightarrow\infty} \frac{1}{n}\tr\lp \hat{\gamma}  I+\lp L^T A^TAL +  \dbover{A}^T \dbover{A}\rp^{-1} \rp^{-2}.
 \label{eq:compthmproofeq6}
\end{eqnarray}

We then have the following corollary of Theorem \ref{thm:thm2}.

\begin{corollary}(FRM GLS estimator - compact risk form) Consider the FRM from (\ref{eq:model01}) with factors $\bar{Z}$, loadings $L$, noises $\e$ and $\bar{E}$ and assume the setup of Theorem \ref{thm:thm2}. Let the factors and noises' covariances be $\Sigma_{\bar{Z}}$, $\Sigma_{\bar{E}}$, and $\Sigma_{\e}$, i.e., let
\begin{eqnarray}\label{eq:corcompeq1}
 \Sigma_{\bar{Z}}=A^TA,\quad  \Sigma_{\bar{E}}=\dbover{A}^T\dbover{A},\quad \Sigma_{\e}=\sigma^2\overline{A}^T\overline{A}.
 \end{eqnarray}
Let $\hat{\gamma}$ be the unique solution of
\begin{eqnarray}
\lim_{n\rightarrow\infty} \frac{1}{n}\tr\lp I+\hat{\gamma} L^T \Sigma_{\bar{Z}} L + \hat{\gamma} \Sigma_{\bar{E}} \rp^{-1}
 & = & 1-\alpha,
\label{eq:corcompeq2}
\end{eqnarray}
Set
\begin{eqnarray}
a_{1}
& = &
 \lim_{n\rightarrow\infty}   \lp
\bar{\beta}^T\Sigma_{\bar{Z}}\bar{\beta}
- \hat{\gamma} \bar{\beta}^T \Sigma_{\bar{Z}} L \lp
\lp I + \hat{\gamma} L^T \Sigma_{\bar{Z}} L + \hat{\gamma}\Sigma_{\bar{E}}   \rp^{-1}
+
\lp I + \hat{\gamma} L^T \Sigma_{\bar{Z}} L + \hat{\gamma}\Sigma_{\bar{E}}   \rp^{-2}
\rp    L^T \Sigma_{\bar{Z}}\bar{\beta}
  \rp
\nonumber \\
a_{2}
  & = &
 \frac{\hat{\gamma}^2}{\alpha^2} \lim_{n\rightarrow\infty} \frac{1}{n}\tr\lp \hat{\gamma}  I+\lp L^T  \Sigma_{\bar{Z}} L +  \Sigma_{\bar{E}}\rp^{-1} \rp^{-2} \nonumber \\
 \bar{\sigma}^2
 & = &
 \lim_{n\rightarrow\infty} \frac{1}{m}\tr\lp \Sigma_{\e}\rp.
\label{eq:corcompeq4}
\end{eqnarray}
 Then
\begin{eqnarray}\label{eq:corcompeq5}
  \lim_{n\rightarrow\infty} \mE_{Z,\v,E} R(\bar{\beta},\beta_{gls})
& = &
\lim_{n\rightarrow\infty} \mE_{Z,\v,E} \lp
\lp \bar{\beta}  - L\beta_{gls} \rp^T
\Sigma_{\bar{Z}}
\lp \bar{\beta}  - L\beta_{gls} \rp +   \beta_{gls}^T \Sigma_{\bar{E}} \beta_{gls} \rp \nonumber \\
& = &
 \frac{a_1+\alpha \bar{\sigma}^2a_2 }{1-\alpha a_2}, \nonumber \\
\end{eqnarray}
and
\begin{eqnarray}
  \lim_{n\rightarrow\infty}\mP_{Z,\v,E} \lp  (1-\epsilon)\mE_{Z,\v,E} R(\bar{\beta},\beta_{gls})  \leq  R(\bar{\beta},\beta_{gls}) \leq (1+\epsilon)\mE_{Z,\v,E}R(\bar{\beta},\beta_{gls})\rp
& \longrightarrow & 1.\label{eq:corcompeq6}
\end{eqnarray}
 \label{cor:cor1}
\end{corollary}
\begin{proof}
Follows from Theorem  \ref{thm:thm2} and (\ref{eq:compthmproofeq0a1}), (\ref{eq:compthmproofeq0}), (\ref{eq:compthmproofeq2})- (\ref{eq:compthmproofeq4}), and (\ref{eq:compthmproofeq6}).
\end{proof}

%%%%%%%%%%%%%%%%%%%%%%%%%%%%%%%%%%%%%%%%%%%%%%%%%%%%%%%%%%%%%%%%%
%%%%%%%%%%%%%%%%%%%%%%%%%%%%%%%%%%%%%%%%%%%%%%%%%%%%%%%%%%%%%%%%%
%%%%%%%%%%%%%%%%%%%%%%%%%%%%%%%%%%%%%%%%%%%%%%%%%%%%%%%%%%%%%%%%%
%%%%%%%%%%%%%%%%%%%%%%%%%%%%%%%%%%%%%%%%%%%%%%%%%%%%%%%%%%%%%%%%%
%%%%%%%%%%%%%%%%%%%%%%%%%%%%%%%%%%%%%%%%%%%%%%%%%%%%%%%%%%%%%%%%%
%%%%%%%%%%%%%%%%%%%%%%%%%%%%%%%%%%%%%%%%%%%%%%%%%%%%%%%%%%%%%%%%%
\subsection{Excess risk of the ridge estimator}
\label{sec:analridge}
%%%%%%%%%%%%%%%%%%%%%%%%%%%%%%%%%%%%%%%%%%%%%%%%%%%%%%%%%%%%%%%%%
%%%%%%%%%%%%%%%%%%%%%%%%%%%%%%%%%%%%%%%%%%%%%%%%%%%%%%%%%%%%%%%%%
%%%%%%%%%%%%%%%%%%%%%%%%%%%%%%%%%%%%%%%%%%%%%%%%%%%%%%%%%%%%%%%%%
%%%%%%%%%%%%%%%%%%%%%%%%%%%%%%%%%%%%%%%%%%%%%%%%%%%%%%%%%%%%%%%%%
%%%%%%%%%%%%%%%%%%%%%%%%%%%%%%%%%%%%%%%%%%%%%%%%%%%%%%%%%%%%%%%%%
%%%%%%%%%%%%%%%%%%%%%%%%%%%%%%%%%%%%%%%%%%%%%%%%%%%%%%%%%%%%%%%%%

There is a very strong connection between the above presented FRM GLS interpolator analysis and the corresponding ridge estimator one (for a corresponding LRM connection see \cite{Stojnicridge24}). Such a connection is utilized below to obtain a precise characterization of the ridge estimator's excess prediction  risk. To smoothen the presentation, we try to parallel the presentation of the previous sections but fairly often proceed in a faster fashion making sure that the unnecessary repetition of essentially same  arguments is avoided.

Relying on (\ref{eq:model014}), the objective value of the underlying ridge regression id defined as
 \begin{eqnarray}
\xi_{rr}(\lambda) \triangleq \min_{\beta} \lp \lambda\|\beta\|_2^2 +\frac{1}{m}\|\y- X\beta\|_2^2\rp. \label{eq:ridgerandlincons1}
\end{eqnarray}
for a given fixed (dimensions independent) ridge parameter $\lambda\geq 0$. A quick rewriting of (\ref{eq:ridgerandlincons1}) esily gives
 \begin{eqnarray}
\xi_{rr}(\lambda) = \min_{\beta,\z,c_{3,s}} & &   \lambda\|\beta\|_2^2 + \frac{c_{3,s}^2}{m} \nonumber \\
\mbox{subject to} & & X\beta=\y+\z \nonumber \\
& & \|\z\|_2=c_{3,s}. \label{eq:ridgerandlincons1a0}
\end{eqnarray}
Keeping in mind (\ref{eq:model010}) and (\ref{eq:model010a0}), one also has the following statistical equivalent to (\ref{eq:ridgerandlincons1a0})
\begin{eqnarray}
\xi_{rr}(\lambda) \triangleq \min_{\beta,\|\z\|_2=c_{3,s}} & &  \lambda\|\beta\|_2^2 + \frac{c_{3,s}^2}{m} \nonumber \\
\mbox{subject to} & & Z\Sigma V^TL\beta + E\dbover{\Sigma}\dbover{V}^T\beta=Z\Sigma V^T\bar{\beta}+\sigma \overline{U}\overline{\Sigma}\v+\z. \label{eq:ridgerandlincons1a1}
\end{eqnarray}
Writing the Lagrangian further gives
\begin{eqnarray}
\xi_{rr}(\lambda) = \min_{\beta,\|\z\|_2=c_{3,s}} \max_{\nu} \lp \lambda\|\beta\|_2^2 +\nu^T \lp Z\Sigma V^TL\beta-Z\Sigma V^T\bar{\beta} +E \dbover{\Sigma}\dbover{V}^T\beta-\sigma \overline{U}\overline{\Sigma}\v -\z \rp  + \frac{c_{3,s}^2}{m}\rp. \label{eq:ridgerandlincons2}
\end{eqnarray}
Following the methodology of the previous sections, we again proceed by utilizing the RDT machinery to analytically characterize the above problem.

%%%%%%%%%%%%%%%%%%%%%%%%%%%%%%%%%%%%%%%%%%%%%%%%%%%%%%%%%%%%%%%%%
%%%%%%%%%%%%%%%%%%%%%%%%%%%%%%%%%%%%%%%%%%%%%%%%%%%%%%%%%%%%%%%%%
%%%%%%%%%%%%%%%%%%%%%%%%%%%%%%%%%%%%%%%%%%%%%%%%%%%%%%%%%%%%%%%%%
\subsubsection{Ridge estimator excess risk via RDT}
\label{sec:randlinconsrdtridge}
%%%%%%%%%%%%%%%%%%%%%%%%%%%%%%%%%%%%%%%%%%%%%%%%%%%%%%%%%%%%%%%%%
%%%%%%%%%%%%%%%%%%%%%%%%%%%%%%%%%%%%%%%%%%%%%%%%%%%%%%%%%%%%%%%%%
%%%%%%%%%%%%%%%%%%%%%%%%%%%%%%%%%%%%%%%%%%%%%%%%%%%%%%%%%%%%%%%%%

As was the case earlier when we discussed the GLS interpolator, we again consider separately four main RDT principles.

\vspace{.1in}

\noindent \underline{1) \textbf{\emph{Algebraic characterization:}}}  The following is a ridge analogue to Lemma \ref{lemma:lemma1} and  provides a useful representation of the objective $\xi_{rr}(\lambda)$ by conveniently summarizing the above algebraic considerations.

\begin{lemma}(Algebraic optimization representation) Assume the setup of Lemma \ref{lemma:lemma1} and let $\lambda\geq 0$ be a fixed real scalar independent of $n$ or any other quantity. Let $\xi_{rr}(\lambda)$ be as in (\ref{eq:ridgerandlincons1}) or (\ref{eq:ridgerandlincons1a0}) and set ${\mathcal D}=\{\beta,\z,c_{3,s}|\beta\in\mR^n,\z\in\mR^m,c_{3,s}\in\mR,\|\z\|_2=c_{3,s}\}$ and
\begin{eqnarray}\label{eq:ridgeta11}
f_{rp,r}(Z,\v,E) & \triangleq & \min_{{\mathcal D}} \max_{\nu} \lp \lambda\|\beta\|_2^2 +\nu^T \lp Z\Sigma V^T L\beta-Z\Sigma V^T\bar{\beta}
+E\dbover{\Sigma}\dbover{V}^T\beta -\sigma \overline{U}\overline{\Sigma}\v -\z\rp +\frac{c_{3,s}^2}{m} \rp \nonumber \\
& &
 \hspace{3.8in} (\bl{\textbf{random primal}})
\nonumber \\
\xi_{rp,r} & \triangleq & \lim_{n\rightarrow\infty } \mE_{Z,\v,E} f_{rp,r}(Z,\v,E).   \end{eqnarray}
Then
\begin{equation}\label{eq:ridgeta11a0}
\xi_{rr}(\lambda)=f_{rp,r}(Z,\v,E) \quad \mbox{and} \quad \lim_{n\rightarrow\infty} \mE_{Z,\v,E}\xi_{rr}(\lambda) =\xi_{rp,r}.
\end{equation}
\label{lemma:ridgelemma1}
\end{lemma}
\begin{proof}
 Follows immediately through the Lagrangian given in (\ref{eq:ridgerandlincons2}).
\end{proof}

While the above lemma holds even deterministically (i.e., for any $Z$, $\v$, and $E$), to utilize its full power, the RDT proceeds by imposing a statistics  on $Z$, $\v$, and $E$.

\vspace{.1in}
\noindent \underline{2) \textbf{\emph{Determining the random dual:}}} As earlier, the concentration of measure is utilized again. Here it implies that for any fixed $\epsilon >0$,  we can write (see, e.g. \cite{StojnicCSetam09,StojnicRegRndDlt10,StojnicICASSP10var})
\begin{equation}
\lim_{n\rightarrow\infty}\mP_{Z,\v,E}\left (\frac{|f_{rp,r}(Z,\v,E)-\mE_{Z,\v,E}(f_{rp,r}(Z,\v,E))|}{\mE_{Z,\v,E}(f_{rp,r}(Z,\v,E))}>\epsilon\right )\longrightarrow 0.\label{eq:ridgeta15}
\end{equation}
The following is basically the random dual ridge analogue to Theorem \ref{thm:thm1}).

\begin{theorem}(Objective characterization via random dual -- Ridge estimator) Assume the setup of Lemma \ref{lemma:ridgelemma1} and Theorem \ref{thm:thm1}. Set
\vspace{-.0in}
\begin{eqnarray}
   f_{rd,r}(\g,\h^{(1)},\h^{(2)},\v) & \triangleq &
\hspace{-.0in} \min_{{\mathcal D}}\max_{\nu} \Bigg.\Bigg( \lambda \|\beta\|_2^2+c_2\nu^T\g
+\|\nu\|_2 \lp \lp \h^{(1)}\rp^T\Sigma V^T\lp L\beta -\bar{\beta}\rp  + \lp \h^{(2)}\rp^T \dbover{\Sigma}\dbover{V}^T\beta\rp
\nonumber \\
& &
-\sigma \nu^T\overline{U}\overline{\Sigma} \v -\nu^T\z  +\frac{c_{3,s}^2}{m} \Bigg.\Bigg)  \hspace{2.0in}    (\bl{\textbf{random dual}})
  \nonumber \\
 \xi_{rd,r} & \triangleq &  \lim_{n\rightarrow\infty} \mE_{\g,\h^{(1)},\h^{(2)},\v} f_{rd,r}(\g,\h^{(1)},\h^{(2)},\v)  .\label{eq:ridgeta16}
\vspace{-.0in}
\end{eqnarray}
One then has \vspace{-.0in}
\begin{eqnarray}
  \xi_{rd,r} & \triangleq & \lim_{n\rightarrow\infty} \mE_{\g,\h^{(1)},\h^{(2)},\v} f_{rd,r}(\g,\h^{(1)},\h^{(2)},\v)
  \leq
  \lim_{n\rightarrow\infty} \mE_{Z,\v,E} f_{rp,r}(Z,\v,E)  \triangleq  \xi_{rp,r}. \label{eq:ridgeta16a0}
\vspace{-.0in}\end{eqnarray}
and
\begin{eqnarray}
 \lim_{n\rightarrow\infty}\mP_{\g,\h^{(1)},\h^{(2)},\v} \lp f_{rd,r}(\g,\h^{(1)},\h^{(2)},\v)\geq (1-\epsilon)\xi_{rd,r}\rp
 \leq  \lim_{n\rightarrow\infty}\mP_{Z,\v,E} \lp f_{rp,r}(Z,\v,E)\geq (1-\epsilon)\xi_{rd,r}\rp.\label{eq:ridgeta17}
\end{eqnarray}
\label{thm:ridgethm1}
\end{theorem}\vspace{-.17in}
\begin{proof}
As the proof of Theorem \ref{thm:thm1}, it follows as an immediate application of the Gordon's probabilistic comparison theorem (see, e.g., Theorem B in \cite{Gordon88} and also Theorem 1, Corollary 1, and Section 2.7.2 in \cite{Stojnicgscomp16}).
\end{proof}

 \vspace{.1in}
\noindent \underline{3) \textbf{\emph{Handling the random dual:}}} Following into the footsteps of the GLS analysis from the previous sections, we preserve the lighter notation by keeping the functional randomness dependence implicit and, after solving the inner maximization over $\nu$, obtain analogously to (\ref{eq:ta18a0})
\begin{eqnarray}
  f_{rd,r}
  & \triangleq &
\min_{\beta,\|\z\|_2=\c_{3,s}}\max_{\nu}
\Bigg.\Bigg( \lambda \|\beta\|_2^2+c_2\nu^T\g
+\|\nu\|_2 \lp \lp \h^{(1)}\rp^T\Sigma V^T\lp L\beta -\bar{\beta}\rp  + \lp \h^{(2)}\rp^T \dbover{\Sigma}\dbover{V}^T\beta\rp
\nonumber \\
& &
-\sigma \nu^T\overline{U}\overline{\Sigma} \v -\nu^T\z  +\frac{c_{3,s}^2}{m} \Bigg.\Bigg)  \nonumber \\
   & = &
\min_{\x,\|\z\|_2=c_{3,s},c_2\geq 0}\max_{\nu_s\geq 0}
\Bigg.\Bigg(  \lambda \|\x\|_2^2
+\nu_s \lp \lp\h^{(1)}\rp^T\Sigma \lp \bar{L}\x -\c\rp  + \lp\h^{(2)}\rp^T\dbover{\Sigma}\x\rp \nonumber \\
& &
+\nu_s\|c_2\g-\sigma \overline{U}\overline{\Sigma} \v -\z\|_2 +\frac{c_{3,s}^2}{m}\Bigg.\Bigg), \label{eq:ridgeta18a0}
\end{eqnarray}
where $\bar{L},\x,\c,\nu_s,$ and $c_2$ are as in (\ref{eq:ta18a0b0}) and (\ref{eq:ta18a0b0c0}). After setting
\begin{eqnarray}
  \cL_r\lp\x,\z,\c_{3,s},c_2,\nu_s \rp \triangleq   \lambda \|\x\|_2^2
+\nu_s \lp \lp\h^{(1)}\rp^T\Sigma \lp \bar{L}\x -\c\rp  + \lp\h^{(2)}\rp^T\dbover{\Sigma}\x\rp
+\nu_s\|c_2\g-\sigma \overline{U}\overline{\Sigma} \v -\z\|_2 +\frac{c_{3,s}^2}{m}, \nonumber \\
\label{eq:ridgeta18a0b1}
\end{eqnarray}
one can easily rewrite (\ref{eq:ridgeta18a0}) as
\begin{eqnarray}
  f_{rd,r}
  =
\min_{\x,\|\z\|_2=c_{3,s},c_2\geq 0}\max_{\nu_s\geq 0}  \cL_r\lp\x,\z,c_{3,s},c_2,\nu_s \rp. \label{eq:ridgeta18a0b2}
\end{eqnarray}
Following further the GLS analysis from Section \ref{sec:analgls}, we keep  $c_{3,s}$, $c_2$, and $\nu_s$ fixed and, analogously to (\ref{eq:ta18a0b2c0}), consider the following minimization  over $\x$ and $\z$
\begin{eqnarray}
  \cL_{1,r}\lp c_{3,s},c_2,\nu_s\rp \triangleq  \min_{\x,\z} & & \lambda \|\x\|_2^2
+\nu_s \lp \lp\h^{(1)}\rp^T\Sigma \lp \bar{L}\x -\c\rp  + \lp\h^{(2)}\rp^T\dbover{\Sigma}\x\rp
  + \nu_s\|c_2\g-\sigma \overline{U}\overline{\Sigma} \v -\z\|_2 +\frac{c_{3,s}^2}{m} \nonumber \\
  \mbox{subject to} & &  \sqrt{\left \|\Sigma\lp \bar{L} \x-\c\rp \right \|_2^2+ \left \| \dbover{\Sigma}\x \right \|_2^2}=c_2, \|\z\|_2=c_{3,s}. \label{eq:ridgeta18a0b2c0}
\end{eqnarray}
Analogously to (\ref{eq:ta18a0b2c1}), we further have
\begin{eqnarray}
  \cL_{1,r}\lp c_{3,s},c_2,\nu_s\rp
   & = &  \max_{\gamma}   \min_{\x,\|\z\|_2=c_{3,s}}   \cL_{2,r}\lp \x,\z,\gamma;c_{3,s},c_2,\nu_s\rp, \label{eq:ridgeta18a0b2c1}
\end{eqnarray}
with
\begin{eqnarray}
  \cL_{2,r}\lp \x,\z,\gamma;c_{3,s},c_2,\nu_s\rp
  & \triangleq &
\lambda \|\x\|_2^2
+\nu_s \lp \lp\h^{(1)}\rp^T\Sigma \lp \bar{L}\x -\c\rp  + \lp\h^{(2)}\rp^T\dbover{\Sigma}\x\rp
  + \nu_s\|c_2\g-\sigma \overline{U}\overline{\Sigma} \v -\z\|_2  \nonumber \\
  & &  +\frac{c_{3,s}^2}{m}  +\gamma \lp \left \|\Sigma\lp \bar{L} \x-\c\rp \right \|_2^2+ \left \| \dbover{\Sigma}\x \right \|_2^2 \rp -\gamma c_2^2 . \label{eq:ridgeta18a0b2c2}
\end{eqnarray}
Moreover, analogously to (\ref{eq:ta18a0b4c0}), we obtain for optimal $\x$
\begin{eqnarray}
 \hat{\x}=
 \lp \lambda I +\gamma \lp \bar{L}^T\Sigma\Sigma\bar{L} + \dbover{\Sigma} \dbover{\Sigma} \rp  \rp^{-1}
 \lp -\frac{1}{2}\nu_s \lp \bar{L}^T\Sigma \h^{(1)} +\dbover{\Sigma}\h^{(2)}   \rp +\gamma\bar{L}^T\Sigma\Sigma\c\rp. \label{eq:ridgta18a0b4}
\end{eqnarray}
which after  scaling $\nu_s\rightarrow \frac{\nu_1}{\sqrt{n}}$ (where $\nu_1$ again does not change as $n$ grows) gives the following analogue to (\ref{eq:ta18a0b4c1})
\begin{eqnarray}
  \cL_{2,r}\lp \hat{\x},\z,\gamma;c_{3,s},c_2,\nu_1\rp
 & = &
-\lp -\frac{1}{2\sqrt{n}}\nu_1 \lp \bar{L}^T\Sigma \h^{(1)} +\dbover{\Sigma}\h^{(2)}   \rp +\gamma\bar{L}^T\Sigma\Sigma\c\rp^T
 \lp \lambda I +\gamma \lp \bar{L}^T\Sigma\Sigma\bar{L} + \dbover{\Sigma} \dbover{\Sigma} \rp  \rp^{-1} \nonumber \\
 & &
\times
\lp -\frac{1}{2\sqrt{n}}\nu_1 \lp \bar{L}^T\Sigma \h^{(1)} +\dbover{\Sigma}\h^{(2)}   \rp +\gamma\bar{L}^T\Sigma\Sigma\c\rp
\nonumber \\
   & & -\frac{1}{\sqrt{n}}\nu_1 \lp \h^{(1)} \rp^T\Sigma\c
+ \frac{1}{\sqrt{n}}\nu_1\|c_2\g-\sigma \overline{U}\overline{\Sigma} \v -\z \|_2 +\frac{c_{3,s}^2}{m}+\gamma \c^T\Sigma\Sigma\c -\gamma c_2^2. \nonumber \\ \label{eq:ridgeta18a0b4c1}
\end{eqnarray}
After finding the optimal $\z$
\begin{eqnarray}
     \hat{\z}=c_{3,s}\frac{c_2\g-\sigma \overline{U}\overline{\Sigma} \v}{\|c_2\g-\sigma \overline{U}\overline{\Sigma} \v\|_2}, \label{eq:ridgeta18a0b4c1d0}
\end{eqnarray}
one, from (\ref{eq:ridgeta18a0b4c1}), has
\begin{eqnarray}
  \cL_{2,r}\lp \hat{\x},\hat{\z},\gamma;c_{3,s},c_2,\nu_1\rp
 & = &
-\lp -\frac{1}{2\sqrt{n}}\nu_1 \lp \bar{L}^T\Sigma \h^{(1)} +\dbover{\Sigma}\h^{(2)}   \rp +\gamma\bar{L}^T\Sigma\Sigma\c\rp^T
 \lp \lambda I +\gamma \lp \bar{L}^T\Sigma\Sigma\bar{L} + \dbover{\Sigma} \dbover{\Sigma} \rp  \rp^{-1} \nonumber \\
 & &
\times
\lp -\frac{1}{2\sqrt{n}}\nu_1 \lp \bar{L}^T\Sigma \h^{(1)} +\dbover{\Sigma}\h^{(2)}   \rp +\gamma\bar{L}^T\Sigma\Sigma\c\rp
\nonumber \\
   & & -\frac{1}{\sqrt{n}}\nu_1 \lp \h^{(1)} \rp^T\Sigma\c
+ \frac{1}{\sqrt{n}}\nu_1|\|c_2\g-\sigma \overline{U}\overline{\Sigma} \v\|_2 -c_{3,s}| +\frac{c_{3,s}^2}{m}+\gamma \c^T\Sigma\Sigma\c -\gamma c_2^2. \nonumber \\ \label{eq:ridgeta18a0b4c1d1}
\end{eqnarray}
Relying on concentrations and introducing scaling $c_3\rightarrow\frac{c_{3,s}}{\sqrt{m}}$ while keeping $c_3$, $c_2$, $\nu_1$, and $\gamma$ fixed, we then, analogously to (\ref{eq:ta18a0b4c3}), find
\begin{eqnarray}
\lim_{n\rightarrow\infty} \hspace{-.06in} \mE\cL_{2,r}\lp \hat{\x},\hat{\z},\gamma;c_3,c_2,\nu_1\rp
 & = &
\lim_{n\rightarrow\infty} \hspace{-.06in} \mE \Bigg ( \Big. -\lp -\frac{1}{2\sqrt{n}}\nu_1 \lp \bar{L}^T\Sigma \h^{(1)} +\dbover{\Sigma}\h^{(2)}   \rp +\gamma\bar{L}^T\Sigma\Sigma\c\rp^T
 \nonumber \\
 & &
\times
 \lp \lambda I +\gamma \lp \bar{L}^T\Sigma\Sigma\bar{L} + \dbover{\Sigma} \dbover{\Sigma} \rp  \rp^{-1}
\lp -\frac{1}{2\sqrt{n}}\nu_1 \lp \bar{L}^T\Sigma \h^{(1)} +\dbover{\Sigma}\h^{(2)}   \rp +\gamma\bar{L}^T\Sigma\Sigma\c\rp
\nonumber \\
   & & -\frac{1}{\sqrt{n}}\nu_1 \lp \h^{(1)} \rp^T\Sigma\c
+ \frac{1}{\sqrt{n}}\nu_1|\|c_2\g-\sigma \overline{U}\overline{\Sigma} \v\|_2 -c_{3,s}| +\frac{c_{3,s}^2}{m}
\nonumber \\
& &
+\gamma \c^T\Sigma\Sigma\c -\gamma c_2^2\Big.\Bigg ) \nonumber \\
  & = & \lim_{n\rightarrow\infty} \Bigg ( \Big. -\frac{\nu_1^2}{4n}\sum_{i=1}^n \frac{\s_i^2}{1+\gamma\s_i^2}
 - \sum_{i=1}^n \frac{\gamma^2\bar{\c}_i^2}{1+\gamma\s_i^2}
 +\nu_1\sqrt{\alpha}\left | \sqrt{c_2^2+\bar{\sigma}^2}-c_3\right | +c_3^2
 \nonumber \\
 & &
  + \gamma \sum_{i=1}^n \lp \s_i^{(1)}\rp^2\c_i^2  -\gamma c_2^2
  \Big.\Bigg ), \nonumber \\ \label{eq:ridgeta18a0b4c3}
\end{eqnarray}
with $\bar{\sigma}$ as in (\ref{eq:ta18a0b4c4}). Combining (\ref{eq:ridgeta18a0b1})-(\ref{eq:ridgeta18a0b2c1}), (\ref{eq:ridgeta18a0b4c1d1}), and  (\ref{eq:ridgeta18a0b4c3}), we arrive at the following
\begin{eqnarray}
\lim_{n\rightarrow\infty}  \mE_{\g,\h^{(1)},\h^{(2)},\v}f_{rd,r}
   & = & \lim_{n\rightarrow\infty} \min_{c_3,c_2\geq 0} \max_{\nu_1\geq 0,\gamma}  f_{0,r}(c_2,\nu_1,\gamma), \label{eq:ridgeta18a0b4c5}
\end{eqnarray}
where
\begin{eqnarray}
f_{0,r}(c_3,c_2,\nu_1,\gamma) \triangleq  -\frac{\nu_1^2}{4n}\sum_{i=1}^n \frac{\s_i^2}{1+\gamma\s_i^2}
 - \sum_{i=1}^n \frac{\gamma^2\bar{\c}_i^2}{1+\gamma\s_i^2}
 +\nu_1\sqrt{\alpha}\left | \sqrt{c_2^2+\bar{\sigma}^2}-c_3\right | +c_3^2
   + \gamma \sum_{i=1}^n \lp \s_i^{(1)}\rp^2\c_i^2  -\gamma c_2^2. \nonumber \\
    \label{eq:ridgeta18a0b4c6}
\end{eqnarray}
Optimization over $\nu_1$, we find the optimal
\begin{eqnarray}
\hat{\nu}_1 =
 \frac{2\sqrt{\alpha}\left | \sqrt{c_2^2+\bar{\sigma}^2} -c_3\right |}{\frac{1}{n}\sum_{i=1}^n \frac{\s_i^2}{\lambda+\gamma\s_i^2}}, \label{eq:ridgeta18a0b4c8}
\end{eqnarray}
which,  together with (\ref{eq:ridgeta18a0b4c6}), further gives
\begin{eqnarray}
\max_{\nu_1\geq 0} f_{0,r}(c_3,c_2,\nu_1,\gamma)
& = & f_{0,r}(c_3,c_2,\hat{\nu}_1,\gamma) \nonumber \\
& = &
  - \sum_{i=1}^n \frac{\gamma^2\bar{\c}_i^2}{\lambda+\gamma\s_i^2}
+ \frac{\alpha\lp \sqrt{c_2^2+\bar{\sigma}^2} -c_3\rp^2}{\frac{1}{n}\sum_{i=1}^n \frac{\s_i^2}{\lambda+\gamma\s_i^2}} +c_3^2 + \gamma \sum_{i=1}^n \lp\s_i^{(1)}\rp^2\c_i^2 -\gamma c_2^2.  \label{eq:ridgeta18a0b4c9}
\end{eqnarray}
From (\ref{eq:ridgeta18a0b4c5}) and (\ref{eq:ridgeta18a0b4c9}) we also find
\begin{eqnarray}
\lim_{n\rightarrow\infty}  \mE_{\g,\h^{(1)},\h^{(2)},\v}f_{rd,r}(\g,\h^{(1)},\h^{(2)},\v)
   & = & \lim_{n\rightarrow\infty}  \min_{c_3,c_2\geq 0} \max_{\gamma}  f_{0,r}(c_3,c_2,\hat{\nu}_1,\gamma). \label{eq:ridgeta18a0b4c10}
\end{eqnarray}
Optimization over $c_3$ gives the optimal
\begin{eqnarray}
\hat{c}_3=\frac{\sqrt{c_2^2+\bar{\sigma}^2}}{1+\frac{1}{\alpha n}\sum_{i=1}^n \frac{\s_i^2}{\lambda+\gamma\s_i^2}}, \label{eq:ridgeta18a0b4c10d0}
\end{eqnarray}
which, together with (\ref{eq:ridgeta18a0b4c9}), establishes
\begin{eqnarray}
\max_{c_3\geq 0} f_{0,r}(c_3,c_2,\hat{\nu}_1,\gamma) = f_{0,r}(\hat{c}_3,c_2,\hat{\nu}_1,\gamma) =
  - \sum_{i=1}^n \frac{\gamma^2\bar{\c}_i^2}{\lambda+\gamma\s_i^2}
+ \frac{c_2^2+\bar{\sigma}^2}{1+\frac{1}{\alpha n}\sum_{i=1}^n \frac{\s_i^2}{\lambda+\gamma\s_i^2}}   + \gamma \sum_{i=1}^n \lp \s_i^{(1)}\rp^2\c_i^2 -\gamma c_2^2. \nonumber \\
  \label{eq:ridgeta18a0b4c10d1}
\end{eqnarray}
Following further the program of Section \ref{sec:analgls}, we then take the derivatives with respect to $c_2$ and $\gamma$ and obtain
\begin{eqnarray}
\frac{d f_{0,r}(\hat{c}_3,c_2,\hat{\nu}_1,\gamma)}{d c_2} =
  \frac{2 c_2}{1+\frac{1}{\alpha n}\sum_{i=1}^n \frac{\s_i^2}{\lambda+\gamma\s_i^2}}   -2\gamma c_2,  \label{eq:ridgeta18a0b4c11}
\end{eqnarray}
and
\begin{eqnarray}
\frac{d f_{0,r}(\hat{c}_3,c_2,\hat{\nu}_1,\gamma)}{d \gamma}
& = &
  - \sum_{i=1}^n \frac{2\gamma\bar{\c}_i^2}{\lambda+\gamma\s_i^2}
  + \sum_{i=1}^n \frac{\gamma^2\s_i^2\bar{\c}_i^2}{\lp \lambda +\gamma\s_i^2\rp^2}
 +\frac{\lp c_2^2+\bar{\sigma}^2\rp}{\lp 1+ \frac{1}{\alpha n}\sum_{i=1}^n \frac{\s_i^2}{\lambda+\gamma\s_i^2}\rp^2}
 \lp \frac{1}{\alpha n}\sum_{i=1}^n \frac{\s_i^4}{\lp \lambda  +\gamma\s_i^2\rp^2}\rp
  \nonumber \\
 & &
 + \sum_{i=1}^n \lp \s_i^{(1)}\rp^2\c_i^2 - c_2^2. \label{eq:ridgeta18a0b4c12}
\end{eqnarray}
Equalling the above derivatives to zero,  we first from (\ref{eq:ridgeta18a0b4c11}) find that the optimal $\hat{\gamma}$ satisfies
\begin{eqnarray}
  \frac{1}{n}\sum_{i=1}^n \frac{\hat{\gamma}\s_i^2}{\lambda+\hat{\gamma}\s_i^2} =\alpha(1-\hat{\gamma}),  \label{eq:ridgeta18a0b4c13}
\end{eqnarray}
and, for such a $\hat{\gamma}$, we then from (\ref{eq:ridgeta18a0b4c12}) find  that the optimal $\hat{c}_2$ satisfies
\begin{equation}
  - \sum_{i=1}^n \frac{2\hat{\gamma}\bar{\c}_i^2}{\lambda+\hat{\gamma}\s_i^2}
  + \sum_{i=1}^n \frac{\hat{\gamma}^2\s_i^2\bar{\c}_i^2}{\lp \lambda +\hat{\gamma}\s_i^2\rp^2}
 + \sum_{i=1}^n \lp \s_i^{(1)}\rp^2\c_i^2 +\frac{\lp \hat{c}_2^2+\bar{\sigma}^2\rp}{\lp 1 + \frac{1}{\alpha n}\sum_{i=1}^n \frac{\s_i^2}{\lambda+\hat{\gamma}\s_i^2}\rp^2}
 \lp \frac{1}{\alpha n}\sum_{i=1}^n \frac{\s_i^4}{\lp \lambda +\hat{\gamma}\s_i^2\rp^2}\rp
   - \hat{c}_2^2=0.
     \label{eq:ridgeta18a0b4c14}
\end{equation}
After setting
\begin{eqnarray}
a_{1,r} & = &
  - \sum_{i=1}^n \frac{2\hat{\gamma}\bar{\c}_i^2}{\lambda+\hat{\gamma}\s_i^2}
  + \sum_{i=1}^n \frac{\hat{\gamma}^2\s_i^2\bar{\c}_i^2}{\lp \lambda +\hat{\gamma}\s_i^2\rp^2}
 + \sum_{i=1}^n \lp \s_i^{(1)}\rp^2\c_i^2
 \nonumber \\
a_{2,r}  &= &
\frac{  \frac{1}{n}\sum_{i=1}^n \frac{\s_i^4}{\lp \lambda+\hat{\gamma}\s_i^2\rp^2}}{\lp \alpha+ \frac{1}{  n}\sum_{i=1}^n \frac{\s_i^2}{\lambda +\hat{\gamma}\s_i^2}\rp^2},  \label{eq:ridgeta18a0b4c15}
\end{eqnarray}
we, from (\ref{eq:ridgeta18a0b4c14}), obtain
\begin{eqnarray}
\hat{c}_2^2= \frac{a_{1,r} + \alpha \bar{\sigma}^2a_{2,r} }{1- \alpha a_{2,r}}.  \label{eq:ridgeta18a0b4c16}
\end{eqnarray}
Finally, from (\ref{eq:ridgeta18a0b4c9}) and (\ref{eq:ridgeta18a0b4c10}), we then also have
\begin{eqnarray}
\lim_{n\rightarrow\infty}  \mE_{\g,\h^{(1)},\h^{(2)},\v}f_{rd,r}
   & = & \lim_{n\rightarrow\infty} \min_{c_2\geq 0} \max_{\gamma}  f_0(\hat{c}_3,c_2,\hat{\nu}_1,\gamma) \nonumber \\
   & = & \lim_{n\rightarrow\infty} f_0(\hat{c}_3,\hat{c}_2,\hat{\nu}_1,\hat{\gamma}) \nonumber \\
   & = &  \lim_{n\rightarrow\infty} \lp  - \sum_{i=1}^n \frac{\hat{\gamma}^2\bar{\c}_i^2}{\lambda+\hat{\gamma}\s_i^2}
+ \frac{\alpha\lp \hat{c}_2^2+\bar{\sigma}^2\rp}{\alpha +\frac{1}{n}\sum_{i=1}^n \frac{\s_i^2}{\lambda+\hat{\gamma}\s_i^2}} + \hat{\gamma} \sum_{i=1}^n \lp \s_i^{(1)}\rp^2\c_i^2 - \hat{\gamma} \hat{c}_2^2 \rp \nonumber \\
    & = &   \lim_{n\rightarrow\infty} \lp  - \sum_{i=1}^n \frac{\hat{\gamma}^2\bar{\c}_i^2}{\lambda+\hat{\gamma}\s_i^2}
+ \hat{\gamma}\lp \hat{c}_2^2+\bar{\sigma}^2\rp + \hat{\gamma} \sum_{i=1}^n \lp \s_i^{(1)}\rp^2\c_i^2 - \hat{\gamma} \hat{c}_2^2 \rp \nonumber \\
   & = &   \lim_{n\rightarrow\infty}  \lp  - \sum_{i=1}^n \frac{\hat{\gamma}^2\bar{\c}_i^2}{\lambda+\hat{\gamma}\s_i^2}
+ \hat{\gamma} \sum_{i=1}^n \lp \s_i^{(1)}\rp^2\c_i^2
+ \hat{\gamma}\bar{\sigma}^2 \rp, \label{eq:ridgeta18a10}
\end{eqnarray}
where $\hat{c}_2$ and $\hat{\gamma}$ are given through (\ref{eq:ridgeta18a0b4c13}), (\ref{eq:ridgeta18a0b4c15}), and (\ref{eq:ridgeta18a0b4c16}). Utilizing this very same $\hat{c}_2$ and $\hat{\gamma}$, one then, from (\ref{eq:ridgeta18a0b4c10d0}) and (\ref{eq:ridgeta18a0b4c8}), also finds the optimal $c_3$ and $\nu_1$
\begin{eqnarray}
\hat{c}_3=\frac{\sqrt{\hat{c}_2^2+\bar{\sigma}^2}}{1+\frac{1}{\alpha n}\sum_{i=1}^n \frac{\s_i^2}{\lambda+\gamma\s_i^2}}
=\hat{\gamma}\sqrt{\hat{c}_2^2+\bar{\sigma}^2},  \label{eq:ridgeta18a0b4c17d0}
\end{eqnarray}
\begin{eqnarray}
\hat{\nu}_1 =
 \frac{2\sqrt{\alpha}\left | \sqrt{\hat{c}_2^2+\bar{\sigma}^2} -\hat{c}_3 \right |}{\frac{1}{n}\sum_{i=1}^n \frac{\s_i^2}{\lambda+\hat{\gamma}\s_i^2}}
 =
 \frac{2\hat{\gamma}\left | \sqrt{\hat{c}_2^2+\bar{\sigma}^2} -\hat{c}_3\right |}{\sqrt{\alpha}(1-\hat{\gamma}) }
 =
 \frac{2\hat{\gamma} \sqrt{\hat{c}_2^2+\bar{\sigma}^2} }{\sqrt{\alpha} }. \label{eq:ridgeta18a0b4c17d1}
\end{eqnarray}
It is then also possible to determine the optimal $\hat{c}_4^2=\lim_{n\rightarrow \infty} \mE_{\g,\h^{(1)},\h^{(2)},\v}\|\hat{\x}\|_2^2$
\begin{eqnarray}
\hat{c}_4^2 & = & \lim_{n\rightarrow \infty} \mE_{\g,\h^{(1)},\h^{(2)},\v}\| \hat{\x}\|_2^2 \nonumber \\
 & = & \lim_{n\rightarrow \infty} \mE_{\g,\h^{(1)},\h^{(2)},\v}
 \left \|  \lp \lambda I +\gamma \lp \bar{L}^T\Sigma\Sigma\bar{L} + \dbover{\Sigma} \dbover{\Sigma} \rp  \rp^{-1}
 \lp -\frac{1}{2}\nu_s \lp \bar{L}^T\Sigma \h^{(1)} +\dbover{\Sigma}\h^{(2)}   \rp +\gamma\bar{L}^T\Sigma\Sigma\c\rp \right \|_2^2 \nonumber \\
 & = &
\lim_{n\rightarrow \infty}  \frac{\hat{\nu}_1^2}{4n}\sum_{i=1}^n\frac{\s_i^2}{\lp \lambda+\hat{\gamma}\s_i^2\rp^2}
+
\lim_{n\rightarrow \infty}  \sum_{i=1}^n\frac{\hat{\gamma}^2\bar{\c}_i^2}{\lp \lambda+\hat{\gamma}\s_i^2\rp^2}. \label{eq:ridgeta18a0b4c17d1e0}
\end{eqnarray}
The above discussion is summarized in the following theorem.

\begin{theorem}(Characterization of the FRM ridge estimator) For a given fixed real positive number $\lambda$, assume the setup of Lemma \ref{lemma:ridgelemma1} and Theorem \ref{thm:ridgethm1} with large $n$ linear regime ($\alpha=\lim_{n\rightarrow\infty}\frac{m}{n}$) and let  $\xi_{rr}(\lambda)$ be as in (\ref{eq:ridgerandlincons1}) (or (\ref{eq:ridgerandlincons1a0})) and $\xi_{rd,r}$ as in (\ref{eq:ridgeta16}). Also, let $\hat{\beta}=\beta_{rr}(\lambda)$ (where $\beta_{rr}(\lambda)$ is as in
(\ref{eq:model014})) and let $R(\bar{\beta},\hat{\beta})=R(\bar{\beta},\beta_{rr}(\lambda))$ from (\ref{eq:model020}) be the ridge estimator's excess prediction risk. For a unique $\hat{\gamma}$ that satisfies
\begin{eqnarray}
\lim_{n\rightarrow\infty} \frac{1}{n}\sum_{i=1}^n \frac{\hat{\gamma}\s_i^2}{\lambda+\hat{\gamma}\s_i^2} =\alpha (1-\hat{\gamma}),  \label{eq:ridgethmaaeq1}
\end{eqnarray}
set
\begin{eqnarray}
a_{1,r} & = &
  -   \lim_{n\rightarrow\infty}  \sum_{i=1}^n \frac{2\hat{\gamma}\bar{\c}_i^2}{\lambda+\hat{\gamma}\s_i^2}
  +   \lim_{n\rightarrow\infty} \sum_{i=1}^n \frac{\hat{\gamma}^2\s_i^2\bar{\c}_i^2}{\lp \lambda +\hat{\gamma}\s_i^2\rp^2}
 +   \lim_{n\rightarrow\infty}  \sum_{i=1}^n \lp \s_i^{(1)}\rp^2\c_i^2
 \nonumber \\
a_{2,r} & = &
\frac{\lim_{n\rightarrow\infty} \frac{1}{n}\sum_{i=1}^n \frac{\s_i^4}{\lp \lambda +\hat{\gamma}\s_i^2\rp^2}}{\lp \alpha + \lim_{n\rightarrow\infty} \frac{1}{n}\sum_{i=1}^n \frac{\s_i^2}{\lambda+\hat{\gamma}\s_i^2}\rp^2}
=\frac{\hat{\gamma}^2}{\alpha^2 } \lim_{n\rightarrow\infty} \frac{1}{n}\sum_{i=1}^n \frac{\s_i^4}{\lp \lambda+\hat{\gamma}\s_i^2\rp^2}.  \label{eq:ridgethmaaeq2}
\end{eqnarray}
One then has \vspace{-.0in}
\begin{eqnarray}
 \lim_{n\rightarrow\infty} \mE_{Z,\v,E}\xi_{rr}(\lambda)
 &= &\xi_{rp,r}=  \xi_{rd,r}  = \lim_{n\rightarrow\infty} \mE_{\g,\h^{(1)},\h^{(2)},\v} f_{rd,r}(\g,\h^{(1)},\h^{(2)},\v) \nonumber \\
& = &
 \lim_{n\rightarrow\infty}  \lp  - \sum_{i=1}^n \frac{\hat{\gamma}^2\bar{\c}_i^2}{\lambda+\hat{\gamma}\s_i^2}
+ \hat{\gamma} \sum_{i=1}^n \lp \s_i^{(1)}\rp^2\c_i^2
+ \hat{\gamma}\bar{\sigma}^2 \rp \nonumber \\
 \lim_{n\rightarrow\infty} \mE_{Z,\v,E} R(\bar{\beta},\beta_{rr}(\lambda))
& = &
\lim_{n\rightarrow\infty} \hat{c}_2^2
  =
 \frac{a_{1,r} +\alpha \bar{\sigma}^2a_{2,r} }{1-\alpha a_{2,r}} \nonumber \\
 & & \lim_{n\rightarrow\infty}\hat{\nu}_1
  = \frac{2\hat{\gamma}\sqrt{\hat{c}_2^2+\bar{\sigma}^2}}{\sqrt{\alpha}} \nonumber \\
 \lim_{n\rightarrow\infty} \mE_{Z,\v,E} \|\y-X\beta_{rr}(\lambda)\|_2^2
& = &
\lim_{n\rightarrow\infty} \hat{c}_3^2
  = \hat{\gamma}^2(\hat{c}_2^2+\bar{\sigma}^2) \nonumber \\
 \lim_{n\rightarrow\infty} \mE_{Z,\v,E} \|\beta_{rr}(\lambda)\|_2^2
& = &
 \lim_{n\rightarrow\infty}\hat{c}_4^2
 =
 \lim_{n\rightarrow \infty}  \frac{\hat{\nu}_1^2}{4n}\sum_{i=1}^n\frac{\s_i^2}{\lp \lambda+\hat{\gamma}\s_i^2\rp^2}
+
\lim_{n\rightarrow \infty}  \sum_{i=1}^n\frac{\hat{\gamma}^2\bar{\c}_i^2}{\lp \lambda+\hat{\gamma}\s_i^2\rp^2},
\label{eq:ridgethmaaeq3}
\end{eqnarray}
and for any fixed $\epsilon>0$
\begin{eqnarray}
 \lim_{n\rightarrow\infty}\mP_{Z,\v,E} \lp  (1-\epsilon)\mE_{Z,\v,E}\xi_{rr}(\lambda) \leq  \xi_{rr}(\lambda) \leq (1+\epsilon)\mE_{Z,\v,E}\xi_{rr}(\lambda)\rp
& \longrightarrow & 1 \nonumber \\
 \lim_{n\rightarrow\infty}\mP_{Z,\v,E} \lp  (1-\epsilon)\mE_{Z,\v,E} R(\bar{\beta},\beta_{rr}(\lambda))  \leq  R(\bar{\beta},\beta_{rr}(\lambda)) \leq (1+\epsilon)\mE_{Z,\v,E}R(\bar{\beta},\beta_{rr}(\lambda))\rp
& \longrightarrow & 1  \nonumber \\
 \lim_{n\rightarrow\infty}\mP_{\g,\h^{(1)},\h^{(2)},\v} \lp  (1-\epsilon)\hat{\nu}_1 \leq  \nu_1 \leq (1+\epsilon)\hat{\nu}_1\rp
& \longrightarrow & 1 \nonumber \\
 \lim_{n\rightarrow\infty}\mP_{Z,\v,E} \lp  (1-\epsilon)\mE_{Z,\v,E}\|\y-X\beta_{rr}(\lambda)\|_2^2 \leq \|\y-X\beta_{rr}(\lambda)\|_2^2  \leq (1+\epsilon)\mE_{Z,\v,E}\|\y-X\beta_{rr}(\lambda)\|_2^2 \rp
& \longrightarrow & 1\nonumber \\
 \lim_{n\rightarrow\infty}\mP_{Z,\v,E} \lp  (1-\epsilon)\mE_{Z,\v,E}\|\beta_{rr}(\lambda)\|_2^2 \leq \|\beta_{rr}(\lambda)\|_2^2  \leq (1+\epsilon)\mE_{Z,\v,E}\|\beta_{rr}(\lambda)\|_2^2 \rp
& \longrightarrow & 1.\nonumber \\
\label{eq:ridgethmaaeq4}
\end{eqnarray}
\label{thm:ridgethm2}
\end{theorem}\vspace{-.17in}
\begin{proof}
Follows from the above discussion, concentrations of $f_{rd,r}(\g,\h^{(1)},\h^{(2)},\v)$, $c_2$, and $\nu_1$, and utilization of the same arguments as in the proof of Theorem  \ref{thm:thm2}.
\end{proof}

\vspace{.1in}
\noindent \underline{ 4) \textbf{\emph{Double checking strong random duality:}}} As in Section \ref{sec:analgls}, the underlying convexity ensures that  the strong random duality and the reversal arguments of \cite{StojnicGorEx10,StojnicRegRndDlt10} apply as well.

%%%%%%%%%%%%%%%%%%%%%%%%%%%%%%%%%%%%%%%%%%%%%%%%%%%%%%%%%%%%%%%%%
%%%%%%%%%%%%%%%%%%%%%%%%%%%%%%%%%%%%%%%%%%%%%%%%%%%%%%%%%%%%%%%%%
%%%%%%%%%%%%%%%%%%%%%%%%%%%%%%%%%%%%%%%%%%%%%%%%%%%%%%%%%%%%%%%%%
\subsubsection{Compact risk form -- Ridge estimator}
\label{sec:compridge}
%%%%%%%%%%%%%%%%%%%%%%%%%%%%%%%%%%%%%%%%%%%%%%%%%%%%%%%%%%%%%%%%%
%%%%%%%%%%%%%%%%%%%%%%%%%%%%%%%%%%%%%%%%%%%%%%%%%%%%%%%%%%%%%%%%%
%%%%%%%%%%%%%%%%%%%%%%%%%%%%%%%%%%%%%%%%%%%%%%%%%%%%%%%%%%%%%%%%%

Analogously to what we have done in Section \ref{sec:compridge}, we can again obtain the risk characterization given in Theorem \ref{thm:ridgethm2} in a more compact matrix form. Very minimal adjustments are needed. We again start by observing that (\ref{eq:ridgethmaaeq1}) can be (analogously to (\ref{eq:compthmproofeq0a0})) rewritten as
\begin{eqnarray}
 1- \lim_{n\rightarrow\infty} \frac{1}{n}\sum_{i=1}^n \frac{\hat{\gamma}\s_i^2}{\lambda+\hat{\gamma}\s_i^2}
 & = & 1- \alpha (1-\hat{\gamma})
 \nonumber \\
\Longleftrightarrow \qquad  \lim_{n\rightarrow\infty} \frac{1}{n}\sum_{i=1}^n \frac{1}{\lambda+\hat{\gamma}\s_i^2}
 & = & \frac{1- \alpha (1-\hat{\gamma})}{\lambda},
 \label{eq:ridgecompthmproofeq0a0}
\end{eqnarray}
or, analogously to (\ref{eq:compthmproofeq0a1}), in a matrix form
\begin{eqnarray}
   \lim_{n\rightarrow\infty} \frac{1}{n}\tr\lp \lambda I+\hat{\gamma}D \rp^{-1}
 & = & \frac{1- \alpha (1-\hat{\gamma})}{\lambda} \nonumber \\
 \Longleftrightarrow \qquad
\lim_{n\rightarrow\infty} \frac{1}{n}\tr\lp \lambda I+\hat{\gamma} L^T A^TAL + \hat{\gamma}\dbover{A}^T \dbover{A} \rp^{-1}
 & = & \frac{1- \alpha (1-\hat{\gamma})}{\lambda},
 \label{eq:ridgecompthmproofeq0a1}
\end{eqnarray}
Similarly to (\ref{eq:compthmproofeq0}),  we can also rewrite $a_{1}$ from (\ref{eq:ridgethmaaeq2}) in the following way
\begin{eqnarray}
a_{1}
& = &
-\lim_{n\rightarrow\infty}\lp     \sum_{i=1}^n \frac{2\hat{\gamma}\bar{\c}_i^2}{\lambda+\hat{\gamma}\s_i^2} \rp
+\lim_{n\rightarrow\infty}\lp    \sum_{i=1}^n \frac{\hat{\gamma}^2\s_i^2\bar{\c}_i^2}{\lp \lambda+\hat{\gamma}\s_i^2\rp^2} \rp
+\lim_{n\rightarrow\infty}\lp    \sum_{i=1}^n \lp \s_i^{(1)}\rp^2\c_i^2\rp
 \nonumber \\
& = &
 \lim_{n\rightarrow\infty}  \lp -\hat{\gamma}  \bar{\c}^T  \lp \lambda I +\hat{\gamma} D\rp ^{-1}\bar{\c}
- \lambda \hat{\gamma} \bar{\c}^T  \lp \lambda I +\hat{\gamma} D\rp ^{-2}\bar{\c}
 + \c^T\Sigma\Sigma\c
 \rp.
 \label{eq:ridgecompthmproofeq0}
\end{eqnarray}
Adjusting for $\lambda$ , we find the following ridge analogues to (\ref{eq:compthmproofeq2})- (\ref{eq:compthmproofeq3})
 \begin{eqnarray}
   \lim_{n\rightarrow\infty}  \lp \bar{\c}^T \lp I+\hat{\gamma} D\rp^{-1}\bar{\c}
 \rp
    & = &
  \lim_{n\rightarrow\infty}   \lp
  \bar{\beta}^T A^TA L
\lp \lambda I + \hat{\gamma} L^T A^TA L + \hat{\gamma}\dbover{A}^T\dbover{A}    \rp^{-1}
    L^T A^TA\bar{\beta}
  \rp.
 \label{eq:ridgecompthmproofeq2}
\end{eqnarray}
and
\begin{eqnarray}
   \lim_{n\rightarrow\infty}  \lp \bar{\c}^T \lp I+\hat{\gamma} D\rp^{-2}\bar{\c}
 \rp
   & = &
 \lim_{n\rightarrow\infty}   \lp
  \bar{\beta}^T A^TA L
\lp \lambda I + \hat{\gamma} L^T A^TA L + \hat{\gamma}\dbover{A}^T\dbover{A}    \rp^{-2}
    L^T A^TA\bar{\beta}
  \rp.
 \label{eq:ridgecompthmproofeq3}
\end{eqnarray}
We also recall on
\begin{eqnarray}
   \lim_{n\rightarrow\infty}
   \lp
   \c^T\Sigma\Sigma \c
  \rp
  =
 \lim_{n\rightarrow\infty}
   \lp
\bar{\beta}^T A^T A \bar{\beta}
  \rp,
 \label{eq:ridgecompthmproofeq4}
\end{eqnarray}
rewrite $a_{2}$ from (\ref{eq:ridgethmaaeq2}) as
\begin{eqnarray}
 a_2 & =&
 \frac{\hat{\gamma}^2}{\alpha^2} \lim_{n\rightarrow\infty} \frac{1}{n}\sum_{i=1}^n \frac{\s_i^4}{\lp \lambda+\hat{\gamma}\s_i^2\rp^2}
  =
 \frac{\hat{\gamma}^2}{\alpha^2} \lim_{n\rightarrow\infty} \frac{1}{n}\tr \lp  \lambda D^{-1}+\hat{\gamma I}   \rp^{-2},
 \label{eq:ridgecompthmproofeq5}
\end{eqnarray}
and analogously to (\ref{eq:compthmproofeq6}) find
\begin{eqnarray}
a_{2}   & = &
 \frac{\hat{\gamma}^2}{\alpha^2} \lim_{n\rightarrow\infty} \frac{1}{n}\tr\lp \hat{\gamma}  I + \lambda \lp L^T A^TAL +  \dbover{A}^T \dbover{A}\rp^{-1} \rp^{-2}.
 \label{eq:ridgecompthmproofeq6}
\end{eqnarray}

We then have the following corollary of Theorem \ref{thm:ridgethm2} (basically a ridge analogue to Corollary \ref{cor:cor1}.

\begin{corollary}(FRM Ridge estimator - compact risk form) Consider the FRM from (\ref{eq:model01}) with factors $\bar{Z}$, loadings $L$, noises $\e$ and $\bar{E}$ and assume the setup of Theorem \ref{thm:thm2}. Let the factors and noises' covariances be $\Sigma_{\bar{Z}}$, $\Sigma_{\bar{E}}$, and $\Sigma_{\e}$, i.e., let
\begin{eqnarray}\label{eq:ridgecorcompeq1}
 \Sigma_{\bar{Z}}=A^TA,\quad  \Sigma_{\bar{E}}=\dbover{A}^T\dbover{A},\quad \Sigma_{\e}=\sigma^2\overline{A}^T\overline{A}.
 \end{eqnarray}
Let $\hat{\gamma}$ be the unique solution of
\begin{eqnarray}
\lim_{n\rightarrow\infty} \frac{1}{n}\tr\lp \lambda I+\hat{\gamma} L^T \Sigma_{\bar{Z}} L + \hat{\gamma} \Sigma_{\bar{E}} \rp^{-1}
 & = & \frac{1-\alpha(1-\hat{\gamma})}{\lambda},
\label{eq:ridgecorcompeq2}
\end{eqnarray}
Set
\begin{eqnarray}
a_{1,r}
& = &
 \lim_{n\rightarrow\infty}   \lp
\beta^T\Sigma_{\bar{Z}}\beta
-  \hat{\gamma}\beta^T \Sigma_{\bar{Z}} L \lp
\lp \lambda I + \hat{\gamma} L^T \Sigma_{\bar{Z}} L + \hat{\gamma}\Sigma_{\bar{E}}   \rp^{-1}
+
 \lambda \lp \lambda I + \hat{\gamma} L^T \Sigma_{\bar{Z}} L + \hat{\gamma}\Sigma_{\bar{E}}   \rp^{-2}
\rp    L^T \Sigma_{\bar{Z}}\beta
  \rp
\nonumber \\
a_{2,r}
  & = &
 \frac{\hat{\gamma}^2}{\alpha^2} \lim_{n\rightarrow\infty} \frac{1}{n}\tr\lp \hat{\gamma}  I  + \lambda \lp L^T  \Sigma_{\bar{Z}} L +  \Sigma_{\bar{E}}\rp^{-1} \rp^{-2} \nonumber \\
 \bar{\sigma}^2
 & = &
 \lim_{n\rightarrow\infty} \frac{1}{m}\tr\lp \Sigma_{\e}\rp.
\label{eq:ridgecorcompeq4}
\end{eqnarray}
 Then
\begin{eqnarray}\label{eq:ridgecorcompeq5}
  \lim_{n\rightarrow\infty} \mE_{Z,\v,E} R(\bar{\beta},\beta_{rr}(\lambda))
& = &
\lim_{n\rightarrow\infty} \mE_{Z,\v,E} \lp
\lp \bar{\beta}  - L\beta_{rr}(\lambda) \rp^T
\Sigma_{\bar{Z}}
\lp \bar{\beta}  - L\beta_{rr}(\lambda) \rp +   \beta_{rr}(\lambda)^T \Sigma_{\bar{E}} \beta_{rr}(\lambda)\rp \nonumber \\
& = &
 \frac{a_{1,r}+\alpha \bar{\sigma}^2a_{2,r} }{1-\alpha a_{2,r}}, \nonumber \\
\end{eqnarray}
and
\begin{eqnarray}
  \lim_{n\rightarrow\infty}\mP_{Z,\v,E} \lp  (1-\epsilon)\mE_{Z,\v,E} R(\bar{\beta},\beta_{rr}(\lambda))  \leq  R(\bar{\beta},\beta_{rr}(\lambda)) \leq (1+\epsilon)\mE_{Z,\v,E}R(\bar{\beta},\beta_{rr}(\lambda))\rp
& \longrightarrow & 1.\label{eq:ridgecorcompeq6}
\end{eqnarray}
 \label{cor:cor2}
\end{corollary}
\begin{proof}
Follows from Theorem  \ref{thm:ridgethm2} and (\ref{eq:ridgecompthmproofeq0a1}), (\ref{eq:ridgecompthmproofeq0}), (\ref{eq:ridgecompthmproofeq2})- (\ref{eq:ridgecompthmproofeq4}), and (\ref{eq:ridgecompthmproofeq6}).
\end{proof}

%%%%%%%%%%%%%%%%%%%%%%%%%%%%%%%%%%%%%%%%%%%%%%%%%%%%%%%%%%%%%%%%%
%%%%%%%%%%%%%%%%%%%%%%%%%%%%%%%%%%%%%%%%%%%%%%%%%%%%%%%%%%%%%%%%%
%%%%%%%%%%%%%%%%%%%%%%%%%%%%%%%%%%%%%%%%%%%%%%%%%%%%%%%%%%%%%%%%%
%%%%%%%%%%%%%%%%%%%%%%%%%%%%%%%%%%%%%%%%%%%%%%%%%%%%%%%%%%%%%%%%%
%%%%%%%%%%%%%%%%%%%%%%%%%%%%%%%%%%%%%%%%%%%%%%%%%%%%%%%%%%%%%%%%%
%%%%%%%%%%%%%%%%%%%%%%%%%%%%%%%%%%%%%%%%%%%%%%%%%%%%%%%%%%%%%%%%%
\subsection{Excess risk of the LS estimator}
\label{sec:analls}
%%%%%%%%%%%%%%%%%%%%%%%%%%%%%%%%%%%%%%%%%%%%%%%%%%%%%%%%%%%%%%%%%
%%%%%%%%%%%%%%%%%%%%%%%%%%%%%%%%%%%%%%%%%%%%%%%%%%%%%%%%%%%%%%%%%
%%%%%%%%%%%%%%%%%%%%%%%%%%%%%%%%%%%%%%%%%%%%%%%%%%%%%%%%%%%%%%%%%
%%%%%%%%%%%%%%%%%%%%%%%%%%%%%%%%%%%%%%%%%%%%%%%%%%%%%%%%%%%%%%%%%
%%%%%%%%%%%%%%%%%%%%%%%%%%%%%%%%%%%%%%%%%%%%%%%%%%%%%%%%%%%%%%%%%
%%%%%%%%%%%%%%%%%%%%%%%%%%%%%%%%%%%%%%%%%%%%%%%%%%%%%%%%%%%%%%%%%

Mathematical analysis preceding Lemma \ref{lemma:lemma2} and Theorem \ref{thm:thm2} holds even if $\lambda=0$ provided that $\alpha>1$ and $\alpha\kappa<1$. This fits the scenario of the LS estimator (see (\ref{eq:model017})-(\ref{eq:model019})),
 and we then have the following  corollary of Theorem \ref{thm:ridgethm2} that particularly relates to the LS estimator and under-parameterized regimes.

\begin{corollary}(Characterization of FRM LS estimator) Assume the setup of Theorem \ref{thm:ridgethm2} with $\lambda =0$ and $\alpha>1$. Let unique $\hat{\gamma}$ be such that
\begin{eqnarray}
\hat{\gamma}=1-\frac{1}{\alpha}.  \label{eq:lsridgethmaaeq1}
\end{eqnarray}
Set
\begin{eqnarray}
a_{1,r} & = &
  -   \lim_{n\rightarrow\infty} \sum_{i=1}^n \frac{\bar{\c}_i^2}{\s_i^2}
  +   \lim_{n\rightarrow\infty} \sum_{i=1}^n \lp \s_i^{(1)}\rp^2\c_i^2 \nonumber \\
 a_{2,r} & = & \frac{1}{\alpha^2}.  \label{eq:lsridgethmaaeq2}
\end{eqnarray}
One then has \vspace{-.0in}
\begin{eqnarray}
 \lim_{n\rightarrow\infty} \mE_{Z,\v,E}\xi_{ls}
 &= &
  \lim_{n\rightarrow\infty} \mE_{Z,\v,E}\xi_{rr}(0)
=
  \hat{\gamma}\lp a_{1,r}+\bar{\sigma}^2\rp =   \lp a_{1,r}+\bar{\sigma}^2\rp\frac{\alpha -1}{\alpha}\nonumber \\
 \lim_{n\rightarrow\infty} \mE_{Z,\v,E} R(\bar{\beta},\beta_{ls})
& = &
 \lim_{n\rightarrow\infty} \mE_{Z,\v,E} R(\bar{\beta},\beta_{ls}(0))
  =
 \frac{a_{1,r}+\alpha \bar{\sigma}^2a_{2,r} }{1-\alpha a_{2,r}}=
\frac{ \alpha  a_{1,r}+\bar{\sigma}^2 }{\alpha -1} \nonumber \\
 & & \lim_{n\rightarrow\infty} \hat{\nu}_1
  = 2\sqrt{a_{1,r}+\bar{\sigma}^2}\frac{\sqrt{\alpha -1}}{\alpha} \nonumber \\
 \lim_{n\rightarrow\infty} \mE_{Z,\v,E} \|\y-X\beta_{ls}\|_2^2
& = &
 \lim_{n\rightarrow\infty} \mE_{Z,\v,E} \|\y-X\beta_{rr}(0)\|_2^2
= \lim_{n\rightarrow\infty}\hat{c}_3^2
 =   \lp a_{1,r}+\bar{\sigma}^2\rp\frac{\alpha -1}{\alpha} \nonumber \\
 \lim_{n\rightarrow\infty} \mE_{Z,\v,E} \|\beta_{ls}\|_2^2
& = &
 \lim_{n\rightarrow\infty} \mE_{Z,\v,E} \|\beta_{rr}(0)\|_2^2
   =
  \lim_{n\rightarrow \infty}  \sum_{i=1}^n\frac{\bar{\c}_i^2}{\s_i^4}
  +
\frac{a_{1,r}+\bar{\sigma}^2}{\alpha-1} \lim_{n\rightarrow \infty}  \frac{1}{n}\sum_{i=1}^n\frac{1}{\s_i^2},\nonumber \\
\label{eq:lsridgethmaaeq3}
\end{eqnarray}
and for any fixed $\epsilon>0$
\begin{eqnarray}
 \lim_{n\rightarrow\infty}\mP_{Z,\v,E} \lp  (1-\epsilon)\mE_{Z,\v,E}\xi_{ls} \leq  \xi_{ls} \leq (1+\epsilon)\mE_{Z,\v,E}\xi_{ls}\rp
& \longrightarrow & 1 \nonumber \\
 \lim_{n\rightarrow\infty}\mP_{Z,\v,E} \lp  (1-\epsilon)\mE_{Z,\v,E} R(\bar{\beta},\beta_{ls})  \leq  R(\bar{\beta},\beta_{ls}) \leq (1+\epsilon)\mE_{Z,\v,E}R(\bar{\beta},\beta_{ls})\rp
& \longrightarrow & 1  \nonumber \\
 \lim_{n\rightarrow\infty}\mP_{\g,\h^{(1)},\h^{(2)},\v} \lp  (1-\epsilon)\hat{\nu}_1 \leq  \nu_1 \leq (1+\epsilon)\hat{\nu}_1\rp
& \longrightarrow & 1 \nonumber \\
 \lim_{n\rightarrow\infty}\mP_{Z,\v,E} \lp  (1-\epsilon)\mE_{Z,\v,E}\|\y-X\beta_{ls}\|_2^2 \leq \|\y-X\beta_{ls}\|_2^2  \leq (1+\epsilon)\mE_{Z,\v,E}\|\y-X\beta_{ls}\|_2^2 \rp
& \longrightarrow & 1\nonumber \\
 \lim_{n\rightarrow\infty}\mP_{Z,\v,E} \lp  (1-\epsilon)\mE_{Z,\v,E}\|\beta_{ls}\|_2^2 \leq \|\beta_{ls}\|_2^2  \leq (1+\epsilon)\mE_{Z,\v,E}\|\beta_{ls}\|_2^2 \rp
& \longrightarrow & 1.\label{eq:lsridgethmaaeq4}
\end{eqnarray}
\label{cor:cor3}
\end{corollary}\vspace{-.17in}
\begin{proof}
Follows directly from Theorem \ref{thm:ridgethm2} after noting that the discussion preceding theorem holds for $\lambda=0$. Plugging $\lambda=0$ in (\ref{eq:ridgethmaaeq1}),  (\ref{eq:ridgethmaaeq2}), and  (\ref{eq:ridgethmaaeq3}) then gives (\ref{eq:lsridgethmaaeq1}),  (\ref{eq:lsridgethmaaeq2}), and  (\ref{eq:lsridgethmaaeq3}).
\end{proof}

%%%%%%%%%%%%%%%%%%%%%%%%%%%%%%%%%%%%%%%%%%%%%%%%%%%%%%%%%%%%%%%%%
%%%%%%%%%%%%%%%%%%%%%%%%%%%%%%%%%%%%%%%%%%%%%%%%%%%%%%%%%%%%%%%%%
%%%%%%%%%%%%%%%%%%%%%%%%%%%%%%%%%%%%%%%%%%%%%%%%%%%%%%%%%%%%%%%%%
\subsubsection{Compact risk form -- LS}
\label{sec:compls}
%%%%%%%%%%%%%%%%%%%%%%%%%%%%%%%%%%%%%%%%%%%%%%%%%%%%%%%%%%%%%%%%%
%%%%%%%%%%%%%%%%%%%%%%%%%%%%%%%%%%%%%%%%%%%%%%%%%%%%%%%%%%%%%%%%%
%%%%%%%%%%%%%%%%%%%%%%%%%%%%%%%%%%%%%%%%%%%%%%%%%%%%%%%%%%%%%%%%%

As was the case in Sections \ref{sec:compgls} and \ref{sec:compridge}, we can again obtain the risk characterization given in Corollary \ref{cor:cor3} in a more compact matrix form. This time, however, the analytical forms are much simpler. In particular, we only need
  to rewrite $a_{1}$ from (\ref{eq:lsridgethmaaeq2}). Similarly to (\ref{eq:compthmproofeq0}), we have
\begin{eqnarray}
a_{1,r}
& = &
-\lim_{n\rightarrow\infty}\lp     \sum_{i=1}^n \frac{\bar{\c}_i^2}{\s_i^2} \rp
+\lim_{n\rightarrow\infty}\lp    \sum_{i=1}^n \lp \s_i^{(1)}\rp^2\c_i^2\rp
= \lim_{n\rightarrow\infty}  \lp -  \bar{\c}^T  D ^{-1}\bar{\c}
  + \c^T\Sigma\Sigma\c
 \rp.
 \label{eq:lsridgecompthmproofeq0}
\end{eqnarray}
Adjusting for $\lambda$ , we find the following ridge analogues to (\ref{eq:compthmproofeq2})
 \begin{eqnarray}
   \lim_{n\rightarrow\infty}  \lp \bar{\c}^T D^{-1}\bar{\c}
 \rp
    & = &
  \lim_{n\rightarrow\infty}   \lp
  \bar{\beta}^T A^TA L
\lp   L^T A^TA L + \dbover{A}^T\dbover{A}    \rp^{-1}
    L^T A^TA\bar{\beta}
  \rp.
 \label{eq:lsridgecompthmproofeq2}
\end{eqnarray}
 Recalling on (\ref{eq:ridgecompthmproofeq4}) and combining (\ref{eq:lsridgecompthmproofeq0})
 and (\ref{eq:lsridgecompthmproofeq2}), we find
  \begin{eqnarray}
a_{1,r}
     & = &
  \lim_{n\rightarrow\infty}
  \bar{\beta}^T  \lp A^TA - A^TA L
\lp   L^T A^TA L +  \dbover{A}^T\dbover{A}    \rp^{-1}
    L^T A^TA \rp \bar{\beta}
  \nonumber \\
      & = &
  \lim_{n\rightarrow\infty}
  \bar{\beta}  \lp   \lp  A^TA\rp^{-1}  +  L \lp \dbover{A}^T\dbover{A}\rp^{-1} L^T   \rp^{-1}
     \bar{\beta},
 \label{eq:lsridgecompthmproofeq3}
\end{eqnarray}
where the matrix inversion lemma ensures that the second equality holds. This is sufficient to establish the following corollary.

\begin{corollary}(FRM LS estimator - compact risk form) Assume the setup of Corollaries  \ref{cor:cor1} and \ref{cor:cor2}.
  Then
\begin{eqnarray}\label{eq:lscorcompeq5}
  \lim_{n\rightarrow\infty} \mE_{Z,\v,E} R(\bar{\beta},\beta_{ls})
& = &
 \frac{\alpha}{\alpha-1}  \lim_{n\rightarrow\infty}   \lp
  \bar{\beta}^T \lp   \Sigma_{\bar{Z}}^{-1}  +  L \Sigma_{\bar{E}}^{-1} L^T   \rp^{-1}
     \bar{\beta}
  \rp
  +
  \frac{1}{\alpha-1} \lim_{n\rightarrow\infty} \frac{1}{m}\tr\lp \Sigma_{\e}\rp,
  \nonumber \\
\end{eqnarray}
and
\begin{eqnarray}
  \lim_{n\rightarrow\infty}\mP_{Z,\v,E} \lp  (1-\epsilon)\mE_{Z,\v,E} R(\bar{\beta},\beta_{ls})   \leq  R(\bar{\beta},\beta_{ls})\leq (1+\epsilon)\mE_{Z,\v,E}R(\bar{\beta},\beta_{ls})\rp
& \longrightarrow & 1.\label{eq:lscorcompeq6}
\end{eqnarray}
 \label{cor:cor4}
\end{corollary}
\begin{proof}
Follows from Corollary  \ref{cor:cor3} and (\ref{eq:lsridgecompthmproofeq3}).
\end{proof}

Corollary \ref{cor:cor4}   provides a convenient way to express the concentrating point of the LS excess risk so that
the bias and variance terms are clearly distinguished and obtained directly as functions of the key covariance matrices and loadings coefficients. When all statistical components of the model are uncorrelated the above excess risk expression further simplifies and becomes only a function of the loadings $L$
\begin{equation}
\mbox{(\textbf{Fully uncorrealted FRM):}} \quad \lim_{n\rightarrow\infty} \mE_{Z,\v,E} R(\bar{\beta},\beta_{ls}) \hspace{-.02in}
  =
 \frac{ \alpha  }{\alpha -1}  \lim_{n\rightarrow\infty}  \lp   \bar{\beta}^T
\lp L L^T  +  I \rp^{-1}
   \bar{\beta}
  \rp
 +\frac{\sigma^2 }{\alpha -1}.
  \label{eq:lseq3}
\end{equation}

%%%%%%%%%%%%%%%%%%%%%%%%%%%%%%%%%%%%%%%%%%%%%%%%%%%%%%%%%%%%%%%%%
%%%%%%%%%%%%%%%%%%%%%%%%%%%%%%%%%%%%%%%%%%%%%%%%%%%%%%%%%%%%%%%%%
%%%%%%%%%%%%%%%%%%%%%%%%%%%%%%%%%%%%%%%%%%%%%%%%%%%%%%%%%%%%%%%%%
%%%%%%%%%%%%%%%%%%%%%%%%%%%%%%%%%%%%%%%%%%%%%%%%%%%%%%%%%%%%%%%%%
\section{An analytical connection between FRM and LRM}
\label{sec:analfrmlrm}
%%%%%%%%%%%%%%%%%%%%%%%%%%%%%%%%%%%%%%%%%%%%%%%%%%%%%%%%%%%%%%%%%
%%%%%%%%%%%%%%%%%%%%%%%%%%%%%%%%%%%%%%%%%%%%%%%%%%%%%%%%%%%%%%%%%
%%%%%%%%%%%%%%%%%%%%%%%%%%%%%%%%%%%%%%%%%%%%%%%%%%%%%%%%%%%%%%%%%
%%%%%%%%%%%%%%%%%%%%%%%%%%%%%%%%%%%%%%%%%%%%%%%%%%%%%%%%%%%%%%%%%

As mention earlier, the factor regression considered in this paper is tightly connected to the classical linear regression. We discussed the conceptual differences earlier when the mathematical setup of the FRM was introduced. Despite these differences, the two models can still be analytically connected. We discuss the connection in this section.

To make this connection clearer, it is convenient that we recall on the (statistical) mathematical form of the FRM  from (\ref{eq:model05})
\begin{eqnarray}
\y=\bar{Z}\bar{\beta}+\e=ZA\bar{\beta}+\sigma\overline{A}\v, \quad  X=\bar{Z}L+\bar{E}=ZAL+E\dbover{A}. \label{eq:conecmodel01}
\end{eqnarray}
We assume $\overline{A}=I$ and observe that the rows of $X$ are uncorrelated and that the covariance of each row $X_{i,:},1\leq i\leq m,$ is
\begin{eqnarray}
\Sigma_{X}=L^T\Sigma_{\bar{Z}}L+\Sigma_{\bar{E}}. \label{eq:conecmodel02}
\end{eqnarray}
We now consider the LRM with such an $X$, i.e. we consider
\begin{eqnarray}
\dbover{\y}=\dbover{X}\dbover{\beta}+\dbover{\sigma}\v, \label{eq:conecmodel03}
\end{eqnarray}
where the covariance of each row of $\dbover{X}_{i,:}$ is
\begin{eqnarray}
\Sigma_{\dbover{X}} =L^T\Sigma_{\bar{Z}}L+\Sigma_{\bar{E}}, \label{eq:conecmodel04}
\end{eqnarray}
and
\begin{eqnarray}
 \dbover{\beta} & = & \Sigma_{X}^{-1}L^T\Sigma_{\bar{Z}}\bar{\beta} \nonumber \\
 \dbover{\sigma} & = & \sqrt{\sigma^2+ \bar{\beta}^T\lp L \Sigma_{\bar{E}}^{-1}  L^T +\Sigma_{\bar{Z}}^{-1}      \rp^{-1}\bar{\beta}  }, \label{eq:conecmodel05}
\end{eqnarray}
Given the independence over $i$, we gave that for any function $f:\mR^{n+1}\rightarrow \mR$
\begin{eqnarray}
\lp \lp \y_i,X_{i,:}\rp \Longleftrightarrow \lp \dbover{\y}_i,\dbover{X}_{i,:}\rp \rp
\quad \Longrightarrow \quad
\lp f\lp \y_i,X_{i,:}\rp  \Longleftrightarrow  f\lp \dbover{\y}_i,\dbover{X}_{i,:}\rp \rp, \label{eq:conecmodel06}
\end{eqnarray}
where the equivalences are in the probabilistic sense. In other words, if we can show that the pairs $\lp \y_i,X_{i,:}\rp$ and $\lp \dbover{\y}_i,\dbover{X}_{i,:}\rp$ are statistically identical, we will then have that $f\lp \y_i,X_{i,:}\rp$ and $f\lp \dbover{\y}_i,\dbover{X}_{i,:}\rp$ are statistically identical as well.

By the definitions and (\ref{eq:conecmodel02}) and (\ref{eq:conecmodel04}) we have for the covariances of $X$ and $\dbover{X}$, $\Sigma_{X}$ and $\Sigma_{\dbover{X}}$,
\begin{eqnarray}
\Sigma_{X}=\mE \lp \lp X_{i,:}\rp^TX_{i,:} \rp =\Sigma_{\dbover{X}}=\mE\lp \lp \dbover{X}_{i,:}\rp^T\dbover{X}
_{i,:} \rp = L^T\Sigma_{\bar{Z}}L+\Sigma_{\bar{E}}. \label{eq:conecmodel07}
\end{eqnarray}
We then also have for the cross-correlations
\begin{eqnarray}
 \mE \lp \y_iX_{i,:}\rp
 = \mE \lp  \bar{\beta}^T \lp\bar{Z}_{i,:}\rp^TX_{i,:} \rp
 = \mE \lp \bar{\beta}^T \lp\bar{Z}_{i,:}\rp^T \bar{Z}_{i,:}L \rp
 =\bar{\beta}^T\Sigma_{\bar{Z}} L, \label{eq:conecmodel08}
\end{eqnarray}
and
\begin{eqnarray}
 \mE \lp \dbover{\y}_i \dbover{X}_{i,:}\rp
 = \mE\lp\dbover{\beta}^T   \lp \dbover{X}_{i,:}\rp^T \dbover{X}_{i,:}\rp
 =\dbover{\beta}^T    \Sigma_{\dbover{X}}. \label{eq:conecmodel09}
\end{eqnarray}
Combining (\ref{eq:conecmodel05}), (\ref{eq:conecmodel07}), (\ref{eq:conecmodel08}), and (\ref{eq:conecmodel09}), we then find
\begin{eqnarray}
 \mE \lp \dbover{\y}_i \dbover{X}_{i,:}\rp
  =\dbover{\beta}^T    \Sigma_{\dbover{X}}
  =\bar{\beta}^T \Sigma_{\bar{Z}} L \Sigma_{\X}^{-1} \Sigma_{\dbover{X}}
  =\bar{\beta}^T \Sigma_{\bar{Z}} L =   \mE \lp \y_iX_{i,:}\rp. \label{eq:conecmodel010}
\end{eqnarray}
We also have for the variance of $\y_i$
\begin{eqnarray}
 \mE \lp \y_i^2\rp
= \mE \lp \bar{\beta}^T \lp\bar{Z}_{i,:}\rp^T \bar{Z}_{i,:}\bar{\beta} \rp +\sigma^2\mE \v_i^2
= \bar{\beta}^T \Sigma_{\bar{Z}} \bar{\beta} +\sigma^2, \label{eq:conecmodel011}
\end{eqnarray}
and
\begin{eqnarray}
 \mE \lp \dbover{\y}_i^2\rp
& = & \mE \lp \dbover{\beta}^T \lp\dbover{X}_{i,:}\rp^T \dbover{X}_{i,:}\dbover{\beta} \rp +\dbover{\sigma}^2\mE \v_i^2
= \dbover{\beta}^T \Sigma_{\dbover{X}} \dbover{\beta} +\dbover{\sigma}^2
  \nonumber \\
& = &
\lp \Sigma_{X}^{-1}L^T\Sigma_{\bar{Z}}\bar{\beta}\rp^T  \Sigma_{\dbover{X}}  \Sigma_{X}^{-1}L^T\Sigma_{\bar{Z}}\bar{\beta}
+\bar{\beta}^T\lp L \Sigma_{\bar{E}}^{-1}  L^T +\Sigma_{\bar{Z}}^{-1}      \rp^{-1}\bar{\beta} + \sigma^2 \nonumber \\
& = &
\bar{\beta}^T \Sigma_{\bar{Z}}L \Sigma_{X}^{-1} L^T\Sigma_{\bar{Z}}\bar{\beta}
+\bar{\beta}^T\lp L \Sigma_{\bar{E}}^{-1}  L^T +\Sigma_{\bar{Z}}^{-1}      \rp^{-1}\bar{\beta}  + \sigma^2 \nonumber \\
& = &
\bar{\beta}^T \Sigma_{\bar{Z}}L   \lp L^T\Sigma_{\bar{Z}}L+\Sigma_{\bar{E}} \rp^{-1} L^T\Sigma_{\bar{Z}}\bar{\beta}
+\bar{\beta}^T\lp L \Sigma_{\bar{E}}^{-1}  L^T +\Sigma_{\bar{Z}}^{-1}      \rp^{-1}\bar{\beta}  + \sigma^2  \nonumber \\
& = & \bar{\beta}^T \Sigma_{\bar{Z}} \bar{\beta} +\sigma^2 \nonumber  \\
& = &  \mE \lp \y_i^2\rp, \label{eq:conecmodel012}
\end{eqnarray}
where the last equality is a consequence of the matrix inversion lemma
\begin{eqnarray}
 \lp L \Sigma_{\bar{E}}^{-1}  L^T +\Sigma_{\bar{Z}}^{-1}      \rp^{-1}
& = &  \Sigma_{\bar{Z}}
-
  \Sigma_{\bar{Z}}L  \lp L^T\Sigma_{\bar{Z}}L+\Sigma_{\bar{E}} \rp^{-1} L^T\Sigma_{\bar{Z}}. \label{eq:conecmodel013}
\end{eqnarray}
A combination of  (\ref{eq:conecmodel07}), (\ref{eq:conecmodel010}), and (\ref{eq:conecmodel012}) and the fact that $\y,X,\dbover{\y}$, and $\dbover{X}$ are formed as zero-mean linearly  Gaussian, gives that
$\lp \lp \y_i,X_{i,:}\rp \Longleftrightarrow \lp \dbover{\y}_i,\dbover{X}_{i,:}\rp \rp$ and  that (\ref{eq:conecmodel06}) indeed holds.

One also needs to adjust for the differences in the definitions of the excess risk in two models. For example, for the LRM risk we have (see also \cite{Stojnicridge24})
\begin{eqnarray}
R_{LRM}(\dbover{\beta},\hat{\beta})
& = &
\lp\dbover{\beta}-\hat{\beta}\rp^T
\Sigma_{\dbover{X}}
\lp\dbover{\beta}-\hat{\beta}\rp \nonumber \\
& = &
\dbover{\beta}^T \Sigma_{\dbover{X}}\dbover{\beta}
-2\dbover{\beta}^T\Sigma_{\dbover{X}}\hat{\beta}
+\hat{\beta}^T\Sigma_{\dbover{X}}\hat{\beta} \nonumber \\
& = &
\lp  \Sigma_{X}^{-1}L^T\Sigma_{\bar{Z}}\bar{\beta}
 \rp^T \Sigma_{\dbover{X}} \Sigma_{X}^{-1}L^T\Sigma_{\bar{Z}}\bar{\beta}
-2\lp \Sigma_{X}^{-1}L^T\Sigma_{\bar{Z}}\bar{\beta}
\rp^T\Sigma_{\dbover{X}}\hat{\beta}
+\hat{\beta}^T\Sigma_{\dbover{X}}\hat{\beta} \nonumber \\
& = &
 \bar{\beta}^T \Sigma_{\bar{Z}} L    \Sigma_{X}^{-1}
 L^T\Sigma_{\bar{Z}}\bar{\beta}
-2 \bar{\beta}^T \Sigma_{\bar{Z}} L    \Sigma_{X}^{-1} \Sigma_{\dbover{X}}\hat{\beta}
+\hat{\beta}^T\Sigma_{\dbover{X}}\hat{\beta} \nonumber \\
& = &
 \bar{\beta}^T \Sigma_{\bar{Z}} L   \lp  L^T\Sigma_{\bar{Z}} L  + \Sigma_{\bar{E}} \rp^{-1}
 L^T\Sigma_{\bar{Z}}\bar{\beta}
-2 \bar{\beta}^T \Sigma_{\bar{Z}} L    \hat{\beta}
+\hat{\beta}^T\lp  L^T\Sigma_{\bar{Z}} L  + \Sigma_{\bar{E}} \rp\hat{\beta}.
\label{eq:conecmodel014}
\end{eqnarray}
On the other hand, we here for the FRM have
\begin{eqnarray}
R_{FRM}(\bar{\beta},\hat{\beta})=\lp \bar{\beta}-L\hat{\beta}\rp^T \Sigma_{\bar{Z}}\lp\bar{\beta}-L\hat{\beta}\rp
+ \hat{\beta}^T\Sigma_{\bar{E}}\hat{\beta}
= \bar{\beta}^T\Sigma_{\bar{Z}}\bar{\beta} -2\hat{\beta}^TL^T\Sigma_{\bar{Z}}\bar{\beta}
+  \hat{\beta}^T L^T\Sigma_{\bar{Z}} L\hat{\beta}
+ \hat{\beta}^T\Sigma_{\bar{E}}\hat{\beta}. \label{eq:conecmodel015}
\end{eqnarray}
The above basically means that the FRM excess risks are connected to the corresponding LRM ones  via the following relation
 \begin{eqnarray}
R_{FRM}(\bar{\beta},\hat{\beta})=
R_{LRM}(\dbover{\beta},\hat{\beta})+
 \bar{\beta}^T\Sigma_{\bar{Z}}\bar{\beta}
-
 \bar{\beta}^T \Sigma_{\bar{Z}} L   \lp  L^T\Sigma_{\bar{Z}} L  + \Sigma_{\bar{E}} \rp^{-1}
 L^T\Sigma_{\bar{Z}}\bar{\beta}. \label{eq:conecmodel016}
\end{eqnarray}
A combination of (\ref{eq:conecmodel013}) and (\ref{eq:conecmodel016}) finally gives
 \begin{eqnarray}
R_{FRM}(\bar{\beta},\hat{\beta})=
R_{LRM}(\dbover{\beta},\hat{\beta})+
 \bar{\beta}^T \lp L \Sigma_{\bar{E}}^{-1}  L^T +\Sigma_{\bar{Z}}^{-1}      \rp^{-1}
\bar{\beta}. \label{eq:conecmodel017}
\end{eqnarray}

%%%%%%%%%%%%%%%%%%%%%%%%%%%%%%%%%%%%%%%%%%%%%%%%%%%%%%%%%%%%%%%%%
%%%%%%%%%%%%%%%%%%%%%%%%%%%%%%%%%%%%%%%%%%%%%%%%%%%%%%%%%%%%%%%%%
%%%%%%%%%%%%%%%%%%%%%%%%%%%%%%%%%%%%%%%%%%%%%%%%%%%%%%%%%%%%%%%%%
%%%%%%%%%%%%%%%%%%%%%%%%%%%%%%%%%%%%%%%%%%%%%%%%%%%%%%%%%%%%%%%%%
\section{Numerical results}
\label{sec:numresall}
%%%%%%%%%%%%%%%%%%%%%%%%%%%%%%%%%%%%%%%%%%%%%%%%%%%%%%%%%%%%%%%%%
%%%%%%%%%%%%%%%%%%%%%%%%%%%%%%%%%%%%%%%%%%%%%%%%%%%%%%%%%%%%%%%%%
%%%%%%%%%%%%%%%%%%%%%%%%%%%%%%%%%%%%%%%%%%%%%%%%%%%%%%%%%%%%%%%%%
%%%%%%%%%%%%%%%%%%%%%%%%%%%%%%%%%%%%%%%%%%%%%%%%%%%%%%%%%%%%%%%%%

%%%%%%%%%%%%%%%%%%%%%%%%%%%%%%%%%%%%%%%%%%%%%%%%%%%%%%%%%%%%%%%%%
%%%%%%%%%%%%%%%%%%%%%%%%%%%%%%%%%%%%%%%%%%%%%%%%%%%%%%%%%%%%%%%%%
%%%%%%%%%%%%%%%%%%%%%%%%%%%%%%%%%%%%%%%%%%%%%%%%%%%%%%%%%%%%%%%%%
%%%%%%%%%%%%%%%%%%%%%%%%%%%%%%%%%%%%%%%%%%%%%%%%%%%%%%%%%%%%%%%%%
\subsection{A generic setup}
\label{sec:numresuncorr}
%%%%%%%%%%%%%%%%%%%%%%%%%%%%%%%%%%%%%%%%%%%%%%%%%%%%%%%%%%%%%%%%%
%%%%%%%%%%%%%%%%%%%%%%%%%%%%%%%%%%%%%%%%%%%%%%%%%%%%%%%%%%%%%%%%%
%%%%%%%%%%%%%%%%%%%%%%%%%%%%%%%%%%%%%%%%%%%%%%%%%%%%%%%%%%%%%%%%%
%%%%%%%%%%%%%%%%%%%%%%%%%%%%%%%%%%%%%%%%%%%%%%%%%%%%%%%%%%%%%%%%%

The above theoretical analysis is supplemented with the results obtained through numerical simulations. Let $\bar{\cA}(q)$ be a matrix defined via its $ij$-th entry, $\bar{\cA}_{ij}(q)$, as
 \begin{eqnarray}
\bar{\cA}_{ij}(q)=q^{|j-i|}.  \label{eq:numuncorr1}
\end{eqnarray}
Let $\cA(q)$ then be
 \begin{eqnarray}
\cA(q)=\frac{1}{2} \lp I + \bar{\cA}(q)\rp.  \label{eq:numuncorr2}
\end{eqnarray}
 For the covariance matrices, we take
\begin{eqnarray}
A^TA & = &\cA(q)
\nonumber \\
\overline{A}^T\overline{A} &= & \cA(q_v)
\nonumber \\
\dbover{A}^T\dbover{A} &= & \sigma_e^2\cA(q_e).
  \label{eq:numuncorr3}
\end{eqnarray}
To ensure that $L$ is generated so that it retains the same form as dimensions change,
we take it as the leading $k$ eigenvectors of $\cA(q_L)$. In Figure \ref{fig:fig1}, we choose $q=0.5,q_v=0,q_e=0.3,q_L=0.3,\sigma_e^2=0.01$, $n=600$, and
\begin{eqnarray}
\bar{\beta}=\frac{1}{\sqrt{n}}\begin{bmatrix}
               1 & 1 & \dots & 1
             \end{bmatrix}.
  \label{eq:numuncorr4}
\end{eqnarray}
As Figure \ref{fig:fig1} shows, dependence on the over-parametrization ratio, one needs to ensure that the ratio $\frac{k}{m}$ is kept fixed. We take for all simulations $\kappa=\frac{k}{m}=\frac{1}{2}$ and $\sigma^2=0.2$. There are three curves and each of them corresponds to one of the studied estimators. The theoretical predictions of the GLS and LS estimators are full green and blue curves and the corresponding one of the \emph{optimally tuned} (over $\lambda$) ridge estimator is the dotted red curve. The dots on the curves correspond to the simulated values. Even for fairly small simulated dimensions, we observe an excellent agreement between the theoretical predictions and simulations. As we choose to keep $n$ fixed, as $\frac{1}{\alpha}$ increases both $m$ and $k$ are decreasing. Although it is not necessarily easy to think of relations of such dimensions as being linear the theoretical predictions remain very accurate.
Also, due to decreased dimensions $m$ and $k$, one observes that the null-risk drops from its theoretical value. To see why this happens let us first discuss the theoretical null-risk limit. Recalling on (\ref{eq:model020}) we have for $\hat{\beta}=0$
\begin{eqnarray}
  R(\bar{\beta},0) =
\left  \|\Sigma V^T \bar{\beta}\right \|_2^2
=  \bar{\beta}^T V\Sigma\Sigma V^T \bar{\beta}
= \bar{\beta}^T  A^TA \bar{\beta}
=\bar{\beta}^T \cA(q) \bar{\beta}
=\frac{1}{2}\bar{\beta}^T\lp I + \bar{\cA}(q)\rp \bar{\beta}
=\frac{1}{k}\sum_{i=1}^{k} (k-i)q^i.
  \label{eq:numuncorr5}
\end{eqnarray}
In the limit when $n,k\rightarrow\infty$ we have
\begin{eqnarray}
\lim_{k\rightarrow\infty}  R(\bar{\beta},0)
& = &\lim_{k\rightarrow\infty} \frac{1}{k}\sum_{i=1}^{k} (k-i)q^i
 =\lim_{k\rightarrow\infty} \sum_{i=1}^{k} q^i -\lim_{k\rightarrow\infty}\frac{1}{k}\sum_{i=1}^{k}i q^i \nonumber \\
& = & \frac{1}{1-q}
 -\lim_{k\rightarrow\infty}\frac{1}{k}\frac{d}{dq}\sum_{i=1}^{k}q^{i+1}
 = \frac{1}{1-q}
 -\lim_{k\rightarrow\infty}\frac{1}{k}\frac{d}{dq}\frac{1}{1-q} \nonumber \\
& = & \frac{1}{1-q}
 -\lim_{k\rightarrow\infty}\frac{1}{k} \frac{1}{(1-q)^2}
 =  \frac{1}{1-q}.
   \label{eq:numuncorr5}
\end{eqnarray}
For $q=0.5$ one has $\lim_{k\rightarrow\infty}  R(\bar{\beta},0)=2$ which is exactly as the figure shows for lower values of $\frac{1}{\alpha}$. For these values, $m$ and $k$ are sufficiently large that the above limiting calculations kicks in. On the other hand, as $\frac{1}{\alpha}$ increases, due to \emph{finite dimensions effect}, the true null-risk deviates from the above limiting calculations. Nonetheless the theoretical predictions for both the GLS and the ridge estimator's excess risks remain fully accurate (when $\frac{1}{\alpha}>1$, the LS estimator is not defined). We particularly selected an example where the null-risk changes to showcase the ability of the theoretical risk predictions to remain in a perfect agreement with the simulations.

\begin{figure}[h]
%\begin{minipage}[b]{.5\linewidth}
\centering
\centerline{\includegraphics[width=1\linewidth]{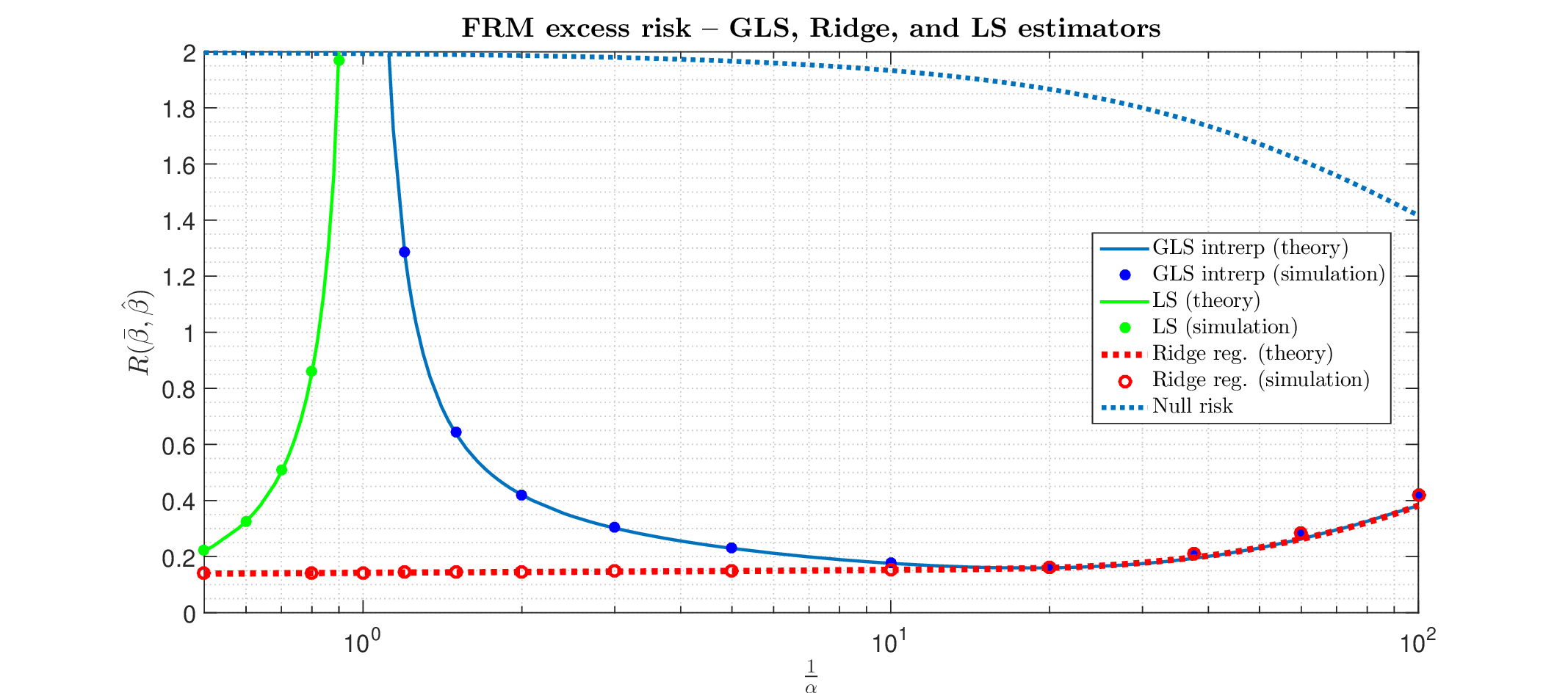}}
%\end{minipage}
%\begin{minipage}[b]{.5\linewidth}
%\centering
%\centerline{\epsfig{figure=finprerral08.eps,width=9cm,height=6.5cm}}
%\end{minipage}
\caption{Excess risk -- low SNR;  Covariance matrices are: $A=\cA(q)$, and $\dbover{A}=\cA(q_e)$; $q=0.5,q_e=0.3$.}
\label{fig:fig1}
\end{figure}

The choice of the scaling of loading matrices is done very particularly as well. Namely, one can note that we do not have dimension dependent scaling which we should have if the loadings are unitary. Instead, we selected a small $0.01$ but fixed scaling of the noise $E$'s covariance. Ultimately we have for the factors associated SNR, $\mbox{SNR}_f$
\begin{eqnarray}
\mbox{SNR}_f
& = & \frac{\tr\lp A LL^T A^T\rp}{\tr\lp \Sigma_{\bar{E}}\rp}
=\frac{\tr\lp A^TA\rp}{\tr\lp \dbover{A}^T\dbover{A} \rp}
=\frac{\tr\lp \cA(q)\rp}{\tr\lp \sigma_{e}^2 \cA(q_e)\rp} \nonumber \\
& = &\frac{\tr\lp \frac{1}{2}\lp 1+\bar{\cA}(q) \rp\rp}{\tr\lp \sigma_{e}^2 \frac{1}{2}\lp 1+\bar{\cA} (q_e)\rp \rp}
=\frac{k}{\sigma_e^2 n}=\frac{1}{\sigma_e^2}\frac{k}{m}\frac{m}{n}=\frac{\alpha\kappa}{\sigma_e^2}.
    \label{eq:numuncorr6}
 \end{eqnarray}
Since $\kappa=\frac{1}{2}$ and $\sigma_e^2=0.01$ are fixed, we obtain that the factor SNR behaves as
\begin{eqnarray}
\mbox{SNR}_f =\frac{\alpha\kappa}{\sigma_e^2} =\frac{50}{\frac{1}{\alpha}},
    \label{eq:numuncorr7}
 \end{eqnarray}
i.e., it is a decaying function of the over-parametrization ratio $\frac{1}{\alpha}$. For example for $\frac{1}{\alpha}=60$ we have
$\mbox{SNR}_f =\frac{5}{6}$, i.e. a fairly low SNR ratio. Yet the risk still remains well below the null one. It should also be noted, that by a no surprise the risk curve exhibits U-shape lack of monotonicity in the over-parameterized regime. However, what is a bit remarkable is that it takes rather strong over-parameterizations of the order of a few tens to ensure that the risk curve starts turning upward. As the above example shows, before reaching such over-parameterizations it remarkably remains way below the corresponding null-risk one (the excess risk for both the GLS and optimally tuned ridge is about ten times smaller than the null-risk). This is in part due to the fact that the SNR of $\y$ linearity is strong ($\mbox{SNR}_{\y}=\frac{\|\bar{\beta}\|_2^2}{\sigma^2}=5$) and in part due to a residual factor effect and the fact that the number of factors is smaller than the sample size, i.e., $\kappa=\frac{1}{2}$. As we will see in the next subsection, when everything is scaled properly the FRM model, in addition to exhibiting the double-descent phenomenon, also allows for basic interpolator and optimally tuned ridge regressor to achieve a monotonic behavior as the over-parametrization increases.

%%%%%%%%%%%%%%%%%%%%%%%%%%%%%%%%%%%%%%%%%%%%%%%%%%%%%%%%%%%%%%%%%
%%%%%%%%%%%%%%%%%%%%%%%%%%%%%%%%%%%%%%%%%%%%%%%%%%%%%%%%%%%%%%%%%
%%%%%%%%%%%%%%%%%%%%%%%%%%%%%%%%%%%%%%%%%%%%%%%%%%%%%%%%%%%%%%%%%
%%%%%%%%%%%%%%%%%%%%%%%%%%%%%%%%%%%%%%%%%%%%%%%%%%%%%%%%%%%%%%%%%
\subsection{An uncorrelated setup}
\label{sec:numresuncor}
%%%%%%%%%%%%%%%%%%%%%%%%%%%%%%%%%%%%%%%%%%%%%%%%%%%%%%%%%%%%%%%%%
%%%%%%%%%%%%%%%%%%%%%%%%%%%%%%%%%%%%%%%%%%%%%%%%%%%%%%%%%%%%%%%%%
%%%%%%%%%%%%%%%%%%%%%%%%%%%%%%%%%%%%%%%%%%%%%%%%%%%%%%%%%%%%%%%%%
%%%%%%%%%%%%%%%%%%%%%%%%%%%%%%%%%%%%%%%%%%%%%%%%%%%%%%%%%%%%%%%%%

As a specific numerical setup, we now take $\Sigma_{\bar{Z}}=I$, $\Sigma_{\bar{E}}=I$, and $LL^T=c_LI$.  We then from (\ref{eq:lseq3})  have for the LS estimator
\begin{equation}
  \lim_{n\rightarrow\infty} \mE_{Z,\v,E} R(\bar{\beta},\beta_{ls}) \hspace{-.02in}
  =
 \frac{ \alpha  }{\alpha -1} \frac{1}{c_L+1}
 +\frac{\sigma^2 }{\alpha -1}.
  \label{eq:nrseq1}
\end{equation}
For the GLS estimator, we utilize (\ref{eq:corcompeq2}) to obtain $\hat{\gamma}$ as a unique positive solution of
\begin{eqnarray}
 \frac{\alpha\kappa}{1+\hat{\gamma}\lp c_L+1\rp} + \frac{1-\alpha\kappa}{1+\hat{\gamma}}
 & = & 1-\alpha.
\label{eq:nrseq2}
\end{eqnarray}
We then find
\begin{eqnarray}
\hat{\gamma}=
\frac{\sqrt{\alpha^2 c_L^2 - 2 \alpha c_L ((c_L + 2) \alpha \kappa - 1) + (c_L \alpha\kappa + 1)^2} + \alpha (c_L + 2) - c_L \alpha\kappa - 1}{2 (1-\alpha ) (c_L + 1)}.
\label{eq:nrseq2a0}
\end{eqnarray}
For such $\hat{\gamma}$, we, from (\ref{eq:corcompeq4}), then also find
\begin{eqnarray}
a_{1}
& = &
1-\frac{\hat{\gamma} c_L}{1+\hat{\gamma}\lp c_L+1\rp}-\frac{\hat{\gamma} c_L}{\lp 1+\hat{\gamma}\lp c_L+1\rp\rp^2}
 \nonumber \\
a_{2}
  & = &
   \frac{\hat{\gamma}^2}{\alpha^2} \lp  \frac{{\alpha\kappa \lp c_L+1 \rp^2}}{\lp 1+\hat{\gamma} (c_L+1) \rp^2}
+\frac{1-\alpha\kappa}{\lp \hat{\gamma}+1 \rp^2} \rp
 \nonumber \\
 \bar{\sigma}^2
 & = &
 \sigma^2.
\label{eq:nrseq3}
\end{eqnarray}
 Then
\begin{eqnarray}\label{eq:nrseq4}
  \lim_{n\rightarrow\infty} \mE_{Z,\v,E} R(\bar{\beta},\beta_{gls})
& = &
 \frac{a_1+\alpha \bar{\sigma}^2a_2 }{1-\alpha a_2}, \nonumber \\
\end{eqnarray}
For the Ridge estimator we utilize (\ref{eq:ridgecorcompeq2}) to first obtain $\hat{\gamma}$ as a unique positive solution of
\begin{eqnarray}
 \frac{\alpha\kappa}{\lambda+\hat{\gamma}\lp c_L+1\rp} + \frac{1-\alpha\kappa}{\lambda+\hat{\gamma}}
 & = & \frac{1-\alpha(1-\hat{\gamma})}{\lambda},
\label{eq:nrseq5}
\end{eqnarray}
and then
\begin{eqnarray}
a_{1,r}
& = &
1-\frac{\hat{\gamma} c_L}{\lambda+\hat{\gamma}\lp c_L+1\rp}-\frac{\lambda\hat{\gamma} c_L}{\lp \lambda+\hat{\gamma}\lp c_L+1\rp\rp^2}
\nonumber \\
a_{2,r}
  & = &
   \frac{\hat{\gamma}^2}{\alpha^2} \lp  \frac{{\alpha\kappa \lp c_L+1 \rp^2}}{\lp \lambda+\hat{\gamma} (c_L+1) \rp^2}
+\frac{1-\alpha\kappa}{\lp \hat{\gamma}+\lambda \rp^2} \rp
 \nonumber \\
 \bar{\sigma}^2
 & = &
 \sigma^2,
\label{eq:nrseq6}
\end{eqnarray}
 and finally
\begin{eqnarray}\label{eq:nrseq7}
  \lim_{n\rightarrow\infty} \mE_{Z,\v,E} R(\bar{\beta},\beta_{rr}(\lambda))
& = &
  \frac{a_{1,r}+\alpha \bar{\sigma}^2a_{2,r} }{1-\alpha a_{2,r}}.
\end{eqnarray}
The above uncorrelated scenario is very similar the one considered in numerical simulations in \cite{BSMW22} and  in both numerical simulations and the theoretical analysis of the latent space model in \cite{HMRT22}. In Figure \ref{fig:fig2} we show the obtained results for fixed $n=600$, $\kappa=0.5$, $\sigma^2=0.2$, while varying $\alpha$ and $c_L=\frac{4n}{k}=\frac{4}{\alpha\kappa}$ . As earlier, the full/solid curves are theoretical predictions and the corresponding dots are the simulated results. In Table \ref{tab:tab1}, we also supplement the visualized representation from Figure \ref{fig:fig2} with several concrete numerical values. One ratio, $\frac{1}{\alpha}=0.7$, from the under-parameterized and the other, $\frac{1}{\alpha}=3$, from the over-parameterized regime are selected. An excellent agreement between the theoretical and simulated values for all key optimizing quantities is observed. Moreover, already for a mild over-parametrization of $\frac{1}{\alpha}=3$, the GLS risk generalizes fairly well and (despite having a zero training error) is fairly close to the ridge one. Needless to say that the GLS requires no tuning and that ridge results are obtained for optimally tuned $\lambda$ (which practically may not always be achievable). This is exactly in agreement with the neural network zero-training observations.

\begin{figure}[h]
%\begin{minipage}[b]{.5\linewidth}
\centering
\centerline{\includegraphics[width=1\linewidth]{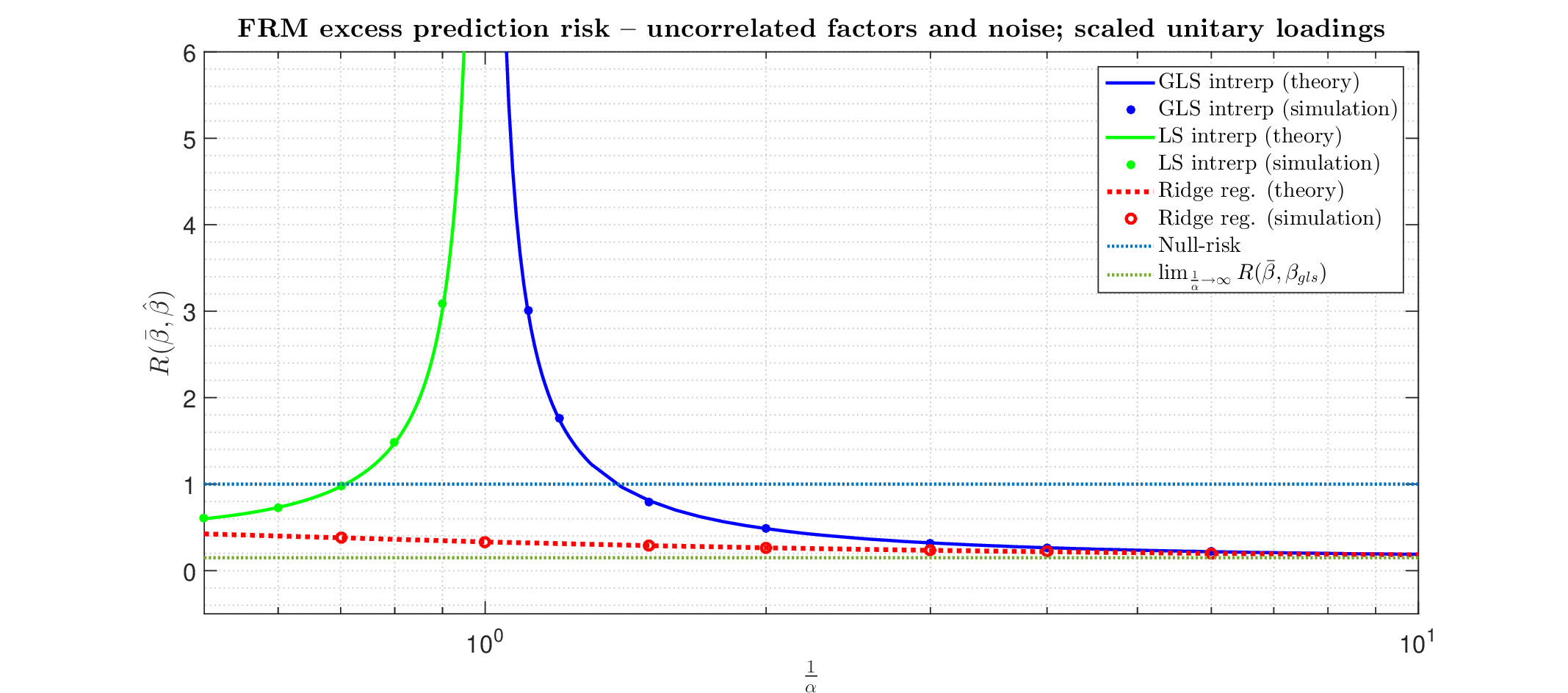}}
%\end{minipage}
%\begin{minipage}[b]{.5\linewidth}
%\centering
%\centerline{\epsfig{figure=finprerral08.eps,width=9cm,height=6.5cm}}
%\end{minipage}
\caption{Excess risk; uncorrelated factors and noise and scaled unitary loadings}
\label{fig:fig2}
\end{figure}

 \begin{table}[h]
  \caption{Optimizing quantities associated with the FRM excess risk -- \textbf{theoretical/\bl{simulated}}}
  \label{tab:tab1}
  \centering
 {\small
  \begin{tabular}{l||c|c||c|c}
    \hline\hline
  \textbf{Over-parametrization} $\frac{1}{\alpha}$ & \multicolumn{2}{c||}{$\mathbf{0.7}$}  & \multicolumn{2}{c}{$\mathbf{3}$}                   \\
    \hline
   \textbf{Estimator}  & \dgr{\textbf{LS}}  & \dgr{\textbf{Ridge}} &  \dgr{\textbf{GLS}} & \dgr{\textbf{Ridge}}
   \\
         \hline\hline
    \dgr{ \textbf{Estimator's norm:}} $\mE\|\hat{\beta}\|_2^2$ & $\mathbf{0.4535}/\bl{\mathbf{0.4546}}$  & $\mathbf{0.0976}/\bl{\mathbf{0.0978}}$           & $\mathbf{0.1031}/\bl{\mathbf{0.1028}}$ &  $\mathbf{0.0408}/\bl{\mathbf{0.0407}}$\\
         \hline
    \dgr{ \textbf{Optimizing objective:}} $\mE\xi$     & $\mathbf{0.1051}/\bl{\mathbf{0.1059}}$  & $\mathbf{0.3455}/\bl{\mathbf{0.3462}}$           & $\mathbf{0.1031}/\bl{\mathbf{0.1028}}$ & $\mathbf{0.1602}/\bl{\mathbf{0.1602}}$
   \\
         \hline\hline
    \dgr{ \textbf{In-sample risk:}} $\mE\frac{1}{m}\|\y-X\hat{\beta}\|_2^2$  & $\mathbf{0.1051}/\bl{\mathbf{0.1059}}$  &  $\mathbf{1.9730}/\bl{\mathbf{1.9765}}$   & $\mathbf{0}/\bl{\mathbf{0}}$  &  $\mathbf{0.0503}/\bl{\mathbf{0.0505}}$
    \\
         \hline\hline
   \dgr{\textbf{Out-of-sample risk:}} $\mE R\lp\bar{\beta},\hat{\beta}\rp$  & $\mathbf{0.9751}/\bl{\mathbf{0.9756}}$  & $\mathbf{0.3764}/\bl{\mathbf{0.3759}}$           & $\mathbf{0.3205}/\bl{\mathbf{0.3205}}$ & $\mathbf{0.2347}/\bl{\mathbf{0.2373}}$  \\
         \hline\hline
  \end{tabular}}
\end{table}

%%%%%%%%%%%%%%%%%%%%%%%%%%%%%%%%%%%%%%%%%%%%%%%%%%%%%%%%%%%%%%%%%
%%%%%%%%%%%%%%%%%%%%%%%%%%%%%%%%%%%%%%%%%%%%%%%%%%%%%%%%%%%%%%%%%
%%%%%%%%%%%%%%%%%%%%%%%%%%%%%%%%%%%%%%%%%%%%%%%%%%%%%%%%%%%%%%%%%
%%%%%%%%%%%%%%%%%%%%%%%%%%%%%%%%%%%%%%%%%%%%%%%%%%%%%%%%%%%%%%%%%
\subsubsection{Monotonicity in the over-parameterized regime}
\label{sec:crate}
%%%%%%%%%%%%%%%%%%%%%%%%%%%%%%%%%%%%%%%%%%%%%%%%%%%%%%%%%%%%%%%%%
%%%%%%%%%%%%%%%%%%%%%%%%%%%%%%%%%%%%%%%%%%%%%%%%%%%%%%%%%%%%%%%%%
%%%%%%%%%%%%%%%%%%%%%%%%%%%%%%%%%%%%%%%%%%%%%%%%%%%%%%%%%%%%%%%%%
%%%%%%%%%%%%%%%%%%%%%%%%%%%%%%%%%%%%%%%%%%%%%%%%%%%%%%%%%%%%%%%%%

Given the above explicit solution for the GLS excess risk, one can precisely determine the type of monotonicity. As the over-parametrization $\frac{1}{\alpha}\rightarrow\infty$, one clearly has $\alpha\rightarrow 0$. We also have for the factor SNR
\begin{eqnarray}
\mbox{SNR}_f=\frac{\tr\lp L^TA^TAL\rp}{\tr\lp\Sigma_{\bar{E}}\rp}=\frac{c_L k}{n}.
\label{eq:nrseq0}
\end{eqnarray}
To have it properly scaled, we set
\begin{eqnarray}
c_L=\frac{c_ln}{k}=\frac{c_l}{\alpha\kappa},
 \label{eq:nrseq0a0}
\end{eqnarray}
where the following simply confirms that $c_l$ is exactly the SNR of factor part of the model,
\begin{eqnarray}
\mbox{SNR}_f =\frac{c_L k}{n}=c_l.
\label{eq:nrseq0a1}
\end{eqnarray}
Keeping in mind that $c_l$ and $\kappa$ are fixed and $\alpha\rightarrow 0$, we from (\ref{eq:nrseq2a0}) further have
\begin{eqnarray}
\hat{\gamma}
& = &
\frac{\sqrt{\alpha^2 c_L^2 - 2 \alpha c_L ((c_L + 2) \alpha \kappa - 1) + (c_L \alpha\kappa + 1)^2} + \alpha (c_L + 2) - c_L \alpha\kappa - 1}{2 (1-\alpha ) (c_L + 1)}  \nonumber \\
& = &
\alpha\kappa\frac{\sqrt{\frac{c_l^2}{\kappa^2} +2\frac{c_l}{\kappa}  - 4 \alpha c_l -2\frac{c_l^2}{\kappa}  + (c_l + 1)^2} + \alpha (c_L + 2) - c_L \alpha\kappa - 1}{2 (1-\alpha ) (c_l + \alpha\kappa)}  \nonumber \\
& \rightarrow &
\alpha\kappa\frac{\sqrt{\frac{c_l^2}{\kappa^2} +2\frac{c_l}{\kappa}  -2\frac{c_l^2}{\kappa}  + (c_l + 1)^2}  +\frac{c_l}{\kappa} -c_l - 1}{2 c_l }
\nonumber \\
& \rightarrow & \alpha\kappa z,
\label{eq:nrseq8}
\end{eqnarray}
where
\begin{eqnarray}
 z=\frac{\sqrt{\frac{c_l^2}{\kappa^2} +2\frac{c_l}{\kappa}  -2\frac{c_l^2}{\kappa}  + (c_l + 1)^2} +\frac{c_l}{\kappa} -c_l- 1}{2 c_l }.
\label{eq:nrseq8a0}
\end{eqnarray}
A combination of (\ref{eq:nrseq0a0}) and  (\ref{eq:nrseq8} gives
\begin{eqnarray}
\hat{\gamma}c_L \rightarrow 0\quad \mbox{and}\quad
\hat{\gamma}c_L
 & \rightarrow & c_L\alpha\kappa z =c_l z,
\label{eq:nrseq8a1}
\end{eqnarray}
From (\ref{eq:nrseq3}), we then further have
\begin{eqnarray}
a_{1}
& = &
1-\frac{\hat{\gamma} c_L}{1+\hat{\gamma}\lp c_L+1\rp}-\frac{\hat{\gamma} c_L}{\lp 1+\hat{\gamma}\lp c_L+1\rp\rp^2}
=1
-\frac{\hat{\gamma} c_L\lp 2+\hat{\gamma}\lp c_L+1\rp\rp}{\lp 1+\hat{\gamma}\lp c_L+1\rp\rp^2}
\nonumber \\
&  \rightarrow &
1
-\frac{zc_l\lp 2+zc_l\rp}{\lp 1+zc_l\rp^2}
= \frac{1}{\lp 1+zc_l\rp^2}
 \nonumber \\
\alpha a_{2}
  & = &
   \frac{\hat{\gamma}^2}{\alpha^2} \lp  \frac{{\alpha^2\kappa \lp c_L+1 \rp^2}}{\lp 1+\hat{\gamma} (c_L+1) \rp^2}
+\alpha\frac{1-\alpha\kappa}{\lp \hat{\gamma}+1 \rp^2} \rp
\rightarrow \kappa^2 z^2 \frac{c_l^2}{\kappa\lp 1+c_lz \rp^2}
= \frac{\kappa z^2c_l^2}{\lp 1+c_lz \rp^2}
 \nonumber \\
 \bar{\sigma}^2
 & = &
 \sigma^2.
\label{eq:nrseq9}
\end{eqnarray}
Finally, from (\ref{eq:nrseq7}), we obtain
\begin{eqnarray}\label{eq:nrseq10}
  \lim_{\alpha\rightarrow 0}   \lim_{n\rightarrow\infty} \mE_{Z,\v,E} R(\bar{\beta},\beta_{gls})
& = &
 \lim_{\alpha\rightarrow 0}  \frac{a_{1}+\alpha \bar{\sigma}^2a_{2} }{1-\alpha a_{2}}
  =    \frac{1 + \kappa z^2c_l^2\bar{\sigma}^2 }{\lp 1+c_lz \rp^2- \kappa z^2c_l^2}.
\end{eqnarray}
For $c_l=4$ and $\kappa=\frac{1}{2}$ as in the numerical example above, we  have
\begin{eqnarray}\label{eq:nrseq11}
  \lim_{\alpha\rightarrow 0}   \lim_{n\rightarrow\infty} \mE_{Z,\v,E} R(\bar{\beta},\beta_{gls})
& \approx &
0.14963.
\end{eqnarray}
as the limiting risk.

Moreover, it is not that difficult to see that choosing $\kappa$ sufficiently small the risk can be made arbitrarily close to zero. For example, as $\kappa\rightarrow 0$, one has from (\ref{eq:nrseq8a0})
\begin{eqnarray}
 z=\frac{\sqrt{\frac{c_l^2}{\kappa^2} +2\frac{c_l}{\kappa}  -2\frac{c_l^2}{\kappa}  + (c_l + 1)^2} +\frac{c_l}{\kappa} -c_l- 1}{2 c_l }
 \rightarrow \frac{1}{\kappa}.
\label{eq:nrseq12}
\end{eqnarray}
From (\ref{eq:nrseq10}) we then find
\begin{eqnarray}\label{eq:nrseq13}
  \lim_{\kappa\rightarrow 0}\lim_{\alpha\rightarrow 0}   \lim_{n\rightarrow\infty} \mE_{Z,\v,E} R(\bar{\beta},\beta_{gls})
& = &
    \lim_{\kappa\rightarrow 0}  \frac{1 + \kappa z^2c_l^2\bar{\sigma}^2 }{\lp 1+c_lz \rp^2- \kappa z^2c_l^2}
    =
        \lim_{\kappa\rightarrow 0}  \frac{1 + \frac{1}{\kappa}c_l^2\bar{\sigma}^2 }{\lp 1+ \frac{c_l}{\kappa} \rp^2- \frac{c_l^2}{\kappa}}
        \nonumber \\
&    = &
        \lim_{\kappa\rightarrow 0}  \frac{ \frac{1}{\kappa}c_l^2\bar{\sigma}^2 }{ \frac{c_l^2}{\kappa^2} - \frac{c_l^2}{\kappa}}
  =    \lim_{\kappa\rightarrow 0}  \frac{ \frac{1}{\kappa}c_l^2\bar{\sigma}^2 }{ \frac{c_l^2}{\kappa^2} }
  =    \lim_{\kappa\rightarrow 0}  \kappa\sigma^2=0.
\end{eqnarray}
In other words, one can get consistency (vanishing risk as over-parametrization increases) even in a fully linear regime and appropriately (constantly) scaled both $\mbox{SNR}_f$ and $\mbox{SNR}_{\y}$.

%%%%%%%%%%%%%%%%%%%%%%%%%%%%%%%%%%%%%%%%%%%%%%%%%%%%%%%%%%%%%%%%%
%%%%%%%%%%%%%%%%%%%%%%%%%%%%%%%%%%%%%%%%%%%%%%%%%%%%%%%%%%%%%%%%%
%%%%%%%%%%%%%%%%%%%%%%%%%%%%%%%%%%%%%%%%%%%%%%%%%%%%%%%%%%%%%%%%%
%%%%%%%%%%%%%%%%%%%%%%%%%%%%%%%%%%%%%%%%%%%%%%%%%%%%%%%%%%%%%%%%%
\subsubsection{A high SNR regime}
\label{sec:highsnr}
%%%%%%%%%%%%%%%%%%%%%%%%%%%%%%%%%%%%%%%%%%%%%%%%%%%%%%%%%%%%%%%%%
%%%%%%%%%%%%%%%%%%%%%%%%%%%%%%%%%%%%%%%%%%%%%%%%%%%%%%%%%%%%%%%%%
%%%%%%%%%%%%%%%%%%%%%%%%%%%%%%%%%%%%%%%%%%%%%%%%%%%%%%%%%%%%%%%%%
%%%%%%%%%%%%%%%%%%%%%%%%%%%%%%%%%%%%%%%%%%%%%%%%%%%%%%%%%%%%%%%%%

It is also interesting to note that the above scenario allows one to immediately sees the effect of the factors $\mbox{SNR}_f=c_l$ on the excess risk.
Looking carefully at (\ref{eq:nrseq10}), we have that no matter how over-parameterized the system is, for a fixed $\kappa$, no further increase in $\mbox{SNR}_f=c_l$ can achieve consistency. To see this, we first from (\ref{eq:nrseq12}) observe
\begin{eqnarray}
 \lim_{\mbox{SNR}_f=c_l\rightarrow\infty} z
 =
  \lim_{\mbox{SNR}_f=c_l\rightarrow\infty}
  \frac{\sqrt{\frac{c_l^2}{\kappa^2} +2\frac{c_l}{\kappa}  -2\frac{c_l^2}{\kappa}  + (c_l + 1)^2} +\frac{c_l}{\kappa} -c_l- 1}{2 c_l }
 = \frac{1}{\kappa}.
\label{eq:nrseq14}
\end{eqnarray}
We then from (\ref{eq:nrseq10}) also have
\begin{eqnarray}\label{eq:nrseq15}
 \lim_{\mbox{SNR}_f=c_l\rightarrow\infty} \lim_{\alpha\rightarrow 0}   \lim_{n\rightarrow\infty} \mE_{Z,\v,E} R(\bar{\beta},\beta_{gls})
& = &
  \lim_{\mbox{SNR}_f=c_l\rightarrow\infty}  \frac{1 + \kappa z^2c_l^2\bar{\sigma}^2 }{\lp 1+c_lz \rp^2- \kappa z^2c_l^2} \nonumber \\
  & = &
     \lim_{\mbox{SNR}_f=c_l\rightarrow\infty}  \frac{1 + \frac{1}{\kappa}c_l^2\bar{\sigma}^2 }{\lp 1+c_l \frac{1}{\kappa} \rp^2-  \frac{1}{\kappa}c_l^2} \nonumber \\
       & = &
  \frac{\frac{1}{\kappa}\bar{\sigma}^2 }{\frac{1}{\kappa^2} -  \frac{1}{\kappa}} \nonumber \\
        & = &
  \frac{\kappa\bar{\sigma}^2 }{1-\kappa}.
\end{eqnarray}
This should, however, come as a no surprise. Namely as $\mbox{SNR}_f=c_l\rightarrow\infty$, for all practical purposes, we have the idealized factor model (i.e. the one without the noise $E$). For such a model one can from $X$ immediately get $Z$ and the resulting estimation/prediction problem transforms into a classical LRM with the under-parametrization ratio $\kappa$ and noise variance $\mbox{SNR}_{\y}=\bar{\sigma}^2$. For such a model the above expression, $\frac{\kappa\bar{\sigma}^2}{1-\kappa}$, obtained in (\ref{eq:nrseq15}) is exactly the associated excess risk.

%%%%%%%%%%%%%%%%%%%%%%%%%%%%%%%%%%%%%%%%%%%%%%%%%%%%%%%%%%%%%%%%%
%%%%%%%%%%%%%%%%%%%%%%%%%%%%%%%%%%%%%%%%%%%%%%%%%%%%%%%%%%%%%%%%%
%%%%%%%%%%%%%%%%%%%%%%%%%%%%%%%%%%%%%%%%%%%%%%%%%%%%%%%%%%%%%%%%%
%%%%%%%%%%%%%%%%%%%%%%%%%%%%%%%%%%%%%%%%%%%%%%%%%%%%%%%%%%%%%%%%%
\subsection{Simulation statistics}
\label{sec:simstat}
%%%%%%%%%%%%%%%%%%%%%%%%%%%%%%%%%%%%%%%%%%%%%%%%%%%%%%%%%%%%%%%%%
%%%%%%%%%%%%%%%%%%%%%%%%%%%%%%%%%%%%%%%%%%%%%%%%%%%%%%%%%%%%%%%%%
%%%%%%%%%%%%%%%%%%%%%%%%%%%%%%%%%%%%%%%%%%%%%%%%%%%%%%%%%%%%%%%%%
%%%%%%%%%%%%%%%%%%%%%%%%%%%%%%%%%%%%%%%%%%%%%%%%%%%%%%%%%%%%%%%%%

In all simulation setups we utilized statistics that exactly match the ones used in the theoretical analyses. However, the analytical results hold way more generally and the Gaussianity of $Z$ or $E$  is not needed. For the completeness, in addition to the presented simulation results, we also simulated two different statistical regimes, the one with the iid $\pm 1$ Bernullis and the other with the iid $\sqrt{12}\mbox{Unif}[-0.5,0.5]$. There was no point to include them as the obtained results were virtually identical to the Gaussian ones even for seemingly relatively small dimensions ($n=600$).

%%%%%%%%%%%%%%%%%%%%%%%%%%%%%%%%%%%%%%%%%%%%%%%%%%%%%%%%%%%%%%%%%%%%%%%%%%%%%%%%
%%%%%%%%%%%%%%%%%%%%%%%%%%%%%%%%%%%%%%%%%%%%%%%%%%%%%%%%%%%%%%%%%%%%%%%%%%%%%%%%
%%%%%%%%%%%%%%%%%%%%%%%%%%%%%%%%%%%%%%%%%%%%%%%%%%%%%%%%%%%%%%%%%%%%%%%%%%%%%%%%
%%%%%%%%%%%%%%%%%%%%%%%%%%%%%%%%%%%%%%%%%%%%%%%%%%%%%%%%%%%%%%%%%%%%%%%%%%%%%%%%
%%%%%%%%%%%%%%%%%%%%%%%%%%%%%%%%%%%%%%%%%%%%%%%%%%%%%%%%%%%%%%%%%%%%%%%%%%%%%%%%
\section{Conclusion}
\label{sec:conc}
%%%%%%%%%%%%%%%%%%%%%%%%%%%%%%%%%%%%%%%%%%%%%%%%%%%%%%%%%%%%%%%%%%%%%%%%%%%%%%%%
%%%%%%%%%%%%%%%%%%%%%%%%%%%%%%%%%%%%%%%%%%%%%%%%%%%%%%%%%%%%%%%%%%%%%%%%%%%%%%%%
%%%%%%%%%%%%%%%%%%%%%%%%%%%%%%%%%%%%%%%%%%%%%%%%%%%%%%%%%%%%%%%%%%%%%%%%%%%%%%%%
%%%%%%%%%%%%%%%%%%%%%%%%%%%%%%%%%%%%%%%%%%%%%%%%%%%%%%%%%%%%%%%%%%%%%%%%%%%%%%%%
%%%%%%%%%%%%%%%%%%%%%%%%%%%%%%%%%%%%%%%%%%%%%%%%%%%%%%%%%%%%%%%%%%%%%%%%%%%%%%%%

We studied classical (linear) factor  regression models (FRM) and analyzed the performance of the following three well known estimators associated with them: \textbf{\emph{(i)}} minimum norm interpolators (generalized least squares (GLS)); \textbf{\emph{(ii)}} ordinary least squares (LS);  and \textbf{\emph{(iii)}} ridge regularized estimators. We considered a statistical Gaussian context with inter-(hidden) factors and noise correlations which by the structure of the model implies inter-features correlations as well. Due to underlying statistics the optimization programs that produce the above three estimators are random. Through a utilization of a powerful mathematical engine, called \emph{Random Duality Theory} (RDT), we analyzed each of these programs and obtained precise closed form characterizations of all associated optimizing quantities. Viewed as a function of the over-parametrization ratio, the \emph{excess prediction risk} (as one of the key such quantities) is shown to exhibit the non-monotonic double-descent behavior. The risk's dependence of  all key model parameters, including the dimensions ratios and covariance matrices can be clearly seen through the provided closed form expressions. Moreover, the expressions turned out to be neatly transparent so that one can easily see the \emph{bias} and \emph{variance} risk's components and how they depend on all system parameters. For the LS estimator, we obtained analytical results that require no numerical optimization and as such are closed form functions of the models underlying covariance matrices.

We showed that (similarly to the linear regression models (LRM)), the optimally tuned ridge regularization can completely smoothen the FRM double-descent phenomenon. The theoretical analyses are supplemented by the corresponding numerical simulations. Two things are observed: \textbf{\emph{(i)}} There is an excellent agrement between the theoretical and simulated results; \textbf{\emph{(ii)}} The ``ridge smoothing'' effect appears limited for over-parametrization ratios larger than $5$ and as virtually nonexistent for those above $10$. This provides a strong confirmation that simple strategies (such as interpolation) can perform as strongly as the more complex/advanced ones.

A fairly generic character of the presented methodology allows for various extensions and/or generalizations. Those mentioned in \cite{Stojnicridge24} apply here as well. For example, imperfectly  observed features, absent features, adversarial responses are only a few among many popular ones. For each of them, different types of estimators can be considered as well. No further conceptual insights are needed to handle all of them  through the program presented here. The concrete technical realizations of these extensions are usually problem specific and need to be accordingly tailored. We will present them in separate papers.

As  pointed out in the companion paper \cite{Stojnicridge24} (and earlier in  \cite{StojnicRegRndDlt10}), the standard Gaussianity is not required for the presented RDT considerations (for the completeness we even simulated a couple of different statistical regimes and obtained virtually identical final results). The presentation is, however, much neater if it is used. While the detailed writing gets more cumbersome, no substantial conceptual adjustments are needed to accommodate deviations from the Gaussianity. We view an application of the Lindeberg central limit theorem variant (see, e.g.,  \cite{Lindeberg22}) as the most convenient way to extend the results to different statistics and its realization via the approach of \cite{Chatterjee06} as particularly elegant.

%\newpage1
%\setcounter{page}{1}
\begin{singlespace}
\bibliographystyle{plain}
\bibliography{nflgscompyxRefs}
\end{singlespace}

\end{document}